\documentclass{article}

\PassOptionsToPackage{numbers,sort&compress}{natbib}
\usepackage[final]{neurips_2024}

\usepackage[utf8]{inputenc} % allow utf-8 input
\usepackage[T1]{fontenc}    % use 8-bit T1 fonts
\usepackage{hyperref}       % hyperlinks
\usepackage{url}            % simple URL typesetting
\usepackage{booktabs}       % professional-quality tables
\usepackage{amsmath, amsthm, amssymb, mathtools}
\usepackage[all]{xy}
\usepackage{amsfonts}       % blackboard math symbols
\usepackage{nicefrac}       % compact symbols for 1/2, etc.
\usepackage{microtype}      % microtypography
\usepackage{xcolor}         % colors
\usepackage{rotating}
\usepackage{multirow}
\usepackage[version=4]{mhchem}
\usepackage{verbatim}

\usepackage[normalem]{ulem}

\usepackage{enumitem}

\usepackage[page]{appendix}
\usepackage{etoolbox}

\renewcommand{\appendixtocname}{Appendix Contents}

\makeatletter
\let\oldappendix\appendices

\renewcommand{\appendices}{%
  \clearpage
  \renewcommand{\thesection}{\Roman{section}}
  \let\tf@toc\tf@app
  \addtocontents{app}{\protect\setcounter{tocdepth}{2}}
  \immediate\write\@auxout{%
    \string\let\string\tf@toc\string\tf@app^^J
  }
  \oldappendix
}%

\newcommand{\listofappendices}{%
  \begingroup
  \renewcommand{\contentsname}{\appendixtocname}
  \let\@oldstarttoc\@starttoc
  \def\@starttoc##1{\@oldstarttoc{app}}
  \tableofcontents% Reusing the code for \tableofcontents with different \contentsname and different file handle app
  \endgroup
}

\makeatother

\let\oldAA\AA
\renewcommand{\AA}{\text{\normalfont\oldAA}}

\newcommand{\fig}{Fig.}
\newcommand{\Fig}{Figure}
\newcommand{\figref}[1]{\fig~\ref{#1}}

\newcommand{\Figref}[1]{\Fig~\ref{#1}}

\newcommand{\tabref}[1]{Table~\ref{#1}}
\newcommand{\Tabref}[1]{Table~\ref{#1}}

\renewcommand{\eqref}[1]{Eq.~(\ref{#1})}

\newcommand{\secref}[1]{Section~\ref{#1}}

\newcommand{\appref}[1]{Appendix~\ref{#1}}
\newcommand{\Appref}[1]{Appendix~\ref{#1}}

\newtheorem{theorem}{Theorem}[section]

\newtheorem{lemma}[theorem]{Lemma}
\newtheorem{proposition}[theorem]{Proposition}

\title{Higher-Rank Irreducible Cartesian Tensors for Equivariant Message Passing}

\author{\textbf{Viktor Zaverkin}$^{1, \ast}$ \hspace{10pt} \textbf{Francesco Alesiani}$^{1, \dagger}$ \hspace{10pt} \textbf{Takashi Maruyama}$^{1, \dagger}$ \hspace{10pt} \textbf{Federico Errica}$^2$ \\ \textbf{Henrik Christiansen}$^1$ \hspace{9pt} \textbf{Makoto Takamoto}$^1$ \hspace{9pt} \textbf{Nicolas Weber}$^1$ \hspace{9pt} \textbf{Mathias Niepert}$^{1,3}$ \\
$^1$NEC Laboratories Europe \quad 
$^2$NEC Italy \quad 
$^3$University of Stuttgart}

\begin{document}

\maketitle

\begin{abstract}
    The ability to perform fast and accurate atomistic simulations is crucial for advancing the chemical sciences. By learning from high-quality data, machine-learned interatomic potentials achieve accuracy on par with ab initio and first-principles methods at a fraction of their computational cost. The success of machine-learned interatomic potentials arises from integrating inductive biases such as equivariance to group actions on an atomic system, e.g., equivariance to rotations and reflections. In particular, the field has notably advanced with the emergence of equivariant message passing. Most of these models represent an atomic system using spherical tensors, tensor products of which require complicated numerical coefficients and can be computationally demanding. Cartesian tensors offer a promising alternative, though state-of-the-art methods lack flexibility in message-passing mechanisms, restricting their architectures and expressive power. This work explores higher-rank irreducible Cartesian tensors to address these limitations. We integrate irreducible Cartesian tensor products into message-passing neural networks and prove the equivariance and traceless property of the resulting layers. Through empirical evaluations on various benchmark data sets, we consistently observe on-par or better performance than that of state-of-the-art spherical and Cartesian models.
\end{abstract}

\section{Introduction \label{sec:introduction}}

\renewcommand{\thefootnote}{\fnsymbol{footnote}}
\footnotetext[1]{Corresponding author: \texttt{viktor.zaverkin@neclab.eu}}
\footnotetext[2]{These authors contributed equally.}
\renewcommand{\thefootnote}{\arabic{footnote}}

The ability to perform sufficiently fast and accurate atomistic simulations for large molecular and material systems holds the potential to revolutionize molecular and materials science~\cite{Butler2018, Vamathevan2019, Keith2021, Unke2021, Fedik2022, Merchant2023, Kovacs2023, Batatia2023}. Conventionally, computational chemistry and materials science rely on ab initio or first-principles approaches---e.g., coupled cluster~\cite{Purvis1982, Crawford2000, Bartlett2007} or density functional theory (DFT)~\cite{Hohenberg1964, Kohn1965}, respectively---that are accurate but computationally demanding, thus limiting the accessible simulation time and system sizes. However, the ability to generate high-quality, first-principles-based data sets has prompted the development of machine-learned interatomic potentials (MLIPs). These potentials enable atomistic simulations with accuracy that is on par with the reference first-principles method but at a fraction of the computational cost. Message-passing neural networks (MPNNs)~\cite{micheli_neural_2009, scarselli_graph_2009, Gilmer2017, hamilton_representation_2017, bronstein_geometric_2017, zhang_deep_2018, zhang_graph_2019, bacciu_gentle_2020} have been employed in chemical and materials sciences, including the development of MLIPs, due to their efficient processing of the graph representation of the atomic system~\cite{Zitnick2020, Jumper2021, Dauparas2022, Batzner2022, Duval2024}. Achieving the desired MLIPs' performance, however, requires the inclusion of inductive biases, like the invariance of total energy to translations, reflections, and rotations in the three-dimensional space, and an effective encoding of the atomic system into a learnable representation~\cite{Langer2022}. Designing equivariant MPNNs~\cite{Schuett2021, Haghighatlari2022, Simeon2023, Cheng2024, Batatia2022design, Batzner2022, Batatia2022, Musaelian2023, Passaro2023, Luo2024, Duval2024}, which preserve the directional information of the local atomic environment, has been one of the most active research directions of the last years. Previous work often achieves equivariance by representing an atomic system as a graph, with node features expressed in spherical harmonics basis or using lower-rank Cartesian tensors and designing ad-hoc message-passing layers~\cite{Schuett2021, Haghighatlari2022, Batzner2022, Batatia2022design, Batatia2022, Thoelke2022, Simeon2023, Passaro2023, Musaelian2023, Luo2024, Cheng2024}.

MLIPs based on spherical harmonics~\cite{Batzner2022, Batatia2022, Musaelian2023}, which are the basis for irreducible representations of the three-dimensional rotation group, often demonstrate a better performance compared to those that use lower-rank Cartesian representations (scalars and vectors)~\cite{Schuett2021, Thoelke2022}. Spherical tensors, however, require the definition of a particular rotational axis, often chosen as the $z$-axis, resulting in inherent bias~\cite{Snider2018}. Furthermore, coupling spherical tensors via tensor products, involved in designing equivariant convolutions and constructing many-body features~\cite{Thomas2018, Drautz2019, Batatia2022}, requires the definition of complicated numerical coefficients, e.g., Wigner $3\mbox{-}j$ symbols defined in terms of the Clebsch--Gordan coefficients~\cite{Wigner2012}, and can be computationally demanding. In contrast, irreducible Cartesian tensors have no preferential directions; their tensor products are simpler and have a better computational complexity---up to a certain tensor rank---than the tensor products of spherical tensors~\cite{Snider2018, Coope1965, Coope1970, Coope1970_2, Lehman1989}. Recent work improved results of Cartesian MPNNs by using rank-two tensors and decomposing them into irreducible representations of the three-dimensional rotation group~\cite{Simeon2023}, i.e., into representations of dimension 1 (trace), 3 (anti-symmetric part), and 5 (traceless symmetric part)~\cite{Weiler2018}. Higher-rank reducible Cartesian tensors have also been integrated into many-body MPNNs~\cite{Cheng2024}. These state-of-the-art Cartesian approaches, however, lack the flexibility of their message-passing mechanisms. They rely exclusively on convolutions with invariant filters and restrict the construction of many-body features, limiting the range of possible architectures and their expressive power.

\textbf{Contributions.} In this work, we address the limitations of state-of-the-art Cartesian models and demonstrate that operating with irreducible Cartesian tensors leads to on-par or, sometimes, even better performance than that of spherical counterparts. Particularly, this work goes beyond scalars, vectors, and rank-two tensors and explores the integration of higher-rank irreducible Cartesian tensors and their products into equivariant MPNNs: i) We demonstrate how irreducible Cartesian tensors that are symmetric and traceless can be constructed from a unit vector and a product of two irreducible Cartesian tensors, essential for equivariant convolutions and constructing many-body features; ii) We prove that the resulting tensors are traceless and equivariant under the action of the three-dimensional rotation and reflection group; iii) We demonstrate that higher-rank irreducible Cartesian tensors can be used to design cost-efficient---up to a certain tensor rank---and accurate equivariant models of many-body interactions; iv) We conduct different experiments to assess the effectiveness of equivariant message passing based on irreducible Cartesian tensors, and achieve state-of-the-art performance on benchmark data sets such as rMD17~\cite{Christensen2020b}, MD22~\cite{Chmiela2023}, 3BPA~\cite{Kovacs2021}, acetylacetone~\cite{Batatia2022design}, and Ta--V--Cr--W~\cite{Gubaev2023}. We hope our contributions will offer new insights into the use of Cartesian tensors for constructing accurate and cost-efficient MLIPs and beyond.

\section{Background \label{sec:background}}

\textbf{Group representations and equivariance.} A group $(G, \cdot)$ is defined by a set of elements $G$ and a group product $\cdot$. A representation $D$ of a group is a function from $G$ to square matrices such that $D[g] D[g^\prime] = D[g \cdot g^\prime],\,\forall\,g,g^\prime \in G$. This representation defines an action of $G$ to any vector space $\mathcal{X}$ (of the same dimension as the dimension of square matrices) through the matrix-vector multiplication, i.e., $(g, \mathbf{x}) \mapsto D_{\mathcal{X}}[g] \mathbf{x}$ for $\forall\, g \in G$ and $\forall\, \mathbf{x} \in \mathcal{X}$. For vector spaces $\mathcal{X}$ and $\mathcal{Y}$, a function $f: \mathcal{X} \rightarrow \mathcal{Y}$ is called equivariant to the action of a group $G$ to $\mathcal{X}$ and $\mathcal{Y}$ iff
\begin{equation*}
    f(D_\mathcal{X}[g]\mathbf{x}) = D_\mathcal{Y}[g] f(\mathbf{x})\,,\forall\,g \in G.
\end{equation*}
Invariance can be seen as a special type of equivariance with $D_\mathcal{Y}[g]$ being the identity for all $g$. An important class of equivariant models focuses on equivariance to the action of the Euclidean group $\mathrm{E}(3)$, comprising translations and the orthogonal group $\mathrm{O}(3)$, i.e., rotation group $\mathrm{SO}(3)$ and reflections, in $\mathbb{R}^3$. MLIPs based on equivariant models usually focus on equivariance to the action of the orthogonal group $\mathrm{O}(3)$ and are invariant to translations.

\textbf{Cartesian tensors.} A vector $\mathbf{x} \in \mathbb{R}^3$ transforms under the action of the rotation group $\mathrm{SO}(3)$ as $\mathbf{x}^\prime = R \mathbf{x}$, i.e., each component of it transforms as $\left(\mathbf{x}\right)_i^\prime = \sum_j R_{ij} \left(\mathbf{x}\right)_j$. Here, $R = D_\mathcal{X}\left[g\right] \in \mathbb{R}^{3\times 3}$ denotes the rotation matrix representation of $g \in \mathrm{SO}(3)$. Cartesian tensors of rank $n$ are described by $3^n$ numbers and generalize the concept of vectors. A rank-$n$ Cartesian tensor $\mathbf{T}$ can be viewed as a multidimensional array with $n$ indices, i.e., $\left(\mathbf{T}\right)_{i_1i_2\cdots i_n}$ with $i_k \in \{1, 2, 3\}$ for $\forall\, k \in \{1, \cdots, n\}$. Furthermore, each index of $\left(\mathbf{T}\right)_{i_1i_2\cdots i_n}$ transforms independently as a vector under rotation. For example, a rank-two tensor transforms under rotation as $\mathbf{T}^\prime = R \mathbf{T} R^\top$, i.e., each component of it transforms as $\left(\mathbf{T}\right)_{i_1i_2}^\prime = \sum_{j_1j_2} R_{i_1j_1}R_{i_2j_2} \left(\mathbf{T}\right)_{j_1j_2}$. For Cartesian tensors, one defines an $r$-fold tensor contraction and an outer product, i.e., $\left(\mathbf{T}\right)_{j_1\cdots j_s} = \sum_{i_1,\cdots,i_r} \left(\mathbf{U}\right)_{i_1\cdots i_rj_1\cdots j_s}\left(\mathbf{S}\right)_{i_1\cdots i_r}$ and $\left(\mathbf{T}\right)_{i_1\cdots i_rj_1\cdots j_s} = \left(\mathbf{U}\right)_{i_1\cdots i_r}\left(\mathbf{S}\right)_{j_1\cdots j_s}$, respectively. Cartesian tensors are generally reducible and can, thus, be decomposed into smaller representations that transform independently within their linear subspaces under rotation. For example, a rank-two reducible Cartesian tensor contains representations of dimension 1 (trace), 3 (anti-symmetric part), and 5 (traceless symmetric part): $\sum_{i_1}(\mathbf{T})_{i_1i_1}$, $(\mathbf{T})_{i_1i_2} - (\mathbf{T})_{i_2i_1}$, and $(\mathbf{T})_{i_1i_2} + (\mathbf{T})_{i_2i_1} - 2/3\sum_{i_1}(\mathbf{T})_{i_1i_1}$, respectively. Under rotation, the traceless symmetric part remains within its irreducible subspace, with a similar behavior for other irreducible components of $(\mathbf{T})_{i_1i_2}$. Furthermore, for a reducible Cartesian tensor of rank $n$, an action of the rotation group $\mathrm{SO}(3)$ can be represented with a $3^n \times 3^n$-dimensional rotation matrix. This rotation matrix is also reducible, meaning that an appropriate change of basis can block diagonalize it into smaller subsets, the irreducible representations of the rotation group $\mathrm{SO}(3)$.

\textbf{Spherical tensors.} Spherical harmonics (or spherical tensors) $Y_m^l$ are functions from the points on a sphere to complex or real numbers, with degree $l \geq 0$ and components $-l \leq m \leq l$. Collecting all components for a given $l$ we obtain a $(2l+1)$-dimensional object $\mathbf{Y}^l = (Y_{-l}^l, \dots, Y_{l-1}^l, Y_{l}^l)$. Spherical tensors transform under rotation as $Y_m^l(R\hat{\mathbf{x}}) = \sum_{m^\prime} (\mathbf{D})_{mm^\prime}^l Y_{m^\prime}^l(\hat{\mathbf{x}})$, where $(\mathbf{D})_{mm^\prime}^l$ is the irreducible representation of $\mathrm{SO}(3)$---the Wigner $D$-matrix---and $R$ denotes the rotation matrix. Spherical tensors are irreducible and can be combined using the Clebsch--Gordan tensor product
\begin{equation*}
\left(Y_{m_1}^{l_1} \otimes_{\mathrm{CG}} Y_{m_2}^{l_2}\right)_{m_3}^{l_3} = \sum\nolimits_{m_1=-l_1}^{l_1}\sum\nolimits_{m_2=-l_2}^{l_2} C_{l_1m_1,l_2m_2}^{l_3m_3}Y_{m_1}^{l_1}Y_{m_2}^{l_2},
\end{equation*}
with Clebsch--Gordan coefficients $C_{l_1m_1,l_2m_2}^{l_3m_3}$ and $l_3 \in \left\{\lvert l_1-l_2\rvert, \cdots, l_1 + l_2\right\}$. Furthermore, any reducible Cartesian tensor of rank $n$ can be written in spherical harmonics as $\left(\hat{\mathbf{x}}\right)_{i_1}\left(\hat{\mathbf{x}}\right)_{i_2}\cdots\left(\hat{\mathbf{x}}\right)_{i_n} = \sum_{l=0}^n\sum_{m=-l}^{l} X_{i_1i_2\cdots i_n}^{lm} Y_m^l$, with coefficients $X_{i_1i_2\cdots i_n}^{lm}$ defined elsewhere~\cite{Drautz2020}.

\section{Related work \label{sec:related_work}}

\textbf{Equivariant message-passing potentials.} Equivariant MPNNs~\cite{Thomas2018, Anderson2019, Schuett2021, Satorras2022, Brandstetter2022, Batzner2022, Thoelke2022, Batatia2022design, Batatia2022, Haghighatlari2022, Passaro2023, Musaelian2023, Simeon2023, Luo2024, Cheng2024} often outperform more traditional invariant models~\cite{Schuett2017, Unke2019, Liu2022, Gasteiger2022b, Gasteiger2022a, Gasteiger2022c}. While many invariant and equivariant MPNNs rely on two-body features within a single message-passing layer, there is a growing interest in constructing higher-body-order features, such as angles and dihedrals, to model many-body interactions in atomic systems~\cite{Gasteiger2022a, Gasteiger2022c, Liu2022}. Note that equivariant models with two-body features in a single message-passing layer build many-body ones through repeated message-passing iterations, increasing the receptive field and, thus, computational cost. Recent work recognized the importance of higher-body-order features but faced challenges due to explicit summation over triplets or quadruplets~\cite{Liu2022, Gasteiger2022a}. In contrast, MACE advances the current state-of-the-art by combining ACE and equivariant message passing, introducing cost-efficient many-body message-passing potentials~\cite{Batatia2022}. With just two message-passing layers, MACE yields accurate potentials for interacting many-body systems~\cite{Kovacs2023b}.

\textbf{Beyond Clebsch--Gordan tensor product.} Despite the success of equivariant MPNNs based on spherical tensors, the high computational cost of the Clebsch--Gordan tensor product limits their computational efficiency~\cite{Fuchs2020, Schuett2021, Satorras2022, Frank2023, Batatia2022, Brandstetter2022, Musaelian2023, Passaro2023, Simeon2023, Liao2023}. Thus, current research focuses on alternatives to the Clebsch--Gordan tensor product, which has a $\mathcal{O}(L^5)$ complexity for tensors up to degree $L$. Recently, the relation of Clebsch--Gordan coefficients to the integral of products of three spherical harmonics, known as the Gaunt coefficients~\cite{Wigner2012}, has been exploited to reduce the computational cost of the tensor product of spherical tensors to $\mathcal{O}(L^3)$, compared to $\mathcal{O}(L^6)$ of the full Clebsch--Gordan tensor product including all $(l_1, l_2) \rightarrow l_3$~\cite{Luo2024}. However, this approach excludes odd tensor products and, thus, restricts the expressive power and the range of possible architectures.

Cartesian tensors and their products offer another promising alternative to the Clebsch--Gordan tensor product, enabling the efficient construction of message-passing layers with two- and many-body features. Recent work has explored decomposing rank-two tensors into their irreducible representations and using higher-rank reducible tensors~\cite{Simeon2023, Cheng2024}. These approaches, however, are limited to convolutions with invariant filters, restricting the range of possible architectures and, thus, the expressive power of resulting Cartesian models. Furthermore, they provide limited mechanisms for constructing higher-body-order features, restricted to three-body features obtained through matrix-matrix multiplication and invariant many-body features obtained through full tensor contractions, similar to moment tensor potentials and Gaussian moments~\cite{Shapeev2016, Zaverkin2020, Zaverkin2021b}. Finally, using reducible Cartesian tensors during message-passing leads to mixing different irreducible representations, which can hinder the performance of resulting models compared to state-of-the-art spherical models~\cite{Cheng2024}.

\section{Methods \label{sec:methods}}

We define an atomic configuration $\mathcal{S} = \{\mathbf{r}_u, Z_u\}_{u=1}^{N_\mathrm{at}}$, where $\mathbf{r}_u \in \mathbb{R}^3$ denotes Cartesian coordinates and $Z_u \in \mathbb{N}$ represents the atomic number of atom $u$, with a total of $N_\mathrm{at}$ atoms. Our focus lies on message-passing MLIPs, parameterized by $\boldsymbol{\theta}$, that learn a mapping from a configuration $\mathcal{S}$ to the total energy $E$, i.e., $f_{\boldsymbol{\theta}}: \mathcal{S} \mapsto E \in \mathbb{R}$. Thus, we represent molecular and material systems as graphs in a three-dimensional Euclidean space. An edge $\{u, v\}$ exists if atoms $u$ and $v$ are within a cutoff distance $r_\mathrm{c}$, i.e., $\lVert \mathbf{r}_u - \mathbf{r}_v \rVert_2 \leq r_\mathrm{c}$. For more details on MPNNs, see \appref{sec:background_appendix}. The total energy of an atomic configuration is defined by the sum of individual atomic energy contributions~\citep{Behler2007}, i.e.,~$E = \sum_{u=1}^{N_\mathrm{at}} E_u$. Atomic forces are computed as negative gradients of the total energy with respect to atomic coordinates, i.e., $\mathbf{F}_u = -\nabla_{\mathbf{r}_u} E$.

\subsection{Irreducible Cartesian tensor product \label{sec:cartesian_product}}

\begin{figure}[t!]
    \begin{center}
        \includegraphics[width=0.78\textwidth]{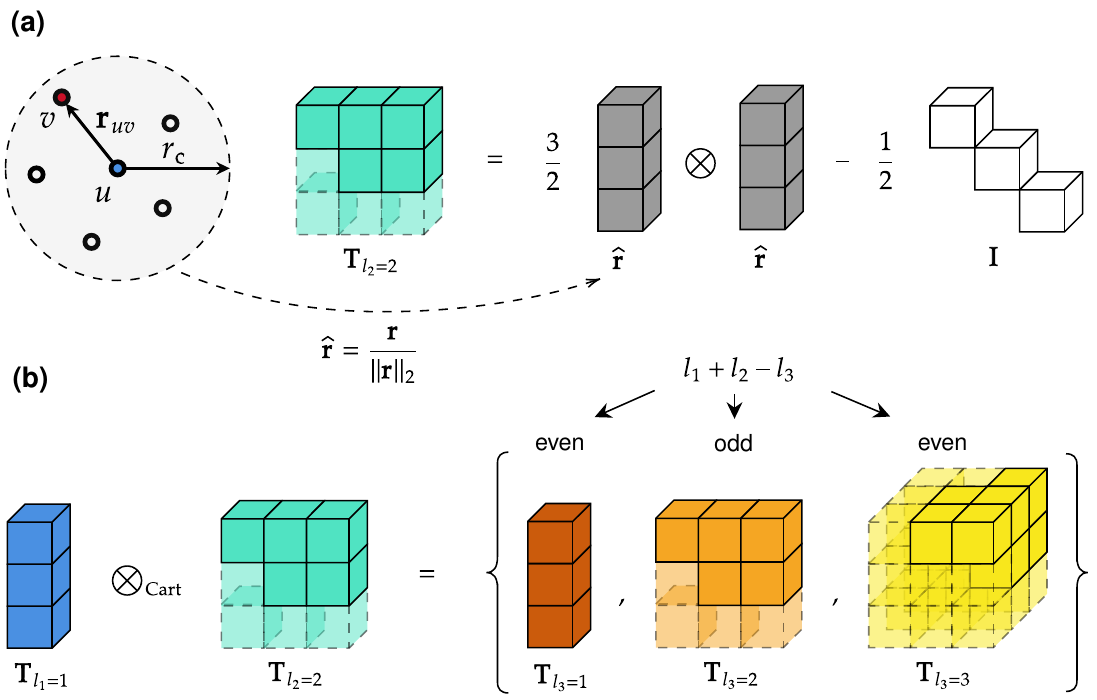}
    \end{center}
    \caption{\textbf{Schematic illustration of (a) the construction of an irreducible Cartesian tensor for a local atomic environment and (b) the tensor product of two irreducible Cartesian tensors of rank $l_1$ and $l_2$.} The construction of an irreducible Cartesian tensor from a unit vector $\hat{\mathbf{r}}$ is defined in \eqref{eq:cartesian_irreps}. In this work, we use tensors with the same rank $n$ and weight $l$, i.e., $n=l$, avoiding the need for embedding tensors with $l < n$ in a higher-dimensional tensor space. Therefore, we use $l$ to identify the rank and the weight of an irreducible Cartesian tensor. The tensor product is defined in Eqs.~(\ref{eq:product_even}) and (\ref{eq:product_odd}), resulting in a new tensor $\mathbf{T}_{l_3} = (\mathbf{T}_{l_1} \otimes_{\mathrm{Cart}} \mathbf{T}_{l_2})_{l_3}$~of rank $l_3 = \{\lvert l_1 - l_2\rvert, \cdots, l_1 + l_2\}$. Transparent boxes denote the linearly dependent elements of symmetric and traceless tensors. The tensor product can be even or odd, defined by~$l_1+l_2-l_3$.}
    \label{fig:tensor_and_tensor_product}
\end{figure}

In the following, we explore irreducible Cartesian tensors and their products based on the three-dimensional vector space and the three-dimensional orthogonal group $\mathrm{O}(3)$, which comprises rotations and reflections. The respective tensors and tensor products are schematically illustrated in \figref{fig:tensor_and_tensor_product} (a) and (b), respectively, and are further employed in constructing MPNNs for atomic systems equivariant under actions of the orthogonal group. An irreducible Cartesian tensor $\mathbf{T}_{n} \in (\mathbb{R}^3)^{\otimes n}$ of rank $n$ and weight $l \leq n$ (related to the degree $l$ of spherical tensors) can be represented by a tensor with $3^n$ components in a three-dimensional vector space. These $3^n$ components form a basis for a $(2l+1)$-dimensional irreducible representation of the three-dimensional rotation group $\mathrm{SO}(3)$~\cite{Fano1959, Gelfand1963}. Thus, only $2l+1$ of the $3^n$ components are independent.

The rotation in the space of all tensors of rank $n$ is induced through the $n$-fold outer product of a rotation matrix $R \in \mathbb{R}^{3 \times 3}$, i.e., $R^{\otimes n} = R \otimes \cdots \otimes R$. The obtained representation of the rotation group $\mathrm{SO}(3)$ is reducible for all $n$ except $n=\{0,1\}$. Reducing a tensor of rank $n$ yields a unique irreducible tensor with the same weight and rank ($n=l$), which is characterized by being symmetric, i.e., $(\mathbf{T}_{n})_{\cdots i\cdots j\cdots}=(\mathbf{T}_{n})_{\cdots j\cdots i\cdots}\ , \forall\ i \neq j \in \{i_1\cdots,i_n\}$, and traceless, i.e., $\sum_{i}(\mathbf{T}_{n})_{\cdots i\cdots i\cdots} = 0\ , \forall\,i \in \{i_1, \cdots, i_n\}$. An irreducible tensor of rank $n$ and weight $l$ with $l < n$ can be viewed as a $l$-rank tensor embedded in the $n$-rank tensor space, e.g., by computing an outer product with the identity matrix. However, the embedding is often not unique. Thus, we construct tensors with the same weight and rank ($n=l$) in the following.

\textbf{Irreducible Cartesian tensors from unit vectors.} Hereafter, we denote the rank and the weight of irreducible Cartesian tensors by $l$ to distinguish them from reducible counterparts. An irreducible Cartesian tensor of an arbitrary rank $l$ can be constructed from a unit vector $\hat{\mathbf{r}}$ in the form~\cite{Lehman1989}
\begin{equation} 
    \label{eq:cartesian_irreps}
    \mathbf{T}_{l} (\hat{\mathbf{r}}) = C \sum\nolimits_{m=0}^{\lfloor l/2\rfloor} (-1)^m \frac{(2l-2m-1)!!}{(2l-1)!!} \big\{\hat{\mathbf{r}}^{\otimes (l-2m)}\otimes\mathbf{I}^{\otimes m}\big\},
\end{equation}
resulting in a symmetric and traceless tensor of rank $l$. Here, $\mathbf{I}$ denotes the $3 \times 3$ identity matrix, $\hat{\mathbf{r}}^{\otimes (l-2m)}=\hat{\mathbf{r}}\otimes\cdots\otimes\hat{\mathbf{r}}$ and $\mathbf{I}^{\otimes m} = \mathbf{I}\otimes\cdots\otimes\mathbf{I}$ are the corresponding $(l-2m)$- and $m$-fold outer products. The curly brackets indicate the summation over all permutations of the $l$ unsymmetrized indices~\cite{Lehman1989}, i.e., $\{\mathbf{T}_l\}_{i_{1} \cdots i_{l}} = \sum_{\pi \in S_l} (\mathbf{T}_l)_{i_{\pi(1)} \cdots i_{\pi(l)}}$, with $S_l$ being the corresponding set of permutations. The expression in \eqref{eq:cartesian_irreps} involves three distinct outer products $(\hat{\mathbf{r}}\otimes\hat{\mathbf{r}})_{i_1i_2} = \hat{r}_{i_1}\hat{r}_{i_2}$, $(\mathbf{I}\otimes\mathbf{I})_{i_1i_2i_3i_4} = \delta_{i_1i_2}\delta_{i_3i_4}$, and $(\hat{\mathbf{r}}\otimes\mathbf{I})_{i_1i_2i_3} = \hat{r}_{i_1}\delta_{i_2i_3}$, where $\delta_{i_2i_3}$ denotes the Kronecker delta. The normalization constant $C = \left(2l-1\right)!!/l!$ is chosen such that an $l$-fold contraction of $\mathbf{T}_l$ with the unit vector $\hat{\mathbf{r}}$ yields unity. An example of an irreducible Cartesian tensor with rank $l=3$ is $(\mathbf{T}_{l=3})_{i_1i_2i_3} = \frac{5}{2} \big(\hat{r}_{i_1}\hat{r}_{i_2}\hat{r}_{i_3} - \frac{1}{5}(\hat{r}_{i_1}\delta_{i_2i_3} + \hat{r}_{i_2}\delta_{i_3i_1} + \hat{r}_{i_3}\delta_{i_1i_2})\big)$.

\textbf{Irreducible Cartesian tensor product.} The following defines the product of two irreducible Cartesian tensors $\mathbf{x}_{l_1} \in (\mathbb{R}^3)^{\otimes l_1}$ and $\mathbf{y}_{l_2} \in (\mathbb{R}^3)^{\otimes l_2}$, yielding an irreducible Cartesian tensor of rank $l_3$, i.e., $\mathbf{z}_{l_3} = (\mathbf{x}_{l_1} \otimes_{\mathrm{Cart}} \mathbf{y}_{l_2})_{l_3} \in (\mathbb{R}^3)^{\otimes l_3}\,\forall\,l_3 \in \{\lvert l_1-l_2 \rvert,  \cdots, l_1+l_2\}$, that is symmetric and traceless. The irreducible Cartesian tensor product is crucial for designing equivariant MPNNs and is used for equivariant convolutions and constructing many-body features in \secref{sec:cartesian_message_passing}. For an even $l_1+l_2-l_3 = 2k$, the general form of an irreducible Cartesian tensor of rank $l_3$ reads~\cite{Lehman1989}
\begin{equation} 
    \label{eq:product_even}
    \begin{split}
        (\mathbf{x}_{l_1} & \otimes_{\mathrm{Cart}} \mathbf{y}_{l_2})_{l_3} \\ 
        = & C_{l_1l_2l_3} \sum\nolimits_{m=0}^{\mathrm{min}(l_1,l_2) - k}(-1)^m 2^m \frac{(2l_3-2m-1)!!}{(2l_3-1)!!}\big\{\left(\mathbf{x}_{l_1}\cdot(k+m)\cdot\mathbf{y}_{l_2}\right)\otimes\mathbf{I}^{\otimes m}\big\},
    \end{split}
\end{equation}
where $\left(\mathbf{x}_{l_1}\cdot(k+m)\cdot\mathbf{y}_{l_2}\right) = \sum_{i_1,\cdots,i_{k+m}} (\mathbf{x}_{l_1})_{i_1\cdots i_{k+m}} (\mathbf{y}_{l_2})_{i_1\cdots i_{k+m}}$ denotes an $(k+m)$-fold tensor contraction, which results in a tensor of rank $l_1+l_2-2(k+m)$. For simplicity, we skip the uncontracted indices in the above definition. For example, for $l_1=4$ and $l_2=3$ and a three-fold tensor contraction we obtain $\left(\mathbf{x}_{l_1}\cdot 3\cdot\mathbf{y}_{l_2}\right)_{i_4} = \sum_{i_1,i_2,i_3} (\mathbf{x}_{l_1})_{i_1i_2i_3i_4}(\mathbf{y}_{l_2})_{i_1i_2i_3}$, i.e., the corresponding tensors are contracted along  $i_1$, $i_2$, and $i_3$. Note that the final result is independent of the index permutation, as the contracted tensors are symmetric. For an odd $l_1+l_2-l_3 = 2k + 1$, we define~\cite{Lehman1989}
\begin{equation} 
    \label{eq:product_odd}
    \begin{split}
        (\mathbf{x}_{l_1} & \otimes_{\mathrm{Cart}} \mathbf{y}_{l_2})_{l_3} \\ 
        = & D_{l_1l_2l_3} \sum\nolimits_{m=0}^{\mathrm{min}(l_1,l_2)-k-1}(-1)^m 2^m \frac{(2l_3-2m-1)!!}{(2l_3-1)!!}\big\{\left(\boldsymbol{\varepsilon}:\mathbf{x}_{l_1}\cdot(k+m)\cdot\mathbf{y}_{l_2}\right)\otimes\mathbf{I}^{\otimes m}\big\},
    \end{split}
\end{equation}
with $\boldsymbol{\varepsilon}$ denoting the Levi-Civita symbol ($\varepsilon_{i_1i_2i_3} = - \varepsilon_{i_3i_2i_1}$ and $\varepsilon_{i_1i_1i_3} = 0$). The double contraction with the Levi-Civita symbol reads $(\boldsymbol{\varepsilon}:\mathbf{x}_{l_1}\cdot(k+m)\cdot\mathbf{y}_{l_2})_{i_1} = \sum_{i_2,i_3} \varepsilon_{i_1i_2i_3} \left(\mathbf{x}_{l_1}\cdot(k+m)\cdot\mathbf{y}_{l_2}\right)_{i_2i_3}$, and yields a tensor of rank $l_1+l_2-2(k+m)-1$. Details on the normalization constants $C_{l_1l_2l_3}$ and $D_{l_1l_2l_3}$ are provided in \appref{sec:normalization_constants}.

\subsection{Equivariant message-passing based on irreducible Cartesian tensors \label{sec:cartesian_message_passing}}

The following section introduces the basic operations for constructing equivariant MPNNs based on irreducible Cartesian tensors. Using their irreducible tensor products, we demonstrate how to build equivariant two- and many-body features, crucial for modeling many-body interactions in molecular and materials systems. Particularly, we focus on MLIPs based on equivariant MPNNs and extend the state-of-the-art MACE architecture~\cite{Batatia2022} to the Cartesian basis. Following the MACE architecture, we use only even tensor products. We split vectors $\mathbf{r}_{uv} = \mathbf{r}_u - \mathbf{r}_v \in \mathbb{R}^3$ from atom $u$ to atom $v$, schematically shown in \figref{fig:tensor_and_tensor_product} (a), into their radial and angular components (unit vectors), i.e., $r_{uv} = \lVert\mathbf{r}_{uv}\rVert_2 \in \mathbb{R}$ and $\hat{\mathbf{r}}_{uv} = \mathbf{r}_{uv} / r_{uv} \in \mathbb{R}^3$, respectively. In the $t$-th message-passing layer, edges $\{u,v\}$ are embedded using a fully connected neural network $R_{kl_1l_2l_3}^{(t)}: \mathbb{R} \rightarrow \mathbb{R}$ with $k$ output feature channels. The radial function $R_{kl_1l_2l_3}^{(t)}$ takes as an input radial distances $r_{uv}$, which are embedded through Bessel functions and multiplied by a smooth polynomial cutoff function~\cite{Gasteiger2022c}. Finally, we use irreducible Cartesian tensors $\mathbf{T}_l(\hat{\mathbf{r}})$, similar to spherical tensors $\mathbf{Y}^l(\hat{\mathbf{r}})$, to embed unit vectors into the tensor space of maximal rank $l_\mathrm{max}$.

\textbf{Equivariant convolutions and two-body features.} Rotation equivariance in MPNNs is typically achieved by constraining convolution filters to be the products between learnable radial functions and spherical tensors, i.e., $R_{kl_1l_2l_3}^{(t)}(r_{uv})Y_{m_1}^{l_1}(\hat{\mathbf{r}}_{uv})$. The two-body features $A_{ukl_3m_3}^{(t)}$ are further obtained through the tensor product---the point-wise convolution~\cite{Thomas2018}---between the respective filters and neighbors' equivariant features $h_{ukl_2m_2}^{(t)}$. The permutational invariance is enforced by pooling over the neighbors $v \in \mathcal{N}(u)$. Here, we use irreducible Cartesian tensors, with the rotation-equivariant filters given by $R_{kl_1l_2l_3}^{(t)}(r_{uv})\big(\mathbf{T}_{l_1}(\hat{\mathbf{r}}_{uv})\big)_{i_1i_2\cdots i_{l_1}}$. Thus, two-body features $\big(\mathbf{A}_{ukl_3}^{(t)}\big)_{i_1i_2\cdots i_{l_3}}$ are obtained using the irreducible Cartesian tensor product and are represented by rank-$l_3$ irreducible Cartesian tensors. The Cartesian two-body features are defined by
\begin{equation}
  \label{eq:atomic_basis}
  \big(\mathbf{A}_{ukl_3}^{(t)}\big)_{i_1i_2\cdots i_{l_3}} = \sum\nolimits_{v \in \mathcal{N}(u)} \Big(R_{kl_1l_2l_3}^{(t)}(r_{uv})\mathbf{T}_{l_1}(\hat{\mathbf{r}}_{uv}) \otimes_{\mathrm{Cart}} \frac{1}{\sqrt{d_t}}\sum_{k^\prime}W_{kk^\prime l_2}^{(t)}\mathbf{h}_{vk^\prime l_2}^{(t)}\Big)_{i_1i_2\cdots i_{l_3}},
\end{equation}
where $d_t$ represents the number of feature channels in the node embeddings $\mathbf{h}_{vk^\prime l_2}^{(t)}$ of the $t$-th message-passing layer. In the first message-passing layer, node embeddings are initialized as learnable weights $W_{kZ_v}$ that are invariant to actions of the orthogonal group, i.e., are scalars or tensors of rank $l_2=0$, and embed the atom type $Z_v$. Thus, constructing equivariant two-body features simplifies to
\begin{equation}
  \label{eq:atomic_basis_first}
  \big(\mathbf{A}_{ukl_1}^{(1)}\big)_{i_1i_2\cdots i_{l_1}} = \sum\nolimits_{v \in \mathcal{N}(u)} R_{kl_1}^{(1)}(r_{uv}) \big(\mathbf{T}_{l_1}(\hat{\mathbf{r}}_{uv})\big)_{i_1i_2\cdots i_{l_1}} W_{kZ_v}.
\end{equation}

\textbf{Equivariant many-body features.} The importance of many-body terms arises from the fact that the interaction between atoms changes when additional atoms are present; see \appref{sec:background_appendix}. Furthermore, many-body terms are often required to ensure the generalization of interatomic potentials, i.e., their ability to accurately predict energies and forces for temperatures and stoichiometries on which they were not trained~\cite{Mueser2023}. Here, we construct $(\nu+1)$-body equivariant features from $\big(\mathbf{A}_{ukl_\xi}^{(t)}\big)_{i_1i_2\cdots i_{l_\xi}}$ obtained using Eqs.~(\ref{eq:atomic_basis}) or (\ref{eq:atomic_basis_first}). The $\nu$-fold Cartesian tensor product, which yields $(\nu+1)$-body features represented by an irreducible Cartesian tensor of rank $L$, reads
\begin{equation}
    \label{eq:product_basis}
    \big(\mathbf{B}_{u\eta_{\nu} kL}^{(t)}\big)_{i_1i_2\cdots i_L} = \big(\underbrace{\tilde{\mathbf{A}}_{ukl_1}^{(t)} \otimes_\mathrm{Cart} \cdots \otimes_\mathrm{Cart} \tilde{\mathbf{A}}_{ukl_\nu}^{(t)}}_{\nu\text{-fold}}\big)_{i_1i_2\cdots i_L},
\end{equation}
where $\eta_\nu$ counts all possible $\nu$-fold products of $\{l_1,\cdots,l_\nu\}$-rank tensors, yielding rank-$L$ irreducible Cartesian tensors, and $\tilde{\mathbf{A}}_{ukl_\xi}^{(t)} = \frac{1}{\sqrt{d_t}}\sum_{k^\prime}W_{kk^\prime l_\xi}^{(t)}\mathbf{A}_{uk^\prime l_\xi}^{(t)}$ with $d_t$ feature channels. 

The irreducible Cartesian tensor product does not allow pre-computing the coefficients of the $\nu$-fold tensor product, differing from spherical tensors that use the generalized Clebsch--Gordan coefficients, contracted with weights from \eqref{eq:messages_uncoupled} and summed over the possible paths $\mathrm{len}(\eta_\nu)$, for this purpose~\cite{Drautz2019, Drautz2020, Batatia2022}. Thus, we obtain the result of \eqref{eq:product_basis} by iteratively applying the irreducible Cartesian tensor product $\left(\nu-1\right)$ times and refer to the respective models as irreducible Cartesian tensor potentials (ICTPs) with the full product basis or ICTP$_\text{full}$. However, two-fold tensor products in \eqref{eq:product_basis} are symmetric to permutations of involved tensors. Thus, the number of the $\nu$-fold tensor products, $\mathrm{len}(\eta_\nu)$, can be significantly reduced; we refer to the corresponding models as ICTP$_\text{sym}$. Furthermore, we can reduce the computational cost of the Cartesian product basis by performing the calculations in a latent feature space. We use learnable weights $W_{kk^\prime l_\xi}$ to reduce the number of feature channels for the product basis calculation and then increase it again for subsequent steps; we refer to the corresponding models as ICTP$_\text{lt}$. For more details on the model architecture, such as the construction of updated many-body node embeddings, readout functions, and different options for the Cartesian product basis, see \appref{sec:architecture_appendix}.

\textbf{Runtime considerations.} When choosing an architecture to implement MLIPs, the runtime per energy and force evaluation is crucial. The computational complexity as a function of rank $L$ is $\mathcal{O}\left(9^{L} L!/\left(2^{L/2}\left(L/2\right)!\right)\right)$ for the irreducible Cartesian tensor product and $\mathcal{O}\left(L^5\right)$ for the Clebsch--Gordan one; see \appref{sec:comput_complexity} for more details. Thus, equivariant convolutions based on spherical tensors are more computationally efficient for $L \rightarrow \infty$ than those based on irreducible Cartesian tensors. However, state-of-the-art models and physical properties typically require $L \leq 4$~\cite{Batatia2022, Grega2024}, making sub-leading terms and implementation-dependent amplitudes crucial. Based on our analysis, for $L \leq 4$, we can expect equivariant convolutions based on irreducible Cartesian tensors to be more computationally and memory efficient than their spherical counterparts. For $L \leq 4$, the $\nu$-fold tensor product in the Cartesian basis can also offer computational advantages. Its cost, as a function of rank $L$ and correlation order $\nu$, is $\mathcal{K} (9^{L} L!/(2^{L/2}(L/2)!))^{\nu-1}$ and $\mathcal{K} L^{\frac{1}{2}\nu(\nu + 3)}$ for ICTP and MACE, respectively. The pre-factor $\mathcal{K}$, which counts all possible $\nu$-fold tensor products, is removed in MACE through generalized Clebsch-Gordan coefficients, though these coefficients increase the memory spherical models use. Therefore, it is essential to consider the inference times for a fair comparison, which we provide in \secref{sec:results}.

\textbf{Equivariance and tracelss property of message-passing layers.} We conclude this section by giving a theoretical result that ensures the equivariance of message-passing layers based on irreducible Cartesian tensors and their irreducible tensor products to actions of the orthogonal group. The proof is provided in \appref{sec:proof_equivariance}. We also prove in Appendix~\ref{sec:proof_traceless} that these message-passing layers preserve the traceless property of irreducible Cartesian tensors.

\begin{proposition}
\label{prop:ictp_equivariance}
The message-passing layers based on irreducible Cartesian tensors and their irreducible tensor products are equivariant to actions of the orthogonal group.
\end{proposition}

\begin{proposition}
\label{prop:ictp_tracelss}
The message-passing layers based on irreducible Cartesian tensors and their irreducible tensor products preserve the traceless property of irreducible Cartesian tensors.
\end{proposition}

\section{Experiments and results \label{sec:results}}

This section presents the results for the five benchmark data sets: rMD17, MD22, 3BPA, acetylacetone, and Ta--V--Cr--W. We describe data sets and training details in Appendices~\ref{sec:datasets_appendix} and \ref{sec:training_appendix}, respectively.

\begin{figure}[t!]
    \begin{center}
        \includegraphics[width=0.99\textwidth]{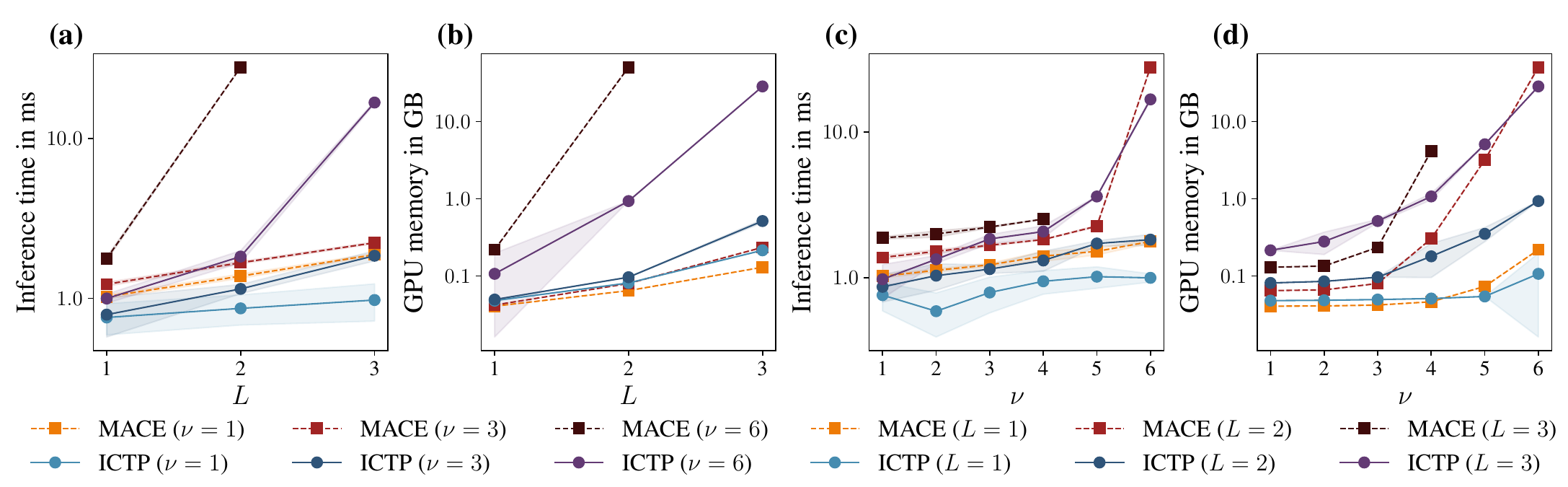}
    \end{center}
    \caption{\textbf{Inference times and memory consumption as a function of the tensor rank $L$ (a)--(b) and the correlation order $\nu$ (c)--(d).} All results are obtained for the 3BPA data set and $l_\mathrm{max} = L$. We used eight feature channels to allow experiments with larger $\nu$ values. MACE models use intermediate tensors with $l > l_\mathrm{max}$ for their product basis, which we fixed to $l = l_\mathrm{max}$. Otherwise, pre-computing generalized Clebsch–Gordan coefficients for $\nu > 4$ would require more than 2~TB of RAM. For ICTP, we used the full product basis to compute the same number of $\nu$-fold tensor products as in MACE.}
    \label{fig:3bpa-results-scaling}
\end{figure}

\textbf{Scaling and computational cost.} The expressive power and computational efficiency of equivariant many-body message-passing potentials depend on the tensor ranks employed in equivariant message passing and embedding local atomic environments, i.e., $L$ and $l_\mathrm{max}$, respectively, as well as the correlation order $\nu$. Recent work has shown that for identifying environments with $L$-fold symmetries, at least rank-$L$ tensors are required~\cite{Joshi2023}. These symmetries are typically lifted in atomistic simulations, motivating the use of $L \leq 4$. Higher body-order correlations $\nu$, in turn, are required if atomic environments are degenerate to a lower body-order correlation $\nu - 1$~\cite{Pozdnyakov2020, Joshi2023}. \Figref{fig:3bpa-results-scaling} demonstrates inference times and memory consumption of models based on irreducible Cartesian and spherical tensors, i.e., ICTP and MACE, respectively, as a function of the tensor rank and the correlation order. \Tabref{tab:3bpa-results-scaling} presents the corresponding numerical results. We find that irreducible Cartesian tensors outperform spherical ones for most parameter values. In particular, irreducible Cartesian tensors allow spanning the $\nu$-space more efficiently, in line with our theoretical results in \secref{sec:cartesian_message_passing}.

\begin{table*}[t!]
	\caption{\textbf{Energy (E) and force (F) mean absolute errors (MAEs) for the rMD17 data set.} E- and F-MAE are given in meV and meV/\AA, respectively. Results are shown for models trained using $N_\text{train}=\{950, 50\}$ configurations randomly drawn from the data set, with further 50 used for early stopping. All values are obtained by averaging over five independent runs, with the standard deviation provided if available. Best performances, considering the standard deviation, are highlighted in bold.
	\label{tab:rmd17-results}}
	\begin{center}
    \resizebox{\textwidth}{!}{
      \begin{tabular}{lclllllrrr}
        \toprule  
                                                &   & \multicolumn{5}{c}{$N_\text{train} = 950$}                                                                                                              & \multicolumn{3}{c}{$N_\text{train} = 50$}                                               \\
			    	                            &   & ICTP$_\text{sym}$       & TensorNet~\cite{Simeon2023}       & MACE~\cite{Batatia2022}   & Allegro~\cite{Musaelian2023}      & NequIP~\cite{Batzner2022}	    & NequIP~\cite{Batzner2022}     & MACE~\cite{Batatia2022}   & ICTP$_\mathrm{sym}$       \\		    	                            
        \cmidrule(lr){1-2} \cmidrule(lr){3-7} \cmidrule(lr){8-10}
        \multirow[l]{2}{*}{Aspirin}             & E & \textbf{2.27 $\pm$ 0.11}	    & 2.4		                        & \textbf{2.2}		        & 2.3		                        & 2.3		                    & 19.5	                        & 17.0	                    & \textbf{14.84 $\pm$ 0.98}     \\
                                                & F & \textbf{6.67 $\pm$ 0.19}	    & 8.9 $\pm$ 0.1		                    & \textbf{6.6}		        & 7.3		                        & 8.2		                    & 52.0		                    & 43.9                      & \textbf{40.19 $\pm$ 0.95}     \\
        \cmidrule(lr){1-2} \cmidrule(lr){3-7} \cmidrule(lr){8-10}
        \multirow[l]{2}{*}{Azobenzene}	        & E & 1.20 $\pm$ 0.01	            & \textbf{0.7}		                & 1.2		                & 1.2		                        & \textbf{0.7}		            & 6.0		                    & \textbf{5.4}              & \textbf{5.47 $\pm$ 0.63}      \\
                                                & F & 2.92 $\pm$ 0.03	            & 3.1		                        & 3.0		                & \textbf{2.6}		                & 2.9		                    & 20.0		                    & \textbf{17.7}             & \textbf{17.25 $\pm$ 0.53}     \\
        \cmidrule(lr){1-2} \cmidrule(lr){3-7} \cmidrule(lr){8-10}
        \multirow[l]{2}{*}{Benzene}	            & E & 0.26 $\pm$ 0.00	            & \textbf{0.02}                     & 0.4		                & 0.3		                        & 0.04		                    & 0.6		                    & 0.7                       & \textbf{0.38 $\pm$ 0.02}      \\
                                                & F & 0.34 $\pm$ 0.02	            & 0.3                               & 0.3		                & \textbf{0.2}		                & 0.3		                    & 2.9		                    & 2.7                       & \textbf{2.45 $\pm$ 0.13}      \\
        \cmidrule(lr){1-2} \cmidrule(lr){3-7} \cmidrule(lr){8-10}
	    \multirow[l]{2}{*}{Ethanol}	            & E & \textbf{0.43 $\pm$ 0.02}		& 0.5			                    & \textbf{0.4}				& \textbf{0.4}				        & \textbf{0.4}				    & 8.7		                    & 6.7                       & \textbf{6.15 $\pm$ 0.26}      \\
                                                & F & 2.63 $\pm$ 0.10                & 3.5			                    & \textbf{2.1}				& \textbf{2.1}				        & 2.8				            & 40.2		                    & 32.6                      & \textbf{29.53 $\pm$ 1.14}     \\
        \cmidrule(lr){1-2} \cmidrule(lr){3-7} \cmidrule(lr){8-10}
        \multirow[l]{2}{*}{Malonaldehyde}	    & E & 0.82 $\pm$ 0.03	            & 0.8			                    & 0.8			            & \textbf{0.6}				        & 0.8				            & 12.7		                    & \textbf{10.0}             & \textbf{9.72 $\pm$ 0.42}      \\
                                                & F & 4.96 $\pm$ 0.21               & 5.4			                    & 4.1			            & \textbf{3.6}				        & 5.1				            & 52.5		                    & \textbf{43.3}             & \textbf{42.88 $\pm$ 3.08}     \\
        \cmidrule(lr){1-2} \cmidrule(lr){3-7} \cmidrule(lr){8-10}
        \multirow[l]{2}{*}{Naphthalene}	        & E & 0.56 $\pm$ 0.00			    & \textbf{0.2}			            & 0.5			            & \textbf{0.2}				        & 0.9				            & \textbf{2.1}		            & \textbf{2.1}              & \textbf{2.06 $\pm$ 0.10}      \\
                                                & F & 1.45 $\pm$ 0.05               & 1.6			                    & 1.6				        & \textbf{0.9}				        & 1.3				            & 10.0		                    & \textbf{9.2}              & \textbf{9.43 $\pm$ 0.46}      \\
        \cmidrule(lr){1-2} \cmidrule(lr){3-7} \cmidrule(lr){8-10}
        \multirow[l]{2}{*}{Paracetamol}	        & E & 1.44 $\pm$ 0.03			    & \textbf{1.3}			            & \textbf{1.3}				& 1.5				                & 1.4				            & 14.3		                    & 9.7                       & \textbf{8.94 $\pm$ 0.66}      \\
                                                & F & \textbf{4.89 $\pm$ 0.11}      & 5.9 $\pm$ 0.1			                & \textbf{4.8}				& 4.9				                & 5.9				            & 39.7		                    & \textbf{31.5}             & \textbf{30.13 $\pm$ 1.51}     \\
        \cmidrule(lr){1-2} \cmidrule(lr){3-7} \cmidrule(lr){8-10}
        \multirow[l]{2}{*}{Salicylic acid}	    & E & 0.97 $\pm$ 0.01			    & 0.8			                    & 0.9				        & 0.9				                & \textbf{0.7}				    & 8.0		                    & 6.5                       & \textbf{5.95 $\pm$ 0.43}      \\
                                                & F & 3.66 $\pm$ 0.06               & 4.6 $\pm$ 0.1			                & 3.1				        & \textbf{2.9}				        & 4.0				            & 35.9		                    & \textbf{28.4}             & \textbf{27.78 $\pm$ 1.93}     \\
        \cmidrule(lr){1-2} \cmidrule(lr){3-7} \cmidrule(lr){8-10}
        \multirow[l]{2}{*}{Toluene}	            & E & 0.46 $\pm$ 0.00			    & \textbf{0.3}			            & 0.5				        & 0.4				                & \textbf{0.3}				    & 3.3		                    & 3.1                       & \textbf{2.45 $\pm$ 0.13}      \\
                                                & F & 1.61 $\pm$ 0.02               & 1.7			                    & \textbf{1.5}				& 1.8				                & 1.6				            & 15.1		                    & 12.1                      & \textbf{11.24 $\pm$ 0.55}     \\
        \cmidrule(lr){1-2} \cmidrule(lr){3-7} \cmidrule(lr){8-10}
        \multirow[l]{2}{*}{Uracil}	            & E & 0.57 $\pm$ 0.01			    & \textbf{0.4}			            & 0.5				        & 0.6				                & \textbf{0.4}				    & 7.3		                    & \textbf{4.4}              & 4.66 $\pm$ 0.16               \\
                                                & F & 2.64 $\pm$ 0.08               & 3.1			                    & 2.1				        & \textbf{1.8}				        & 3.1				            & 40.1		                    & \textbf{25.9}             & \textbf{25.97 $\pm$ 0.78}     \\
	    \bottomrule 
	    \end{tabular}
      }
	\end{center}
\end{table*}

\textbf{Molecular dynamics trajectories.} We assess the performance of ICTP models using the revised MD17 (rMD17) data set, which includes structures, total energies, and atomic forces for ten small organic molecules obtained from ab initio molecular dynamics simulations~\cite{Christensen2020b}. \Tabref{tab:rmd17-results} shows that ICTP$_\text{sym}$ achieves accuracy on par with state-of-the-art spherical and Cartesian models. Notably, several methods exhibit similar accuracy when trained with 950 configurations. However, the achieved accuracy is much lower than the desired accuracy of $43.37$~meV $\approx 1$~kcal/mol, making a model comparison less meaningful. Therefore, we also compare ICTP$_\text{sym}$ with MACE and NequIP, trained using 50 configurations, making learning accurate MLIPs more challenging. From \tabref{tab:rmd17-results}, we see that ICTP$_\text{sym}$ outperforms MACE and NequIP for most molecules in this scenario. 

We further evaluate the performance of ICTP using the MD22 data set, which contains seven molecular systems with 42--370 atoms~\cite{Chmiela2023}. This data set spans four major classes of biomolecules and supramolecules and was designed to challenge short-range models. \Tabref{tab:md22-results} shows that ICTP achieves an accuracy on par with or better than state-of-the-art models, including long-range ones.

\begin{table*}[t!]
	\caption{\textbf{Energy (E) and force (F) root-mean-square errors (RMSEs) for the 3BPA data set.} E- and F-RMSE are given in meV and meV/\AA, respectively. Results are shown for models trained using 450 configurations randomly drawn from the training data set collected at 300~K, with further 50 used for early stopping. All ICTP results are obtained by averaging over five independent runs. For MACE and NequIP, the results are reported for three runs. The standard deviation is provided if it is available. Best performances, considering the standard deviation, are highlighted in bold. Inference time and memory consumption are measured for a batch size of 100. Inference time is reported per structure in ms, while memory consumption is provided for the entire batch in GB.
	\label{tab:3bpa-results}}
	\begin{center}
    \resizebox{\textwidth}{!}{
      \begin{tabular}{lcrrrrrrrr}
        \toprule 
			    	                          &   & ICTP$_\text{full}$          & ICTP$_\text{sym}$         & ICTP$_{\text{sym} + \text{lt}}$  & MACE\textsuperscript{\emph{a}}        & CACE~\cite{Cheng2024}  & MACE~\cite{Batatia2022}                  & NequIP~\cite{Musaelian2023}                                \\
        \cmidrule(lr){1-2} \cmidrule(lr){3-6} \cmidrule(lr){7-9}
        \multirow[l]{2}{*}{300~K}             & E & \textbf{2.70 $\pm$ 0.22}	& \textbf{2.70 $\pm$ 0.08}	& \textbf{2.98 $\pm$ 0.34}		   & \textbf{2.81 $\pm$ 0.18}              & 6.3		            & \textbf{3.0 $\pm$ 0.2}                   & 3.28 $\pm$ 0.10                                            \\
                                              & F & \textbf{9.45 $\pm$ 0.29}	& \textbf{9.39 $\pm$ 0.31}	& \textbf{9.57 $\pm$ 0.20}	       & \textbf{9.47 $\pm$ 0.42}              & 21.4		            & \textbf{8.8 $\pm$ 0.3}                   & 10.77 $\pm$ 0.19                                           \\
        \cmidrule(lr){1-2} \cmidrule(lr){3-6} \cmidrule(lr){7-9}
        \multirow[l]{2}{*}{600~K}	          & E & \textbf{10.74 $\pm$ 0.31}	& \textbf{10.38 $\pm$ 0.80}	& \textbf{10.29 $\pm$ 0.90}		   & \textbf{11.11 $\pm$ 1.41}             & 18.0		            & \textbf{9.7 $\pm$ 0.5}                   & 11.16 $\pm$ 0.14                                           \\
                                              & F & \textbf{22.99 $\pm$ 0.64}	& \textbf{22.87 $\pm$ 0.91}	& \textbf{23.03 $\pm$ 0.76}		   & \textbf{23.27 $\pm$ 1.45}             & 45.2		            & \textbf{21.8 $\pm$ 0.6}                  & 26.37 $\pm$ 0.09                                           \\
        \cmidrule(lr){1-2} \cmidrule(lr){3-6} \cmidrule(lr){7-9}
        \multirow[l]{2}{*}{1200~K}	          & E & \textbf{29.80 $\pm$ 0.92}	& \textbf{30.84 $\pm$ 1.87}	& \textbf{31.32 $\pm$ 1.80}		   & \textbf{31.15 $\pm$ 1.58}             & 58.0		            & \textbf{29.8 $\pm$ 1.0}                  & 38.52 $\pm$ 1.63                                           \\
                                              & F & \textbf{62.82 $\pm$ 1.23}	& \textbf{64.54 $\pm$ 3.88}	& \textbf{65.36 $\pm$ 3.47}		   & \textbf{65.22 $\pm$ 3.52}             & 113.8                  & \textbf{62.0 $\pm$ 0.7}                  & 76.18 $\pm$ 1.11	                                        \\
        \cmidrule(lr){1-2} \cmidrule(lr){3-6} \cmidrule(lr){7-9}
	    \multirow[l]{2}{*}{Dihedral slices}	  & E & \textbf{9.82 $\pm$ 0.79}	& 10.64 $\pm$ 1.07			& 13.03 $\pm$ 3.44				   & \textbf{8.56 $\pm$ 1.53}              & --				        & \textbf{7.8 $\pm$ 0.6}                   & 23.2~\cite{Batatia2022}                                    \\
                                              & F & \textbf{17.52 $\pm$ 0.54}   & \textbf{17.18 $\pm$ 0.81}	& 19.31 $\pm$ 0.83                 & \textbf{17.69 $\pm$ 1.29}             & --				        & \textbf{16.5 $\pm$ 1.7}                  & 23.1~\cite{Batatia2022}                                    \\
        \cmidrule(lr){1-2} \cmidrule(lr){3-6} \cmidrule(lr){7-9}
	    Inference time                        &   & 6.45 $\pm$ 0.50			    & 5.31 $\pm$ 0.02			& \textbf{3.51 $\pm$ 0.22}		   & 4.66 $\pm$ 0.05                       & --		                & \textbf{24.3}\textsuperscript{\emph{b}}  & 103.5\textsuperscript{\emph{b}}~\cite{Batatia2022}         \\
	    \cmidrule(lr){1-2} \cmidrule(lr){3-6} \cmidrule(lr){7-9}
	    Memory consumption                    &   & 49.66 $\pm$ 0.00			& 42.01 $\pm$ 0.11			& 39.08 $\pm$ 0.00		           & \textbf{36.26 $\pm$ 0.00}             & --		                & --	                                   & --                                                         \\
	    \bottomrule 
	    \end{tabular}
      }
	\end{center}
	\footnotesize{\textsuperscript{\emph{a}} During inference time measurements with the MACE source code, we were not able to reproduce the original results~\cite{Batatia2022}. Thus, we re-run MACE experiments using a training setup similar to that of ICTP; see \appref{sec:training_appendix}.}\\
	\footnotesize{\textsuperscript{\emph{b}} The original publication did not report the batch size used to measure inference time~\cite{Batatia2022}. Therefore, the values provided are used solely to demonstrate the relative computational cost of MACE and NequIP.}
\end{table*}

\textbf{Extrapolation to out-of-domain data.} We continue to assess the performance of ICTP models using the 3BPA data set~\cite{Kovacs2021}. The training data set comprises 500 configurations, total energies, and atomic forces acquired from molecular dynamics at 300~K. The test data set is obtained from simulations at 300~K, 600~K, and 1200~K. We also test our models using energies and forces along dihedral rotations of the molecule. \Tabref{tab:3bpa-results} shows that ICTP models trained using 450 configurations perform on par with state-of-the-art spherical models, similar to the results for rMD17. However, we were not able to reproduce the original results using the current MACE source code and the described training setup~\cite{Batatia2022}. Therefore, for a fair comparison, we unified the training setup for ICTP and MACE (see \appref{sec:training_appendix}) and \tabref{tab:3bpa-results} reports the corresponding results for 450 training configurations. In \tabref{tab:3bpa-results-50_configs}, we present the results obtained for ICTP and MACE trained using 50 configurations.

From \tabref{tab:3bpa-results}, we observe that ICTP$_\text{full}$ slightly outperforms MACE in total energy and atomic force RMSEs but is $\sim 1.4$ times less computationally efficient. This difference arises from MACE using the generalized Clebsch--Gordan coefficients and pre-computing their product with the weights in the linear expansion in \eqref{eq:messages_uncoupled}~\cite{Batatia2022}, which reduces the effective number of evaluated tensor products. Thus, we may attribute the lower computational efficiency of ICTP$_\text{full}$ to the use of the MACE architecture, which results in a larger pre-factor $\mathcal{K}$ for Cartesian models but facilitates a fair comparison between irreducible Cartesian and spherical tensors. Using the symmetric Cartesian product basis and that in the latent space, for example, we further improve the runtime of our models while maintaining accuracy on par with MACE.

\Tabref{tab:3bpa-results-no_product} presents additional results obtained for $\nu=1$, i.e., focusing on models that rely exclusively on two-body interactions. We observe that Cartesian models exhibit shorter inference times than spherical ones, with MACE and ICTP$_\text{full}$ achieving $2.96 \pm 0.06$~ms and $2.63 \pm 0.02$~ms, respectively. Regarding memory consumption, MACE and ICTP perform similarly despite the larger number of tensor products required for the Cartesian product basis in \eqref{eq:product_basis}. This observation can be attributed to, for example, the Clebsch--Gordan tensor product requiring the computation of intermediate tensors with $(2L+1)^2$ elements, whereas irreducible Cartesian tensors contain $3^L$ elements.

\begin{figure}[t!]
    \begin{center}
        \includegraphics[width=0.99\textwidth]{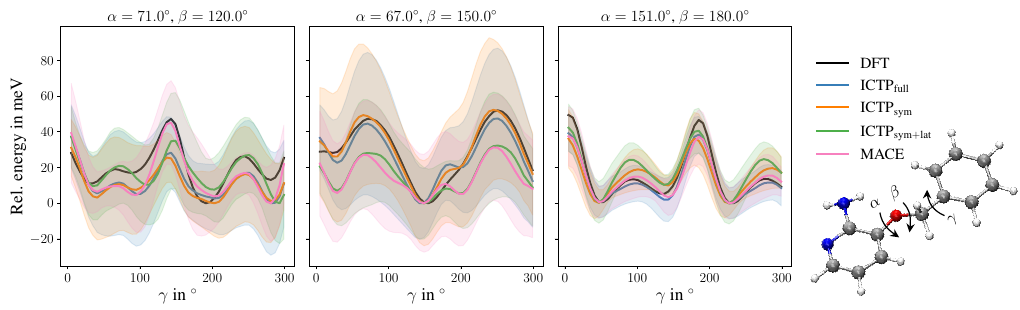}
    \end{center}
    \caption{\textbf{Potential energy profiles for three cuts through the 3BPA molecule's potential energy surface.} All models are trained using 50 configurations, and additional 50 are used for early stopping. The 3BPA molecule, including the three dihedral angles ($\alpha$, $\beta$, and $\gamma$), provided in degrees~$^\circ$, is shown as an inset. The color code of the inset molecule is C grey, O red, N blue, and H white. The reference potential energy profile (DFT) is shown in black. Each profile is shifted such that each model's lowest energy is zero. Shaded areas denote standard deviations across five independent runs.}
    \label{fig:3bpa-results-50_configs}
\end{figure}

\Figref{fig:3bpa-results-50_configs} compares potential energy profiles obtained with ICTP and MACE trained using 50 configurations. For the potential energy cut at $\beta=120^\circ$ (left panel), ICTP and MACE perform similarly, except for the energy barrier at  $\gamma\approx143^\circ$, which ICTP tends to underestimate stronger than MACE. For $\beta=150^\circ$ (middle panel), however, ICTP$_\text{full}$ and ICTP$_\text{sym}$ outperform MACE across nearly the entire range of the dihedral angle $\gamma$. For $\beta=180^\circ$ (right panel), all models perform similarly. \Figref{fig:3bpa-results} shows the corresponding potential energy profiles for models trained with 450 configurations. All models perform similarly in this scenario, with energy profiles close to the reference (DFT).

\begin{table*}[t!]
	\caption{\textbf{Energy (E) and force (F) root-mean-square errors (RMSEs) for the acetylacetone data set.} E- and F-RMSE are given in meV and meV/\AA, respectively. Results are shown for models trained using 450 configurations randomly drawn from the training data set collected at 300~K, with further 50 used for early stopping. All ICTP results are obtained by averaging over five independent runs. For MACE and NequIP, the results are reported for three runs. The standard deviation is provided if it is available. Best performances, considering the standard deviation, are highlighted in bold.
	\label{tab:acac-results}}
	\begin{center}
    \resizebox{\textwidth}{!}{
      \begin{tabular}{lcrrrrrr}
        \toprule 
			    	                &   & ICTP$_\text{full}$        & ICTP$_\text{sym}$          & ICTP$_{\text{sym} + \text{lt}}$  & MACE\textsuperscript{\emph{a}}          & MACE~\cite{Batatia2022}          & NequIP~\cite{Batatia2022design}       \\
        \cmidrule(lr){1-2} \cmidrule(lr){3-6} \cmidrule(lr){7-8}
        \multirow[l]{2}{*}{300~K}   & E & \textbf{0.75 $\pm$ 0.04}	& \textbf{0.76 $\pm$ 0.03}	 & \textbf{0.77 $\pm$ 0.04}         & \textbf{0.75 $\pm$ 0.05}                & 0.9 $\pm$ 0.03		             & \textbf{0.81 $\pm$ 0.05}              \\
                                    & F & \textbf{5.08 $\pm$ 0.11}	& \textbf{5.17 $\pm$ 0.10}	 & \textbf{5.18 $\pm$ 0.16}         & \textbf{5.00 $\pm$ 0.17}                & \textbf{5.1 $\pm$ 0.1}		     & 5.90 $\pm$ 0.46                       \\
        \cmidrule(lr){1-2} \cmidrule(lr){3-6} \cmidrule(lr){7-8}
        \multirow[l]{2}{*}{600~K}	& E & \textbf{5.39 $\pm$ 1.22}	& \textbf{4.43 $\pm$ 0.34}	 & \textbf{5.12 $\pm$ 0.29}         & \textbf{4.96 $\pm$ 0.64}                & \textbf{4.6 $\pm$ 0.3}		     & 6.04 $\pm$ 1.54                       \\
                                    & F & \textbf{23.21 $\pm$ 1.96}	& \textbf{22.90 $\pm$ 1.62}	 & \textbf{24.05 $\pm$ 1.71}        & \textbf{23.25 $\pm$ 1.82}               & \textbf{22.4 $\pm$ 0.9}		     & 27.80 $\pm$ 4.03                      \\
        \cmidrule(lr){1-2} \cmidrule(lr){3-6} \cmidrule(lr){7-8}
	    Number of parameters        &   & 2,774,800				    & 2,736,400				     & 2,648,080                        & 2,803,984         		              & 2,803,984			             & 3,190,488                		     \\
	    \bottomrule 
	    \end{tabular}
      }
	\end{center}
	\footnotesize{\textsuperscript{\emph{a}} Similar to \tabref{tab:3bpa-results}, we re-run MACE experiments using the similar training setup as for ICTP; see \appref{sec:training_appendix}.}
\end{table*}

\textbf{Flexibility and reactivity.} We further use the acetylacetone data set to assess the ICTP models' extrapolation capabilities to higher temperatures (similar to 3BPA), bond breaking, and bond torsions~\cite{Batatia2022design}. \Tabref{tab:acac-results} shows that ICTP models achieve state-of-the-art results while employing fewer parameters than spherical counterparts. \Appref{sec:results_appendix} includes additional results for the acetylacetone data set, such as total energy and atomic force RMSEs for models trained with 50 configurations and details on the potential energy profiles for hydrogen transfer and C-C bond rotation. Overall, ICTP and MACE perform similarly, demonstrating excellent generalization capability. However, when trained using 50 configurations, ICTP$_\text{full}$ is the only MLIP consistently producing the potential energy profile for hydrogen transfer close to the reference (DFT).

\textbf{Multicomponent alloys.} We finally evaluate the ICTP and MACE models using the Ta--V--Cr--W data set, designed to assess the performance of state-of-the-art MLIPs in modeling chemically complex multicomponent systems. In this evaluation, we attempt to predict energies and forces for Ta--V--Cr--W subsystems under two scenarios: The 0 K energies and forces in binary, ternary, and quaternary systems and near-melting temperature energies and forces in 4-component disordered alloys. \Tabref{tab:hea-results} shows that ICTP$_\mathrm{sym}$ outperforms MACE in nearly all subsystems, particularly in energy prediction. ICTP achieves an overall accuracy of 1.38 $\pm$ 0.09~meV/atom for energies and 0.028 $\pm$ 0.001~eV/\AA{} for forces, compared to 2.19 $\pm$ 0.31~meV/atom and 0.029 $\pm$ 0.001~eV/\AA{} by MACE. However, the $\mathcal{K}$ pre-factor from the $\nu$-fold tensor product results in longer inference times for ICTP than MACE, in line with the discussion for 3BPA.

\section{Conclusions and limitations \label{sec:conclusions}}

This work introduces many-body equivariant MPNNs based on higher-rank irreducible Cartesian tensors, offering an alternative to spherical models and addressing the limitations of state-of-the-art Cartesian models. We assess the performance of resulting MPNNs using five benchmark data sets, such as rMD17, MD22, 3BPA, acetylacetone, and Ta--V--Cr--W. In these experiments, MPNNs based on irreducible Cartesian tensors show a lower computational cost of individual operations compared to spherical counterparts. Furthermore, we demonstrate that these Cartesian models achieve accuracy and generalization capability on par with or better than state-of-the-art spherical models while memory consumption is comparable. Our results hold across the typical range of tensor ranks used in modeling many-body interactions and relevant physical properties, i.e., $L\leq 4$.

\textbf{Limitations.} We emphasize our focus on introducing MPNNs based on irreducible Cartesian tensors and prove their equivariance and traceless property. We adapted the MACE architecture, which uses only even tensor products, to enable a fair comparison with state-of-the-art spherical models. Further modifications to the architecture are possible and necessary, e.g., to reduce the pre-factor arising from the Cartesian product basis, before we can fully exploit the potential of irreducible Cartesian tensors.

\newpage

\section*{Data availability}

All data sets used in this study are publicly available: rMD17 (\url{https://doi.org/10.6084/m9.figshare.12672038.v3}), MD22 (\url{http://www.sgdml.org}), 3BPA (\url{https://github.com/davkovacs/BOTNet-datasets}), acetylacetone (\url{https://github.com/davkovacs/BOTNet-datasets}), and Ta--V--Cr--W (\url{https://doi.org/10.18419/darus-3516}).

\section*{Code availability}

The source code is available on GitHub and can be accessed via this link: \url{https://github.com/nec-research/ictp}.

\section*{Acknowledgements}

MN acknowledges support from the Deutsche Forschungsgemeinschaft (DFG, German Research Foundation) under Germany's Excellence Strategy - EXC 2075 – 390740016 and the Stuttgart Center for Simulation Science (SimTech).

\bibliography{references}

\begin{thebibliography}{10}
\providecommand{\url}[1]{\texttt{#1}}
\providecommand{\urlprefix}{URL }
\providecommand{\eprint}[2][]{\url{#2}}

\bibitem{Butler2018}
K.~T. Butler, D.~W. Davies, H.~Cartwright, O.~Isayev, and A.~Walsh: \textit{Machine learning for molecular and materials science}.
\newblock Nature \textbf{559}, 547--555 (2018)

\bibitem{Vamathevan2019}
J.~Vamathevan, D.~Clark, P.~Czodrowski, I.~Dunham, E.~Ferran et~al.: \textit{Applications of machine learning in drug discovery and development}.
\newblock Nat. Rev. Drug Discov. \textbf{18}, 463--477 (2019)

\bibitem{Keith2021}
J.~A. Keith, V.~Vassilev-Galindo, B.~Cheng, S.~Chmiela, M.~Gastegger et~al.: \textit{Combining Machine Learning and Computational Chemistry for Predictive Insights Into Chemical Systems}.
\newblock Chem. Rev. \textbf{121}, 9816--9872 (2021)

\bibitem{Unke2021}
O.~T. Unke, S.~Chmiela, H.~E. Sauceda, M.~Gastegger, I.~Poltavsky et~al.: \textit{{Machine Learning Force Fields}}.
\newblock Chem. Rev. \textbf{121}, 10142--10186 (2021)

\bibitem{Fedik2022}
N.~Fedik, R.~Zubatyuk, M.~Kulichenko, N.~Lubbers, J.~S. Smith et~al.: \textit{Extending machine learning beyond interatomic potentials for predicting molecular properties}.
\newblock Nat. Rev. Chem. \textbf{6}, 653--672 (2022)

\bibitem{Merchant2023}
A.~Merchant, S.~Batzner, S.~S. Schoenholz, M.~Aykol, G.~Cheon et~al.: \textit{Scaling deep learning for materials discovery}.
\newblock Nature \textbf{624}, 80--85 (2023)

\bibitem{Kovacs2023}
D.~P. Kovács, J.~H. Moore, N.~J. Browning, I.~Batatia, J.~T. Horton et~al.: \textit{MACE-OFF23: Transferable Machine Learning Force Fields for Organic Molecules}.
\newblock \eprint{https://arxiv.org/abs/2312.15211} (2023)

\bibitem{Batatia2023}
I.~Batatia, P.~Benner, Y.~Chiang, A.~M. Elena, D.~P. Kovács et~al.: \textit{A foundation model for atomistic materials chemistry}.
\newblock \eprint{https://arxiv.org/abs/2401.00096} (2023)

\bibitem{Purvis1982}
G.~D. Purvis and R.~J. Bartlett: \textit{A full coupled‐cluster singles and doubles model: The inclusion of disconnected triples}.
\newblock J. Chem. Phys. \textbf{76}, 1910--1918 (1982)

\bibitem{Crawford2000}
T.~D. Crawford and H.~F. {Schaefer III}: \textit{An Introduction to Coupled Cluster Theory for Computational Chemists}, pp. 33--136.
\newblock John Wiley \& Sons, Ltd (2000)

\bibitem{Bartlett2007}
R.~J. Bartlett and M.~Musia\l{}: \textit{Coupled-cluster theory in quantum chemistry}.
\newblock Rev. Mod. Phys. \textbf{79}, 291--352 (2007)

\bibitem{Hohenberg1964}
P.~Hohenberg and W.~Kohn: \textit{Inhomogeneous Electron Gas}.
\newblock Phys. Rev. \textbf{136}, B864--B871 (1964)

\bibitem{Kohn1965}
W.~Kohn and L.~J. Sham: \textit{{Self-Consistent Equations Including Exchange and Correlation Effects}}.
\newblock Phys. Rev. \textbf{140}, A1133--A1138 (1965)

\bibitem{micheli_neural_2009}
A.~Micheli: \textit{Neural network for graphs: {A} contextual constructive approach}.
\newblock IEEE Trans. Neural Netw. \textbf{20}, 498--511 (2009)

\bibitem{scarselli_graph_2009}
F.~Scarselli, M.~Gori, A.~C. Tsoi, M.~Hagenbuchner, and G.~Monfardini: \textit{The graph neural network model}.
\newblock IEEE Trans. Neural Netw. \textbf{20}, 61--80 (2009)

\bibitem{Gilmer2017}
J.~Gilmer, S.~S. Schoenholz, P.~F. Riley, O.~Vinyals, and G.~E. Dahl: \textit{{Neural Message Passing for Quantum Chemistry}}.
\newblock Int. Conf. Mach. Learn. \textbf{70}, 1263--1272 (2017)

\bibitem{hamilton_representation_2017}
W.~L. Hamilton, R.~Ying, and J.~Leskovec: \textit{Representation learning on graphs: {Methods} and applications}.
\newblock IEEE Data Eng. Bull. \textbf{40}, 52--74 (2017)

\bibitem{bronstein_geometric_2017}
M.~M. Bronstein, J.~Bruna, Y.~LeCun, A.~Szlam, and P.~Vandergheynst: \textit{Geometric deep learning: going beyond {Euclidean} data}.
\newblock IEEE Signal Process. Mag. \textbf{34}, 18--42 (2017)

\bibitem{zhang_deep_2018}
Z.~Zhang, P.~Cui, and W.~Zhu: \textit{Deep Learning on Graphs: A Survey}.
\newblock IEEE Trans. Knowl. Data Eng. \textbf{34}, 249--270 (2022)

\bibitem{zhang_graph_2019}
S.~Zhang, H.~Tong, J.~Xu, and R.~Maciejewski: \textit{Graph convolutional networks: a comprehensive review}.
\newblock Comput. Soc. Netw. \textbf{6}, 11 (2019)

\bibitem{bacciu_gentle_2020}
D.~Bacciu, F.~Errica, A.~Micheli, and M.~Podda: \textit{A Gentle Introduction to Deep Learning for Graphs}.
\newblock Neural Netw. \textbf{129}, 203--221 (2020)

\bibitem{Zitnick2020}
C.~L. Zitnick, L.~Chanussot, A.~Das, S.~Goyal, J.~Heras-Domingo et~al.: \textit{An Introduction to Electrocatalyst Design using Machine Learning for Renewable Energy Storage}.
\newblock \eprint{https://arxiv.org/abs/2010.09435} (2020)

\bibitem{Jumper2021}
J.~Jumper, R.~Evans, A.~Pritzel, T.~Green, M.~Figurnov et~al.: \textit{Highly accurate protein structure prediction with AlphaFold}.
\newblock Nature \textbf{596}, 583--589 (2021)

\bibitem{Dauparas2022}
J.~Dauparas, I.~Anishchenko, N.~Bennett, H.~Bai, R.~J. Ragotte et~al.: \textit{Robust deep learning–based protein sequence design using ProteinMPNN}.
\newblock Science \textbf{378}, 49--56 (2022)

\bibitem{Batzner2022}
S.~Batzner, A.~Musaelian, L.~Sun, M.~Geiger, J.~P. Mailoa et~al.: \textit{E(3)-equivariant graph neural networks for data-efficient and accurate interatomic potentials}.
\newblock Nat. Commun. \textbf{13}, 2453 (2022)

\bibitem{Duval2024}
A.~Duval, S.~V. Mathis, C.~K. Joshi, V.~Schmidt, S.~Miret et~al.: \textit{A Hitchhiker's Guide to Geometric GNNs for 3D Atomic Systems}.
\newblock \eprint{https://arxiv.org/abs/2312.07511} (2024)

\bibitem{Langer2022}
M.~F. Langer, A.~Goe{\ss}mann, and M.~Rupp: \textit{Representations of molecules and materials for interpolation of quantum-mechanical simulations via machine learning}.
\newblock npj Comput. Mater. \textbf{8}, 41 (2022)

\bibitem{Schuett2021}
K.~T. Sch{\"u}tt, O.~T. Unke, and M.~Gastegger: \textit{Equivariant message passing for the prediction of tensorial properties and molecular spectra}.
\newblock Int. Conf. Mach. Learn. \textbf{139}, 9377--9388 (2021)

\bibitem{Haghighatlari2022}
M.~Haghighatlari, J.~Li, X.~Guan, O.~Zhang, A.~Das et~al.: \textit{NewtonNet: a Newtonian message passing network for deep learning of interatomic potentials and forces}.
\newblock Digital Discovery \textbf{1}, 333--343 (2022)

\bibitem{Simeon2023}
G.~Simeon and G.~D. Fabritiis: \textit{TensorNet: Cartesian Tensor Representations for Efficient Learning of Molecular Potentials}.
\newblock Adv. Neural Inf. Process. Syst. \textbf{36}, 37334--37353 (2023)

\bibitem{Cheng2024}
B.~Cheng: \textit{Cartesian atomic cluster expansion for machine learning interatomic potentials}.
\newblock npj Comput. Mater. \textbf{10}, 157 (2024)

\bibitem{Batatia2022design}
I.~Batatia, S.~Batzner, D.~P. Kovács, A.~Musaelian, G.~N.~C. Simm et~al.: \textit{{The Design Space of E(3)-Equivariant Atom-Centered Interatomic Potentials}}.
\newblock \eprint{https://arxiv.org/abs/2205.06643} (2022)

\bibitem{Batatia2022}
I.~Batatia, D.~P. Kovacs, G.~N.~C. Simm, C.~Ortner, and G.~Csanyi: \textit{{{MACE}: Higher Order Equivariant Message Passing Neural Networks for Fast and Accurate Force Fields}}.
\newblock Adv. Neural Inf. Process. Syst. \textbf{35}, 11423--11436 (2022)

\bibitem{Musaelian2023}
A.~Musaelian, S.~Batzner, A.~Johansson, L.~Sun, C.~J. Owen et~al.: \textit{Learning local equivariant representations for large-scale atomistic dynamics}.
\newblock Nat. Commun. \textbf{14}, 579 (2023)

\bibitem{Passaro2023}
S.~Passaro and C.~L. Zitnick: \textit{{Reducing SO(3) Convolutions to SO(2) for Efficient Equivariant GNNs}}.
\newblock Int. Conf. Mach. Learn. \textbf{202}, 27420--27438 (2023)

\bibitem{Luo2024}
S.~Luo, T.~Chen, and A.~S. Krishnapriyan: \textit{Enabling Efficient Equivariant Operations in the Fourier Basis via Gaunt Tensor Products}.
\newblock Int. Conf. Learn. Represent. \eprint{https://arxiv.org/abs/2401.10216} (2024)

\bibitem{Thoelke2022}
P.~Th{\"o}lke and G.~D. Fabritiis: \textit{Equivariant Transformers for Neural Network based Molecular Potentials}.
\newblock Int. Conf. Learn. Represent. \eprint{https://arxiv.org/abs/2202.02541} (2022)

\bibitem{Snider2018}
R.~F. Snider: \textit{Irreducible Cartesian Tensors}.
\newblock De Gruyter, Berlin, Boston (2018)

\bibitem{Thomas2018}
N.~Thomas, T.~Smidt, S.~Kearnes, L.~Yang, L.~Li et~al.: \textit{Tensor field networks: Rotation- and translation-equivariant neural networks for 3D point clouds}.
\newblock \eprint{https://arxiv.org/abs/1802.08219} (2018)

\bibitem{Drautz2019}
R.~Drautz: \textit{Atomic cluster expansion for accurate and transferable interatomic potentials}.
\newblock Phys. Rev. B \textbf{99}, 014104 (2019)

\bibitem{Wigner2012}
E.~Wigner: \textit{Group Theory and its Application to the Quantum Mechanics of Atomic Spectra}, volume~5.
\newblock Elsevier (2012)

\bibitem{Coope1965}
J.~A.~R. Coope, R.~F. Snider, and F.~R. McCourt: \textit{{Irreducible Cartesian Tensors}}.
\newblock J. Chem. Phys. \textbf{43}, 2269--2275 (1965)

\bibitem{Coope1970}
J.~A.~R. Coope and R.~F. Snider: \textit{{Irreducible Cartesian Tensors. II. General Formulation}}.
\newblock J. Math. Phys. \textbf{11}, 1003--1017 (1970)

\bibitem{Coope1970_2}
J.~A.~R. Coope: \textit{{Irreducible Cartesian Tensors. III. Clebsch‐Gordan Reduction}}.
\newblock J. Math. Phys. \textbf{11}, 1591--1612 (1970)

\bibitem{Lehman1989}
D.~R. Lehman and W.~C. Parke: \textit{{Angular reduction in multiparticle matrix elements}}.
\newblock J. Math. Phys. \textbf{30}, 2797--2806 (1989)

\bibitem{Weiler2018}
M.~Weiler, M.~Geiger, M.~Welling, W.~Boomsma, and T.~S. Cohen: \textit{3D Steerable CNNs: Learning Rotationally Equivariant Features in Volumetric Data}.
\newblock Adv. Neural Inf. Process. Syst. \textbf{31} (2018)

\bibitem{Christensen2020b}
A.~S. Christensen and O.~A. von Lilienfeld: \textit{On the role of gradients for machine learning of molecular energies and forces}.
\newblock Mach. Learn.: Sci. Technol. \textbf{1}, 045018 (2020)

\bibitem{Chmiela2023}
S.~Chmiela, V.~Vassilev-Galindo, O.~T. Unke, A.~Kabylda, H.~E. Sauceda et~al.: \textit{Accurate global machine learning force fields for molecules with hundreds of atoms}.
\newblock Sci. Adv. \textbf{9}, eadf0873 (2023)

\bibitem{Kovacs2021}
D.~P. Kovács, C.~v.~d. Oord, J.~Kucera, A.~E.~A. Allen, D.~J. Cole et~al.: \textit{Linear Atomic Cluster Expansion Force Fields for Organic Molecules: Beyond RMSE}.
\newblock J. Chem. Theory Comput. \textbf{17}, 7696--7711 (2021)

\bibitem{Gubaev2023}
K.~Gubaev, V.~Zaverkin, P.~Srinivasan, A.~I. Duff, J.~K{\"a}stner et~al.: \textit{Performance of two complementary machine-learned potentials in modelling chemically complex systems}.
\newblock npj Comput. Mater. \textbf{9}, 129 (2023)

\bibitem{Drautz2020}
R.~Drautz: \textit{Atomic cluster expansion of scalar, vectorial, and tensorial properties including magnetism and charge transfer}.
\newblock Phys. Rev. B \textbf{102}, 024104 (2020)

\bibitem{Anderson2019}
B.~Anderson, T.~S. Hy, and R.~Kondor: \textit{Cormorant: Covariant Molecular Neural Networks}.
\newblock Adv. Neural Inf. Process. Syst. \textbf{32}, 14537--14546 (2019)

\bibitem{Satorras2022}
V.~G. Satorras, E.~Hoogeboom, and M.~Welling: \textit{E(n) Equivariant Graph Neural Networks}.
\newblock Int. Conf. Mach. Learn. \textbf{139}, 9323--9332 (2021)

\bibitem{Brandstetter2022}
J.~Brandstetter, R.~Hesselink, E.~van~der Pol, E.~J. Bekkers, and M.~Welling: \textit{Geometric and Physical Quantities improve E(3) Equivariant Message Passing}.
\newblock Int. Conf. Learn. Represent. \eprint{https://arxiv.org/abs/2110.02905} (2022)

\bibitem{Schuett2017}
K.~Sch\"{u}tt, P.-J. Kindermans, H.~E. Sauceda~Felix, S.~Chmiela, A.~Tkatchenko et~al.: \textit{SchNet: A continuous-filter convolutional neural network for modeling quantum interactions}.
\newblock Adv. Neural Inf. Process. Syst. \textbf{30}, 991--1001 (2017)

\bibitem{Unke2019}
O.~T. Unke and M.~Meuwly: \textit{PhysNet: A Neural Network for Predicting Energies, Forces, Dipole Moments, and Partial Charges}.
\newblock J. Chem. Theory Comput. \textbf{15}~(6), 3678--3693 (2019)

\bibitem{Liu2022}
Y.~Liu, L.~Wang, M.~Liu, Y.~Lin, X.~Zhang et~al.: \textit{Spherical Message Passing for 3D Molecular Graphs}.
\newblock Int. Conf. Learn. Represent. \eprint{https://arxiv.org/abs/2102.05013} (2022)

\bibitem{Gasteiger2022b}
J.~Gasteiger, S.~Giri, J.~T. Margraf, and S.~G{\"u}nnemann: \textit{Fast and Uncertainty-Aware Directional Message Passing for Non-Equilibrium Molecules}.
\newblock Machine Learning for Molecules Workshop, Adv. Neural Inf. Process. Syst. \eprint{https://arxiv.org/abs/2011.14115} (2020)

\bibitem{Gasteiger2022a}
J.~Gasteiger, F.~Becker, and S.~G\"{u}nnemann: \textit{GemNet: Universal Directional Graph Neural Networks for Molecules}.
\newblock Adv. Neural Inf. Process. Syst. \textbf{34}, 6790--6802 (2021)

\bibitem{Gasteiger2022c}
J.~Gasteiger, J.~Groß, and S.~Günnemann: \textit{Directional Message Passing for Molecular Graphs}.
\newblock Int. Conf. Learn. Represent. \eprint{https://arxiv.org/abs/2003.03123} (2020)

\bibitem{Kovacs2023b}
D.~P. Kovács, I.~Batatia, E.~S. Arany, and G.~Csányi: \textit{{Evaluation of the MACE force field architecture: From medicinal chemistry to materials science}}.
\newblock J. Chem. Phys. \textbf{159}, 044118 (2023)

\bibitem{Fuchs2020}
F.~Fuchs, D.~Worrall, V.~Fischer, and M.~Welling: \textit{SE(3)-Transformers: 3D Roto-Translation Equivariant Attention Networks}.
\newblock Adv. Neural Inf. Process. Syst. \textbf{33}, 1970--1981 (2020)

\bibitem{Frank2023}
T.~Frank, O.~Unke, and K.-R. M\"{u}ller: \textit{{So3krates: Equivariant attention for interactions on arbitrary length-scales in molecular systems}}.
\newblock Adv. Neural Inf. Process. Syst. \textbf{35}, 29400--29413 (2022)

\bibitem{Liao2023}
Y.-L. Liao, B.~Wood, A.~Das, and T.~Smidt: \textit{{EquiformerV2: Improved Equivariant Transformer for Scaling to Higher-Degree Representations}}.
\newblock Int. Conf. Learn. Represent. \eprint{https://arxiv.org/abs/2306.12059} (2023)

\bibitem{Shapeev2016}
A.~V. Shapeev: \textit{Moment Tensor Potentials: A Class of Systematically Improvable Interatomic Potentials}.
\newblock Multiscale Model. Simul. \textbf{14}, 1153--1173 (2016)

\bibitem{Zaverkin2020}
V.~Zaverkin and J.~K{\"{a}}stner: \textit{{Gaussian Moments as Physically Inspired Molecular Descriptors for Accurate and Scalable Machine Learning Potentials}}.
\newblock J. Chem. Theory Comput. \textbf{16}, 5410--5421 (2020)

\bibitem{Zaverkin2021b}
V.~Zaverkin, D.~Holzm{\"{u}}ller, I.~Steinwart, and J.~K{\"{a}}stner: \textit{{Fast and Sample-Efficient Interatomic Neural Network Potentials for Molecules and Materials Based on Gaussian Moments}}.
\newblock J. Chem. Theory Comput. \textbf{17}, 6658--6670 (2021)

\bibitem{Behler2007}
J.~Behler and M.~Parrinello: \textit{{Generalized Neural-Network Representation of High-Dimensional Potential-Energy Surfaces}}.
\newblock Phys. Rev. Lett. \textbf{98}, 146401 (2007)

\bibitem{Fano1959}
U.~Fano and G.~Racah: \textit{Irreducible Tensorial Sets}.
\newblock Academic Press Inc., New York (1959)

\bibitem{Gelfand1963}
I.~M. Gel'fand, R.~A. Minlos, and Z.~Y. Shapiro: \textit{Representations of the Rotation and Lorentz Groups and Their Applications}.
\newblock Pergamon Press, Inc. (1963)

\bibitem{Mueser2023}
S.~V.~S. Martin H.~Müser and L.~Pastewka: \textit{Interatomic potentials: achievements and challenges}.
\newblock Adv. Phys. X \textbf{8}, 2093129 (2023)

\bibitem{Grega2024}
I.~Grega, I.~Batatia, G.~Csanyi, S.~Karlapati, and V.~Deshpande: \textit{Energy-conserving equivariant {GNN} for elasticity of lattice architected metamaterials}.
\newblock Int. Conf. Learn. Represent. \eprint{https://arxiv.org/abs/2401.16914} (2024)

\bibitem{Joshi2023}
C.~K. Joshi, C.~Bodnar, S.~V. Mathis, T.~Cohen, and P.~Lio: \textit{On the Expressive Power of Geometric Graph Neural Networks}.
\newblock {Int. Conf. Learn. Represent.} \eprint{https://arxiv.org/abs/2301.09308} (2023)

\bibitem{Pozdnyakov2020}
S.~N. Pozdnyakov, M.~J. Willatt, A.~P. Bart\'ok, C.~Ortner, G.~Cs\'anyi et~al.: \textit{Incompleteness of Atomic Structure Representations}.
\newblock Phys. Rev. Lett. \textbf{125}, 166001 (2020)

\bibitem{Behler2011a}
J.~Behler: \textit{Atom-centered symmetry functions for constructing high-dimensional neural network potentials}.
\newblock J. Chem. Phys. \textbf{134}, 074106 (2011)

\bibitem{Bartok2010}
A.~P. Bart\'ok, M.~C. Payne, R.~Kondor, and G.~Cs\'anyi: \textit{Gaussian Approximation Potentials: The Accuracy of Quantum Mechanics, without the Electrons}.
\newblock Phys. Rev. Lett. \textbf{104}, 136403 (2010)

\bibitem{Bartok2013}
A.~P. Bart\'ok, R.~Kondor, and G.~Cs\'anyi: \textit{On representing chemical environments}.
\newblock Phys. Rev. B \textbf{87}, 184115 (2013)

\bibitem{Daw1984}
M.~S. Daw and M.~I. Baskes: \textit{{Embedded-atom method: Derivation and application to impurities, surfaces, and other defects in metals}}.
\newblock Phys. Rev. B \textbf{29}, 6443--6453 (1984)

\bibitem{Kim2006}
Y.-M. Kim, B.-J. Lee, and M.~I. Baskes: \textit{{Modified embedded-atom method interatomic potentials for Ti and Zr}}.
\newblock Phys. Rev. B \textbf{74}, 014101--014112 (2006)

\bibitem{Stone1975}
A.~Stone: \textit{Transformation between cartesian and spherical tensors}.
\newblock Mol. Phys. \textbf{29}, 1461--1471 (1975)

\bibitem{Stone1976}
A.~J. Stone: \textit{Properties of Cartesian-spherical transformation coefficients}.
\newblock J. Phys. A \textbf{9}, 485 (1976)

\bibitem{Normand1982}
J.~M. Normand and J.~Raynal: \textit{Relations between Cartesian and spherical components of irreducible Cartesian tensors}.
\newblock J. Phys. A \textbf{15}, 1437 (1982)

\bibitem{He2015b}
K.~He, X.~Zhang, S.~Ren, and J.~Sun: \textit{Deep Residual Learning for Image Recognition}.
\newblock IEEE Conf. Comput. Vis. Pattern Recognit. \eprint{https://doi.org/10.1109/CVPR.2016.90} (2016)

\bibitem{Chmiela2017}
S.~Chmiela, A.~Tkatchenko, H.~E. Sauceda, I.~Poltavsky, K.~T. Sch{\"u}tt et~al.: \textit{Machine learning of accurate energy-conserving molecular force fields}.
\newblock Sci. Adv. \textbf{3}, e1603015 (2017)

\bibitem{Schuett2017_2}
K.~T. Sch{\"u}tt, F.~Arbabzadah, S.~Chmiela, K.~R. M{\"u}ller, and A.~Tkatchenko: \textit{Quantum-chemical insights from deep tensor neural networks}.
\newblock Nat. Commun. \textbf{8}, 13890 (2017)

\bibitem{Chmiela2018}
S.~Chmiela, H.~E. Sauceda, K.-R. M{\"u}ller, and A.~Tkatchenko: \textit{Towards exact molecular dynamics simulations with machine-learned force fields}.
\newblock Nat. Commun. \textbf{9}, 3887 (2018)

\bibitem{Elfwing2018}
S.~Elfwing, E.~Uchibe, and K.~Doya: \textit{{Sigmoid-Weighted Linear Units for Neural Network Function Approximation in Reinforcement Learning}}.
\newblock Neural Netw. \textbf{107}, 3--11 (2018)

\bibitem{Ramachandran2018}
P.~Ramachandran, B.~Zoph, and Q.~V. Le: \textit{Searching for activation functions}.
\newblock Int. Conf. Learn. Represent. \eprint{https://arxiv.org/abs/1710.05941} (2018)

\bibitem{Reddi2018}
S.~J. Reddi, S.~Kale, and S.~Kumar: \textit{On the Convergence of Adam and Beyond}.
\newblock Int. Conf. Learn. Represent. \eprint{https://arxiv.org/abs/1904.09237} (2018)

\bibitem{Li2024}
Y.~Li, Y.~Wang, L.~Huang, H.~Yang, X.~Wei et~al.: \textit{Long-Short-Range Message-Passing: A Physics-Informed Framework to Capture Non-Local Interaction for Scalable Molecular Dynamics Simulation}.
\newblock Int. Conf. Learn. Represent. \eprint{https://arxiv.org/abs/2304.13542} (2024)

\bibitem{Frank2024}
J.~T. Frank, O.~T. Unke, K.-R. M{\"u}ller, and S.~Chmiela: \textit{A Euclidean transformer for fast and stable machine learned force fields}.
\newblock Nat. Commun. \textbf{15}, 6539 (2024)

\end{thebibliography}

%%%%%%%%%%%%%%%%%%%%%%%%%%%%%%%%%%%%%%%%%%%%%%%%%%%%%%%%%%%%

\newpage

\begin{appendices}

\listofappendices

\setcounter{equation}{0}
\renewcommand{\theequation}{A\arabic{equation}}

\setcounter{figure}{0}
\renewcommand{\thefigure}{A\arabic{figure}}

\setcounter{table}{0}
\renewcommand{\thetable}{A\arabic{table}}

\newpage

\section{Background \label{sec:background_appendix}}

\textbf{Message-passing neural networks.} Message-passing neural networks (MPNNs) learn node representations in a graph by iteratively processing local information sent by the nodes' neighbors. The initial features of node $u$ are represented as the vector $\mathbf{x}_u$, and undirected edges $\{u, v\}$ connect pairs of nodes $u,v$. A node $v$ belongs to the neighborhood of node $u$, denoted as $\mathcal{N}(u)$, if there exists an edge $\{u, v\}$ in the graph. Typically, the $(t+1)$-th message-passing layer computes a new node $u$'s representation $\mathbf{h}^{(t+1)}_u$ by applying a permutation invariant aggregation function over the neighbors~\citep{Gilmer2017, bacciu_gentle_2020}
\begin{equation*}
  \mathbf{h}_u^{(t+1)} = \phi^{(t)}\Big(\mathbf{h}^{(t)}_u, \sum\nolimits_{v \in \mathcal{N}(u)} \psi^{(t)}\big(\mathbf{h}^{(t)}_u, \mathbf{h}^{(t)}_v\big)\Big),
\end{equation*}
where $\phi^{(t)}, \psi^{(t)}$ are often implemented as learnable fully-connected neural networks (NNs), and $\mathbf{h}_u^{(0)}=\mathbf{x}_u$. To learn a mapping from a learned representation $\mathbf{h}_{u}^{(t)}$ to the atoms' energies, we can couple $T$ message-passing layers with corresponding readout functions $\mathcal{R}_t,\,t\in\{1,\dots,T\}$ such that $E_u = \sum_{t=1}^{T} \mathcal{R}_t\big(\mathbf{h}^{(t)}_u\big)$.

\textbf{Many-body interatomic potentials.} Interatomic potentials approximate the potential energy of atoms---the electronic ground state energy---as a function of their coordinates~\cite{Mueser2023}. Many-body potentials naturally arise because the interaction between two atoms is influenced by the presence of additional atoms, changing their electronic structure. This concept is formalized by expanding the atomic energy $E_u$ of a many-atom system into a series of two-body, three-body, and higher-body-order contributions
\begin{equation*}
    E_u = E^{(1)}(\mathbf{r}_u) + \sum_{v_1} E^{(2)}(\mathbf{r}_u, \mathbf{r}_{v_1}) + \sum_{v_1 < v_2} E^{(3)}(\mathbf{r}_u, \mathbf{r}_{v_1}, \mathbf{r}_{v_2}) + \cdots, 
\end{equation*}
where $\mathbf{r}_u$ represents the position of atom $u$ and the superscript $\nu$ in $E^{(\nu)}$ indicates the order of the many-body interaction. In the absence of external fields, $E^{(\nu)}$ contributions to the atomic energy are invariant to rotations, and the two-body potential depends only on distances $r_{uv} = \lVert \mathbf{r}_u - \mathbf{r}_v \rVert_2$
\begin{equation*}
    E_u = \sum_{v_1} E^{(2)}(r_{uv_1}) + \sum_{v_1 < v_2} E^{(3)}(r_{uv_1}, r_{uv_2}, r_{v_1v_2}) + \cdots.
\end{equation*}
Expansions of this form have found a broad application in constructing machine-learned interatomic potentials (MLIPs), significantly advancing the field~\cite{Shapeev2016, Drautz2019}.

\textbf{Higher-body-order local descriptors.} Recent advances in MLIPs have been influenced by moment tensor potentials (MTPs)~\cite{Shapeev2016} and atomic cluster expansion (ACE)~\cite{Drautz2019}. These approaches enable systematic construction of higher-body-order polynomial basis functions, encompassing representations like atom-centered symmetry functions (ACSFs)~\cite{Behler2007, Behler2011a}, smooth overlap of atomic positions (SOAP)~\cite{Bartok2010, Bartok2013}, Gaussian moments~\cite{Zaverkin2020, Zaverkin2021b}, and embedded atom/multi-scale embedded atom method (EAM/MEAM) potentials~\cite{Daw1984, Kim2006}. Furthermore, a reducible Cartesian tensor can be represented as a linear combination of irreducible spherical counterparts, and vice versa~\cite{Drautz2019, Drautz2020}. More general expressions, including tensor contractions, have also been provided, hinting at the relationship between Cartesian and spherical models~\cite{Stone1975, Stone1976, Normand1982}. Despite the success of MTP and ACE, defining smaller cutoff radii and rigid architecture can result in limited accuracy compared to MPNNs.

\textbf{Many-body message passing.} Designing accurate and computationally efficient interatomic potential models for interacting many-body systems requires including higher-body-order energy contributions and, thus, higher-body-order learnable features. Recently, a new message construction mechanism has been proposed by expanding the messages $\mathbf{m}_u^{(t)}$ to include many-body contributions~\cite{Batatia2022}
\begin{equation*}
    \mathbf{m}_u^{(t)} = \sum\limits_{v_1} \psi_2^{(t)}\big(\mathbf{h}^{(t)}_u, \mathbf{h}^{(t)}_{v_1}\big) + \sum\limits_{v_1, v_2} \psi_3^{(t)}\big(\mathbf{h}^{(t)}_u, \mathbf{h}^{(t)}_{v_1}, \mathbf{h}^{(t)}_{v_2}\big) + \cdots + \sum\limits_{v_1, \cdots, v_\nu} \psi_{\nu+1}^{(t)}\big(\mathbf{h}^{(t)}_u, \mathbf{h}^{(t)}_{v_1}, \cdots, \mathbf{h}^{(t)}_{v_\nu}\big),
\end{equation*}
where $(\nu+1)$ denotes the order of many-body interactions, defining the number of contracted tensors. Using $\sum_{v_1, \cdots, v_\nu}$ instead of $\sum_{v_1 < \cdots < v_\nu}$ circumvents the exponential increase in computational cost with $\nu$. It allows exploiting the product structure of many-body features, which differs from other approaches~\cite{Gasteiger2022a, Gasteiger2022b, Gasteiger2022c}.

\section{Methods \label{sec:methods_appendix}}

\subsection{Normalization constants \label{sec:normalization_constants}}

The following provides the normalization constants $C_{l_1l_2l_3}$ and $D_{l_1l_2l_3}$ for even and odd irreducible Cartesian tensor products in Eqs.~(\ref{eq:product_even}) and (\ref{eq:product_odd}), respectively. The respective normalization constants read~\cite{Lehman1989}
\begin{equation*}
  \begin{split}
    C_{l_1l_2l_3} & = \frac{l_1!l_2!(2l_3-1)!!((L_1+1)/2)!((L_2+1)/2)!}{l_3!L_1!!L_2!!L_3!!(L/2)!},\\
    D_{l_1l_2l_3} & = \frac{2l_1!l_2!(2l_3-1)!!(L_1/2)!(L_2/2)!}{(l_3-1)!(L_1+1)!!(L_2+1)!!(L_3+1)!!((L+1)/2)!},
  \end{split}
\end{equation*}
with $L=l_1 + l_2 + l_3$ and $L_i = L - 2l_i - 1$. Here, $C_{l_1l_2l_3}$ is defined such that an $l_3$-fold contraction of the tensor $\mathbf{z}_{l_3}$, obtained through the irreducible Cartesian tensor product between $\mathbf{x}_{l_1}$ and $\mathbf{y}_{l_2}$ ($\mathbf{x}_{l_1}$ and $\mathbf{y}_{l_2}$ are obtained using \eqref{eq:cartesian_irreps}), with the unit vector $\hat{\mathbf{r}}$ yields unity. We refer to the original publication for the motivation behind $D_{l_1l_2l_3}$~\cite{Lehman1989}.

\subsection{Further details on the equivariant message passing \label{sec:architecture_appendix}}

The following provides additional details on the employed equivariant message-passing architecture. Learnable weights employed in Eqs.~(\ref{eq:atomic_basis}), (\ref{eq:atomic_basis_first}), and (\ref{eq:product_basis}) (i.e., $W_{kk^\prime l_2}^{(t)}$, $W_{k Z_v}^{(t)}$, and $W_{kk^\prime l_\xi}^{(t)}$ with $k,k^\prime$ running over feature channels) are initialized by picking the respective entries from a normal distribution with zero mean and unit variance.

\textbf{Many-body message-passing.} The many-body equivariant features, represented by an irreducible Cartesian tensor of rank $L$ and obtained through the $\nu$-fold tensor product in \eqref{eq:product_basis}, are combined using the linear expansion
\begin{equation}
  \label{eq:messages_uncoupled}
  \big(\mathbf{m}_{ukL}^{(t)}\big)_{i_1i_2\cdots i_L} = \Big(\sum_\nu\sum_{\eta_\nu} W_{Z_u\eta_{\nu}kL}^{(t)} \mathbf{B}_{u\eta_{\nu}kL}^{(t)}\Big)_{i_1i_2\cdots i_L},
\end{equation}
where $W_{Z_u\eta_{\nu}kL}^{(t)}$ denotes a learnable weight matrix which depends on the chemical element $Z_u$ and rank $L$ and which elements are initialized by picking the respective entries from a normal distribution with zero mean and a standard deviation of $1/\mathrm{len}(\eta_\nu)$. The updated node embeddings are further obtained as a linear function of $\big(\mathbf{m}_{ukL}^{(t)}\big)_{i_1i_2\cdots i_L}$ and the residual connection~\cite{He2015b, Batatia2022}
\begin{equation}
  \label{eq:update}
    \big(\mathbf{h}_{ukL}^{(t+1)}\big)_{i_1i_2\cdots i_L} = \frac{1}{\sqrt{d_t}} \sum_{k^\prime} W_{kk^\prime L}^{(t)}\big(\mathbf{m}_{uk^\prime L}^{(t)}\big)_{i_1i_2\cdots i_L} + \frac{1}{\sqrt{d_t N_Z}} \sum_{k^\prime} W_{Z_ukk^\prime L}^{(t)}\big(\mathbf{h}_{uk^\prime L}^{(t)}\big)_{i_1i_2\cdots i_L},
\end{equation}
where $N_Z$ denotes the number of atom types and learnable weights $ W_{kk^\prime L}^{(t)}$ and $W_{Z_ukk^\prime L}^{(t)}$ are initialized by picking the respective entries from a normal distribution with zero mean and unit variance.

\textbf{Full and symmetric product basis.} In the MACE architecture, the authors pre-compute products between the generalized Clebsch--Gordan coefficients, which define the interactions of $\{l_1,\cdots,l_\nu\}$-rank spherical tensors, and the learnable weight $W_{Z_u\eta_{\nu}kL}^{(t)}$ of the linear expansion in \eqref{eq:messages_uncoupled}~\cite{Batatia2022}. This approach reduces the computational cost of constructing the product basis with spherical tensors as the effective number of evaluated tensor products is smaller by $\mathcal{K} = \mathrm{len}\left(\eta_\nu\right)$. For irreducible Cartesian tensors, operations like matrix-vector and matrix-matrix products define the interactions between $\{l_1,\cdots,l_\nu\}$-rank tensors. Thus, an equivalent operation to those proposed for spherical tensors may be generally impossible for irreducible Cartesian tensors or would lead to an architecture different from MACE.

Note that one of our main goals is to demonstrate that a many-body equivariant message-passing architecture defined using higher-rank irreducible Cartesian tensors can be as expressive as the one using spherical tensors. Therefore, we compute  $\mathbf{m}_{ukL}^{(t)}$ by iteratively evaluating all possible $\nu$-fold tensor products, which effectively leads to a larger number of operations than for MACE. We refer to the models based on this architecture as irreducible Cartesian tensor potentials (ICTPs) with the full product basis, or ICTP$_\mathrm{full}$. We also noticed that the number of tensor products $\mathrm{len}(\eta_\nu)$ leading to the tensor of rank $L$, can be reduced by the symmetry of the two-fold tensor product in \eqref{eq:product_basis}, i.e., $\mathbf{A}_{l_1} \otimes \mathbf{A}_{l_2} = \mathbf{A}_{l_2} \otimes \mathbf{A}_{l_1}$ if $\mathbf{A}_{l_1}$ and $\mathbf{A}_{l_2}$ are symmetric. We refer to this design choice as ICTPs with the symmetric product basis or ICTP$_\mathrm{sym}$.

\textbf{Coupled feature channels.} Using coupled feature channels instead of uncoupled ones in \eqref{eq:messages_uncoupled} can improve the performance of the final model. Additionally, the number of feature channels and, thus, the overall computational cost can be reduced when constructing the product basis. However, the number of parameters and, thus, the expressive power of the model can be preserved. Specifically, we can define a linear expansion for  combining many-body equivariant features as
\begin{equation}
  \label{eq:messages_coupled}
  \big(\mathbf{m}_{ukL}^{(t)}\big)_{i_1i_2\cdots i_L} = \Big(\sum_\nu\sum_{\eta_\nu} \sum_{k^\prime} W_{Z_u\eta_{\nu}kk^\prime L}^{(t)} \mathbf{B}_{u\eta_{\nu}k^\prime L}^{(t)}\Big)_{i_1i_2\cdots i_L},
\end{equation}
where $W_{Z_u\eta_{\nu}kk^\prime L}^{(t)}$ denotes a learnable weight matrix which depends on the chemical element $Z_u$ and rank $L$ and which elements are initialized by picking the respective entries from a normal distribution with zero mean and a standard deviation of $1/(\sqrt{d_t} \times \mathrm{len}(\eta_\nu))$. The construction of the product basis in the latent feature space requires encoding the two-body features $\mathbf{A}_{ukl_\xi}^{(t)}$ using the learnable weight matrices in \eqref{eq:product_basis}. The linear expansion is then constructed in the latent feature space and decoded using the learnable weight matrices of the linear function in \eqref{eq:update}. We refer to this design choice as ICTPs with the product basis constructed in the latent feature space or ICTP$_\mathrm{lt}$. Combining it with the symmetric product basis, we obtain ICTP$_\mathrm{sym+lt}$.

\textbf{Readout.} Atomic energies expanded into a series of many-body contributions, $E_u = E_u^{(0)} + E_u^{(1)} + \cdots + E_u^{(T)}$, are obtained by applying readout functions $\mathcal{R}_t$ to node features with $L=0$, which are invariant to rotations
\begin{equation}
    \label{eq:readout}
    E_u^{(t)} = \mathcal{R}_t\Big(\big\{\mathbf{h}_{uk(L=0)}^{(t)}\big\}_k\Big) = \begin{cases}
      \frac{1}{\sqrt{d_t}} \sum_{k^\prime} W_{k^\prime}^{(t)} h_{uk^\prime (L=0)}^{(t)} & \text{if } t < T, \\
      \mathrm{NN}^{(t)}\Big(\big\{\mathbf{h}_{uk(L=0)}^{(t)}\big\}_k\Big) & \text{if } t = T,
    \end{cases}
\end{equation}
with learnable weights $W_{k^\prime}^{(t)}$ or those of $\mathrm{NN}^{(t)}$ initialized by picking the respective entries from a normal distribution with zero mean and unit variance. Linear readout functions for $t < T$ preserve the many-body orders in $\mathbf{h}_{uk(L=0)}^{(t)}$, while a one-layer fully-connected NN is used for the last message-passing layer and accounts for the residual higher-order terms in the expansion~\cite{Batatia2022design}.

\subsection{Computational complexity and runtime analysis \label{sec:comput_complexity}} 

We analyze the computational complexity of the irreducible Cartesian tensor product with respect to the maximum rank $L$ used. For each tuple of $\left(i_1, \cdots, i_L\right)$ of a rank-$L$ tensor in \eqref{eq:product_even}, the single contraction term $\left(\mathbf{x}_{L}\cdot(k+m)\cdot\mathbf{y}_{L}\right)$ has a cost of $3^{k+m}$. The total computational cost is $3^L$ for even tensor products and $3^{L+1}$ for odd ones---due to the additional double contraction with the Levi-Civita symbol. The computation of the set of permutations over the $L$ unsymmetrized indices scales as $L!/\left(2^{L/2}\left(L/2\right)!\right)$ for even tensor products and $L!/\left(2^{\left(L-1\right)/2}\left(\left(L-1\right)/2\right)!\right)$ for odd ones~\cite{Lehman1989}. Because each final rank-$L$ tensor has $3^{L}$ elements, the complexity for computing an irreducible Cartesian tensor of rank $L$ is $\mathcal{O}\left(9^{L} L!/\left(2^{L/2}\left(L/2\right)!\right)\right)$. This expression is used throughout this work as it captures contributions from tensor contractions and index permutations, though it simplifies asymptotically to $\mathcal{O}\left(L^L\right)$ using Stirling's approximation for factorials.

The number of calculations required to obtain a rank-$L$ spherical tensor through the Clebsch--Gordan tensor product is $\left(2L+1\right)^5$. Thus, the computational complexity of the Clebsch--Gordan tensor product is $\mathcal{O}\left(L^5\right)$. While the Clebsch--Gordan tensor product is more computationally efficient than the irreducible Cartesian tensor product for $L \rightarrow \infty$, the latter is expected to be advantageous for smaller tensor ranks, i.e., $L \leq 4$, assuming similar multiplicative factors and negligible sub-leading terms. Tensors of rank $L \leq 4$ are particularly relevant for equivariant message-passing architectures and representing physical properties~\cite{Batatia2022, Batatia2022design, Grega2024}. Compared to the Gaunt tensor product with the computational complexity of $\mathcal{O}\left(L^3\right)$, including all $\left(l_1, l_2\right) \rightarrow l_3$~\cite{Luo2024}, the irreducible Cartesian tensor product may be less advantageous for $L \geq 3$. However, the Gaunt-coefficients-based approach excludes odd tensor products, limiting the expressive power and the range of possible architectures.

As regards the cost of performing message passing at each layer, the general neighbors' aggregation of \eqref{eq:atomic_basis} has a cost of $\mathcal{E} N_\mathrm{ch} 9^{L} L!/\left(2^{L/2}\left(L/2\right)!\right)$, where $\mathcal{E}$ is the number of edges in the atomic system and $N_\mathrm{ch}$ is the number of feature channels. Computing many-body features via \eqref{eq:product_basis} requires to iteratively perform $\nu-1$ Cartesian tensor products for all $\mathcal{K} = \mathrm{len}(\eta_\nu)$ possible $\nu$-fold tensor products, resulting in a computational cost of $\mathcal{M} N_\mathrm{ch} \mathcal{K} (9^{L} L!/(2^{L/2}(L/2)!))^{\nu-1}$ with $\mathcal{M}$ denoting the number of nodes. For the Clebsch--Gordan tensor product, the corresponding number of calculations is $\mathcal{E}N_\text{ch}L^5$ and $\mathcal{M} N_\text{ch} \mathcal{K} L^{5(\nu-1)}$. Spherical models, however, can use generalized Clebsch--Gordan coefficients, resulting in $\mathcal{M} N_\text{ch} \mathcal{K} L^{\frac{1}{2}\nu(\nu + 3)}$ for the product basis. We can remove the factor $\mathcal{K}$ from the above expression by restricting the parameterization to uncoupled feature channels as in \eqref{eq:messages_uncoupled} and obtain $\mathcal{M} N_\text{ch} L^{\frac{1}{2}\nu(\nu + 3)}$. Thus, MACE with this specific choice of the product basis can be more computationally efficient than ICTP only for large $N_\text{ch}$ and small $\nu$, given $L \leq 4$. In this work, we have shown that leveraging the symmetry of the irreducible Cartesian tensor product and coupled feature channels can improve the computational cost of ICTP. The former, in particular, reduces the effective number of $\nu$-fold tensor products $\mathcal{K}$. Further optimization of the architecture, however, is possible and is expected to fully exploit the advantage of Cartesian tensors.

Finally, memory consumption for Cartesian tensors is often believed to be less advantageous than that of spherical tensors. However, for contracting two rank-$L$ spherical tensors, the memory requirement is about $(2L+1)^2$, resulting from the definition of the Clebsch--Gordan tensor product in \secref{sec:background}. For a Cartesian tensor, the memory requirement is about $3^L$. Thus, the irreducible Cartesian tensor product can be expected to require the same or less memory for $L \leq 4$ than the Clebsch--Gordan counterpart. The memory consumption of models based on spherical tensors is further increased by the use of generalized Clebsch--Gordan coefficients, as we demonstrate in \secref{sec:results}.

\section{Proof of Cartesian message-passing equivariance to actions of the orthogonal group \label{sec:proof_equivariance}}

We first recap the standard action of the $\mathrm{O}(3)$ group onto $\left(\mathbb{R}^{3}\right)^{\otimes l}$. Let $D_{\mathcal{X}}$ be a representation of the orthogonal group $\mathrm{O}(3)$ (also in line with \secref{sec:background}), i.e., 
\begin{equation*}
    D_{\mathcal{X}}: \mathrm{O}(3) \rightarrow \mathbb{R}^{3 \times 3}, \ \ \ g \mapsto D_{\mathcal{X}}[g] = R.
\end{equation*}
We define the action of the $\mathrm{O}(3)$ group onto $\left(\mathbb{R}^{3}\right)^{\otimes l}$ as follows
\begin{equation*}
    \phi : \mathrm{O}(3) \times \left(\mathbb{R}^{3}\right)^{\otimes l} \rightarrow  \left(\mathbb{R}^{3}\right)^{\otimes l}, \ \ \ \phi\left(g, \mathbf{T}_l\right)_{i_{1}i_{2}\cdots i_{l}} \coloneqq \sum_{j_{1}} \cdots \sum_{j_{l}} R_{i_{1} j_{1}} \cdots R_{i_{l} j_{l}} \left(\mathbf{T}_l\right)_{j_{1}\cdots j_{l}},
\end{equation*}
where $\mathbf{T}_l \in \left(\mathbb{R}^{3}\right)^{\otimes l}$. We hereinafter denote $\phi\left(g, \mathbf{T}_l\right)$ also by $R \mathbf{T}_l$. The outer product $\otimes$ is equivariant to this action, and  $R$ acts trivially on the $3 \times 3$-identity matrix.

\begin{lemma} 
\label{lem:equivariance}
Let $R$ be a representation of an element of the orthogonal group $\mathrm{O}(3)$. Then, 
\begin{itemize}
    \item[(i)] $\forall\, \mathbf{T}_{l_1} \in \left(\mathbb{R}^{3}\right)^{\otimes l_1}$ and $\forall\, \mathbf{T}_{l_2} \in \left(\mathbb{R}^{3}\right)^{\otimes l_2}$
    \begin{equation*}
        R\left(\mathbf{T}_{l_1} \otimes \mathbf{T}_{l_2}\right) = \left(R\mathbf{T}_{l_1}\right) \otimes \left(R\mathbf{T}_{l_2}\right).
    \end{equation*}
    \item[(ii)] For the $3 \times 3$-identity matrix $\mathbf{I}$, we have 
    \begin{equation*}
        R \mathbf{I} = \mathbf{I}.
    \end{equation*}
\end{itemize}
\end{lemma}

\begin{proof} (i) We first show that $R\left(\mathbf{T}_{l_1} \otimes \mathbf{T}_{l_2}\right) = \left(R\mathbf{T}_{l_1}\right) \otimes \left(R\mathbf{T}_{l_2}\right)$
\\
\scalebox{0.875}{
\begin{minipage}{\linewidth}
    \begin{equation*}
        \begin{split}
            & \left(\left(R \mathbf{T}_{l_1}\right) \otimes \left(R \mathbf{T}_{l_2}\right)\right)_{i_{1} \cdots i_{l_1} i_{l_1+1} \cdots i_{l_1+l_2}} \\
            & \quad\quad = \left(R \mathbf{T}_{l_1}\right)_{i_{1} \cdots i_{l_1}} \left(R \mathbf{T}_{l_2}\right)_{i_{l_1+1} \cdots i_{l_1+l_2}} \\
            & \quad\quad = \left(\sum_{j_{1}, \dots, j_{l_1}} R_{i_{1} j_{1}} \cdots R_{i_{l} j_{l_1}} \left(\mathbf{T}_{l_1}\right)_{j_{1} \cdots j_{l_1}}\right) \left(\sum_{j_{l_1+1}, \dots, j_{l_1+l_2}} R_{i_{l_1+1} j_{l_1+1}} \cdots R_{i_{l_1+l_2} j_{l_1+l_2}} \left(\mathbf{T}_{l_2}\right)_{j_{l_1+1} \cdots j_{l_1+l_2}}\right) \\
            & \quad\quad = \sum_{j_{1}} \cdots \sum_{j_{l_1}} \sum_{j_{l_1+1}} \cdots \sum_{j_{l_1+l_2}} R_{i_{1} j_{1}} \cdots R_{i_{l_1} j_{l_1}} R_{i_{l_1+1} j_{l_1+1}} \cdots R_{i_{l_1+l_2} j_{l_1+l_2}} \left(\mathbf{T}_{l_1}\right)_{j_{1} \cdots j_{l_1}} \left(\mathbf{T}_{l_2}\right)_{j_{l_1+1} \cdots j_{l_1+l_2}}\\
            & \quad\quad = \left(R \left(\mathbf{T}_{l_1} \otimes \mathbf{T}_{l_2}\right)\right)_{i_{1} \cdots i_{l_1} i_{l_1+1} \cdots i_{l_1+l_2}}.
        \end{split}
    \end{equation*}
\end{minipage}}

(ii) With $\sum_{i}R_{ij}R_{ik} = \delta_{jk}$, or equivalently $R^TR = \mathbf{I}$, we get
\begin{equation*}
(R \mathbf{I})_{i_{1}i_{2}} = \sum_{j_{1}, j_{2}} R_{i_{1}j_{1}} R_{i_{2}j_{2}} \delta_{j_{1}j_{2}} = \sum_{j_{1}} R_{i_{1}j_{1}} R_{i_{2}j_{1}} = \delta_{i_{1}i_{2}}.
\end{equation*} 
\end{proof}
Using Lemma \ref{lem:equivariance}, we can show that the irreducible Cartesian tensor operator $\mathbf{T}_{l}: \mathbb{R}^{3} \rightarrow (\mathbb{R}^{3})^{\otimes l}$ is equivariant to actions of the $\mathrm{O}(3)$ group.

\begin{proposition}
\label{prop:equiv_tensor}
Let $l \geq 0$ be a positive integer. Then, the irreducible Cartesian tensor operator $\mathbf{T}_{l}: \mathbb{R}^{3} \rightarrow (\mathbb{R}^{3})^{\otimes l}$ is equivariant to actions of the $\mathrm{O}(3)$ group.
\end{proposition}

\begin{proof} It suffices to show that the map $\hat{\mathbf{r}} \mapsto \hat{\mathbf{r}}^{\otimes (l-2m)}\otimes\mathbf{I}^{\otimes m}$
is equivariant to actions of the $\mathrm{O}(3)$ group $\forall\, l \geq 1$ and $\forall\, m \leq \lfloor l/2\rfloor$ ($m \geq 0$). 
\begin{equation*} 
    \begin{split}
        \left(R \hat{\mathbf{r}}\right)^{\otimes (l-2m)}\otimes\mathbf{I}^{\otimes m} &\stackrel{\text{Lemma~\ref{lem:equivariance}~(ii)}}{=} \left(R \hat{\mathbf{r}}\right)^{\otimes (l-2m)}\otimes \left(R \mathbf{I}\right)^{\otimes m} \\
        &\stackrel{\text{Lemma~\ref{lem:equivariance}~(i)}}{=} R \left(\hat{\mathbf{r}}^{\otimes (l-2m)} \otimes \mathbf{I}^{\otimes m}\right). 
    \end{split}
\end{equation*} 
\end{proof}

Our next claim is that for $l_{3} \in \left\{\lvert l_{1} - l_{2}\rvert, \lvert l_{1} - l_{2}\rvert + 1, \cdots, l_{1} + l_{2}\right\}$ the irreducible Cartesian tensor product $\otimes_{\mathrm{Cart}}$ defined in Eqs.~(\ref{eq:product_even}) and (\ref{eq:product_odd}) is equivariant to actions of the $\mathrm{O}(3)$ group, i.e., the following diagram commutes
\begin{equation*}
    \xymatrix{
    \left(\mathbb{R}^{3}\right)^{\otimes l_{1}} \times \left(\mathbb{R}^{3}\right)^{\otimes l_{2}} \ar@{->}[d]_{R} \ar@{->}[rr]^*-<-1pt>{\ \ \ \ \ \otimes_{\mathrm{Cart}}} & & \left(\mathbb{R}^{3}\right)^{\otimes l_{3}} \ar@{->}[d]^{R} \\ 
    \left(\mathbb{R}^{3}\right)^{\otimes l_{1}} \times \left(\mathbb{R}^{3}\right)^{\otimes l_{2}} \ar@{->}[rr]_*-<2pt>{\ \ \ \ \ \otimes_{\mathrm{Cart}}} 
	& & \left(\mathbb{R}^{3}\right)^{\otimes l_{3}}
    }
\end{equation*}
The proof for this claim is very similar to Proposition \ref{prop:equiv_tensor}. The only difference is to show the equivariance for the $(k+m)$-fold tensor contraction.

\begin{proposition}
\label{prop:equiv_irr_cart}
The irreducible Cartesian tensor product $\otimes_{\mathrm{Cart}} : \left(\mathbb{R}^{3}\right)^{\otimes l_{1}} \times \left(\mathbb{R}^{3}\right)^{\otimes l_{2}} \rightarrow \left(\mathbb{R}^{3}\right)^{\otimes l_{3}}$ makes the above diagram commute.
\end{proposition}

\begin{proof}
It suffices to show that the $(k+m)$-fold tensor contraction is equivariant to actions of the $\mathrm{O}(3)$ group. For $\mathbf{x}_{l_1} \in (\mathbb{R}^{3})^{\otimes l_{1}}$, $\mathbf{y}_{l_2} \in (\mathbb{R}^{3})^{\otimes l_{2}}$, and $R$ being a representation of an element of the orthogonal group $\mathrm{O}(3)$, we can write
\begin{equation*}
    \begin{split} 
        & \left(\left( R \mathbf{x}_{l_1}\right) \cdot (s) \cdot \left(R \mathbf{y}_{l_2}\right)\right)_{\beta_{1} \cdots \beta_{l_{1}-s} \delta_{1} \cdots \delta_{l_{2}-s}} \\
        & \quad\quad = \sum_{\alpha_{1}, \cdots, \alpha_{s}} \left(R \mathbf{x}_{l_1}\right)_{\alpha_{1} \cdots \alpha_{s} \beta_{1} \cdots \beta_{l_{1} - s}} \left(R \mathbf{y}_{l_2}\right)_{\alpha_{1} \cdots \alpha_{s} \delta_{1} \cdots \delta_{l_{2} - s}}\\
        & \quad\quad = \sum_{\alpha_{1}, \cdots, \alpha_{s}} \left( \sum_{\substack{\gamma_{1}, \cdots, \gamma_{s} \\ \eta_{1}, \cdots, \eta_{l_{1}-s}}} R_{\alpha_{1}\gamma_{1}} \cdots R_{\alpha_{s}\gamma_{s}} R_{\beta_{1}\eta_{1}} \cdots R_{\beta_{l_{1}-s}\eta_{l_{1}-s}} \left(\mathbf{x}_{l_1}\right)_{\gamma_{1} \cdots \gamma_{s} \eta_{1} \cdots \eta_{l_{1}-s}} \right) \\
        &  \hspace{23.5mm} \times \left( \sum_{\substack{\tilde{\gamma}_{1}, \cdots, \widetilde{\gamma}_{s} \\ \widetilde{\eta}_{1}, \cdots, \widetilde{\eta}_{l_{2}-s}}} R_{\alpha_{1}\widetilde{\gamma}_{1}} \cdots R_{\alpha_{s}\widetilde{\gamma}_{s}} R_{\delta_{1}\widetilde{\eta}_{1}} \cdots R_{\delta_{l_{2}-s}\widetilde{\eta}_{l_{2}-s}} \left(\mathbf{y}_{l_2}\right)_{\widetilde{\gamma}_{1} \cdots \widetilde{\gamma}_{s} \widetilde{\eta}_{1} \cdots \widetilde{\eta}_{l_{2}-s}} \right)\\
        & \quad\quad = \left( \sum_{\substack{\eta_{1} \cdots \eta_{l_{1}-s}}} R_{\beta_{1}\eta_{1}} \cdots R_{\beta_{l_{1}-s}\eta_{l_{1}-s}} \right) \left( \sum_{\substack{\widetilde{\eta}_{1} \cdots \widetilde{\eta}_{l_{2}-s}}}  R_{\delta_{1}\widetilde{\eta}_{1}} \cdots R_{\delta_{l_{2}-s}\widetilde{\eta}_{l_{2}-s}} \right) \\
        & \hspace{23.5mm} \times \sum_{\substack{\gamma_{1} \cdots \gamma_{s}}}\left(\mathbf{x}_{l_1}\right)_{\gamma_{1} \cdots \gamma_{s} \eta_{1} \cdots \eta_{l_{1}-s}} \left(\mathbf{y}_{l_2}\right)_{\gamma_{1} \cdots \gamma_{s} \widetilde{\eta}_{1} \cdots \widetilde{\eta}_{l_{2}-s}}\\
        % &= \left( \sum_{\substack{\widetilde{\eta}_{1} \cdots \widetilde{\eta}_{l_{1}-s}}} R_{\beta_{1}\eta_{1}} \cdots R_{\beta_{l_{1}-s}\eta_{l_{1}-s}} \right) \left( \sum_{\substack{\widetilde{\eta}_{1} \cdots \widetilde{\eta}_{l_{1}-s}}}  R_{\delta_{1}\widetilde{\eta}_{1}} \cdots R_{\delta_{l_{1}-s}\widetilde{\eta}_{l_{1}-s}} \right)  \left(\mathbf{x}_{l_1} \cdot (s) \cdot \mathbf{y}_{l_2}\right) \\
        & \quad\quad = \left(R \left(\mathbf{x}_{l_1} \cdot (s) \cdot \mathbf{y}_{l_2}\right)\right)_{\beta_{1} \cdots \beta_{l_{1}-s} \delta_{1} \cdots \delta_{l_{2}-s}}.
    \end{split}
\end{equation*}
A similar derivation applies to the odd case (\ref{eq:product_odd}), which completes the proof.
\end{proof}

We finally give a proof for Proposition \ref{prop:ictp_equivariance}, i.e., the equivariance of message-passing layers based on irreducible Cartesian tensors to actions of the $\mathrm{O}(3)$ group.

\begin{proof}[Proof of Proposition \ref{prop:ictp_equivariance}]
Since a message-passing layer in \secref{sec:cartesian_message_passing} and \appref{sec:architecture_appendix} is defined by stacking layers corresponding to Eqs.~(\ref{eq:atomic_basis}), (\ref{eq:atomic_basis_first}), (\ref{eq:product_basis}), (\ref{eq:messages_uncoupled}), (\ref{eq:update}), and (\ref{eq:messages_coupled}), we show the equivariance for each of these equations. We show the equivariance by induction.

\begin{itemize}[left=0pt]
    \item Case $t=1$:
    \begin{itemize}[left=0pt]
        \item[] \underline{Equivariance of \eqref{eq:atomic_basis_first}}: Since learnable weights $W_{kZ_v}$ and learnable (invariant) radial basis $R_{kl_1}^{(1)}(r_{uv})$ are defined for each feature channel $k$, the weights are multiplied with $\big(\mathbf{T}_{l_1}(\hat{\mathbf{r}}_{uv})\big)_{i_1i_2\cdots i_{l_1}}$ as scalars. Therefore,
        \begin{equation*}
            \big(\mathbf{A}_{ukl_1}^{(1)}\big)_{i_1i_2\cdots i_{l_1}} = \sum\nolimits_{v \in \mathcal{N}(u)} R_{kl_1}^{(1)}(r_{uv}) \big(\mathbf{T}_{l_1}(\hat{\mathbf{r}}_{uv})\big)_{i_1i_2\cdots i_{l_1}} W_{kZ_v}
        \end{equation*}
        is equivariant to actions of the $\mathrm{O}(3)$ group.
        \item[] \underline{Equivariance of \eqref{eq:product_basis}}: The proof is similar to the case of \eqref{eq:atomic_basis_first}, since we again have learnable parameter $W_{kk^\prime l_\nu}^{(1)}$ for each $\mathbf{A}_{uk^\prime l_\nu}^{(1)}$. Therefore, noting that the Cartesian tensor product is equivariant to actions of the $\mathrm{O}(3)$ group, the following function
        \begin{equation*}
            \big(\mathbf{B}_{u\eta_{\nu} kL}^{(1)}\big)_{i_1i_2\cdots i_L} = \big(\tilde{\mathbf{A}}_{ukl_1}^{(1)} \otimes_\mathrm{Cart} \cdots \otimes_\mathrm{Cart} \tilde{\mathbf{A}}_{ukl_\nu}^{(1)}\big)_{i_1i_2\cdots i_L}
        \end{equation*}
        is also equivariant to actions of the $\mathrm{O}(3)$ group.
        \item[] \underline{Equivariance of \eqref{eq:messages_uncoupled}}: The equation defined by 
        \begin{equation*}
          \big(\mathbf{m}_{ukL}^{(1)}\big)_{i_1i_2\cdots i_L} = \Big(\sum_\nu\sum_{\eta_\nu} W_{Z_u\eta_{\nu}kL}^{(1)} \mathbf{B}_{u\eta_{\nu}kL}^{(1)}\Big)_{i_1i_2\cdots i_L},
        \end{equation*}
        is equivariant to actions of the $\mathrm{O}(3)$ group, since for each set of indices $u$, $\eta_{\nu}$, $k$, and $L$ we have a scalar learnable parameter $W_{Z_u\eta_{\nu}kL}^{(1)}$.
        \item[] \underline{Equivariance of \eqref{eq:update}}: The respective parameters $W_{kk^\prime L}^{(1)}$ and $W_{Z_ukk^\prime L}^{(1)}$ in \eqref{eq:update}
        \begin{equation*}
            \big(\mathbf{h}_{ukL}^{(2)}\big)_{i_1i_2\cdots i_L} = \frac{1}{\sqrt{d_t}} \sum_{k^\prime} W_{kk^\prime L}^{(1)}\big(\mathbf{m}_{uk^\prime L}^{(1)}\big)_{i_1i_2\cdots i_L} + \frac{1}{\sqrt{d_t N_Z}} \sum_{k^\prime} W_{Z_ukk^\prime L}^{(1)}\big(\mathbf{h}_{uk^\prime L}^{(1)}\big)_{i_1i_2\cdots i_L}.
        \end{equation*}
        are multiplied as scalar with $\mathbf{h}_{uk^\prime L}^{(1)}$ and $\mathbf{m}_{uk^\prime L}^{(1)}$, which are shown to be equivariant to actions of the $\mathrm{O}(3)$ group. Thus, $\mathbf{h}_{ukL}^{(2)}$ is equivariant to actions of the $\mathrm{O}(3)$ group as a function of $\mathbf{h}_{uk^\prime L}^{(1)}$ and $\mathbf{m}_{uk^\prime L}^{(1)}$.
        \item[] \underline{Equivariance of \eqref{eq:messages_coupled}}: The equivariance of \eqref{eq:messages_coupled} to actions of the $\mathrm{O}(3)$ group is also immediate since learnable parameters apply to each tuple of $(i_1,i_2,\cdots, i_L)$ as a scalar.
    \end{itemize}
    \item General case $t > 1$: The equivariance of Eqs.~(\ref{eq:atomic_basis}), (\ref{eq:atomic_basis_first}), (\ref{eq:product_basis}), (\ref{eq:messages_uncoupled}), (\ref{eq:update}), and (\ref{eq:messages_coupled}) to actions of the $\mathrm{O}(3)$ group for arbitrary $t > 1$ follows in a similar way to those obtained for the $t = 1$ case. However, we need to show the equivariance of \eqref{eq:atomic_basis}. Suppose that Eqs.~(\ref{eq:atomic_basis_first}), (\ref{eq:product_basis}), (\ref{eq:messages_uncoupled}), (\ref{eq:update}), and (\ref{eq:messages_coupled}) are equivariant to actions of the $\mathrm{O}(3)$ group for $t > 1$. Then, for the $(t+1)$-th message-passing layer,
    \begin{equation*}
      \big(\mathbf{A}_{ukl_3}^{(t)}\big)_{i_1i_2\cdots i_{l_3}} = \sum\nolimits_{v \in \mathcal{N}(u)} \Big(R_{kl_1l_2l_3}^{(t)}(r_{uv})\mathbf{T}_{l_1}(\hat{\mathbf{r}}_{uv}) \otimes_{\mathrm{Cart}} \frac{1}{\sqrt{d_t}}\sum_{k^\prime}W_{kk^\prime l_2}^{(t)}\mathbf{h}_{vk^\prime l_2}^{(t)}\Big)_{i_1i_2\cdots i_{l_3}},
    \end{equation*}
    is equivariant to actions of the $\mathrm{O}(3)$ group. Here, the irreducible Cartesian tensor product $\otimes_{\mathrm{Cart}}$ and $\mathbf{T}_{l_{1}}$ are equivariant to actions of the $\mathrm{O}(3)$ group, $\mathbf{h}_{vk^\prime l_2}^{(t)}$ is equivariant by the assumption, and $W_{kk^\prime l_2}^{(t)}$ applies to $\mathbf{h}_{vk^\prime l_2}^{(t)}$ as a scalar.
\end{itemize}
\end{proof}

\section{Proof of the traceless property for irreducible Cartesian tensors} \label{sec:proof_traceless}

This section provides a proof for Proposition \ref{prop:ictp_tracelss} in the main text. The corresponding proposition states the following
\begin{proposition}
The message-passing layers based on irreducible Cartesian tensors and their irreducible tensor products preserve the traceless property of irreducible Cartesian tensors.
\end{proposition}

Our proof of Proposition \ref{prop:ictp_tracelss} comprises two main parts. First, we show that given two irreducible Cartesian tensors, the irreducible Cartesian tensor product yields again an irreducible Cartesian tensor. We then can use this result in the second part, where we show that expressions defined by Eqs.~(\ref{eq:atomic_basis}), (\ref{eq:atomic_basis_first}), (\ref{eq:product_basis}), (\ref{eq:messages_uncoupled}), (\ref{eq:update}), and (\ref{eq:messages_coupled}) preserve the traceless property of irreducible Cartesian tensors. The proof of the second part is straightforward since the message-passing layers are defined by multiplications with scalars, summations, and the irreducible Cartesian tensor product of tensors defined in Eqs.~(\ref{eq:cartesian_irreps}), (\ref{eq:product_even}), and (\ref{eq:product_odd}). Therefore, we provide the proof only for the first part in the rest of this section.

\textbf{Irreducible Cartesian tensors from unit vectors.} Recall that an irreducible Cartesian tensor of arbitrary rank $l$ for a unit vector $\hat{\mathbf{r}} \in \mathbb{R}^{3}$ is defined as
\begin{equation*} 
    \mathbf{T}_{l}\left(\hat{\mathbf{r}}\right) = C \sum\nolimits_{m=0}^{\lfloor l/2\rfloor} (-1)^m \frac{\left(2l-2m-1\right)!!}{\left(2l-1\right)!!} \big\{\hat{\mathbf{r}}^{\otimes \left(l-2m\right)}\otimes\mathbf{I}^{\otimes m}\big\}.
\end{equation*}
The trace of this tensor reads
\begin{equation*}
    \begin{split}
        \operatorname{Tr}\left(\mathbf{T}_{l}\left(\hat{\mathbf{r}}\right)\right) 
        & = \sum\nolimits_{i_{1} = i_{2}} \left(\mathbf{T}_{l}\left(\hat{\mathbf{r}}\right)\right)_{i_{1} i_{2}} \\
        & = \sum\nolimits_{i_{1} = i_{2}}  C \sum\nolimits_{m=0}^{\lfloor l/2\rfloor} \left(-1\right)^m \frac{\left(2l-2m-1\right)!!}{\left(2l-1\right)!!} \left(\big\{\hat{\mathbf{r}}^{\otimes \left(l-2m\right)}\otimes\mathbf{I}^{\otimes m}\big\}\right)_{i_{1} i_{2}} \\
        & = C \sum\nolimits_{m=0}^{\lfloor l/2\rfloor} \left(-1\right)^m \frac{\left(2l-2m-1\right)!!}{\left(2l-1\right)!!} \sum\nolimits_{i_{1} = i_{2}} \left(\big\{\hat{\mathbf{r}}^{\otimes \left(l-2m\right)}\otimes\mathbf{I}^{\otimes m}\big\}\right)_{i_{1} i_{2}}.
    \end{split}
\end{equation*}

Let $\big\{k_{1}, k_{2}, \dots, k_{l-2m}\big\}$ be a subset of $\{1, 2, \dots, l\}$ with $l-2m$ distinct elements, and a subset $I_{l-2m}$ with $l-2m$ distinct elements in $\big\{i_{1}, \dots, i_{l}\big\}$ is written as $I_{l-2m} = \big\{i_{k_{1}}, i_{k_{2}}, \dots, i_{k_{l-2m}}\big\}$. We also let $J_{l-2m}$ be a set of $m$ disjoint subsets with $2$ elements in $\big\{i_{1}, \dots, i_{l} \big\} \backslash I_{l-2m}$, i.e., $J_{l-2m} = \{ \{i_{k_{l-2m+1}}, i_{k_{l-2m+2}}\}, \dots, \{i_{k_{l-1}}, i_{k_{l}}\}\}$. While we introduce $I_{l-2m}$ and $J_{l-2m}$ as sets, we assume the order of elements in those sets if the order is necessary and there is no confusion. Furthermore, we introduce two notations used throughout this section
\begin{equation*}
    \begin{split}
        \hat{\mathbf{r}}^{\otimes \left(l-2m\right)}_{I_{l-2m}} &= \hat{r}_{i_{k_{1}}} \hat{r}_{i_{k_{2}}}\cdots \hat{r}_{i_{k_{l-2m}}} \\
        \mathbf{I}^{\otimes m}_{J_{l-2m}} &= \delta_{i_{k_{l-2m+1}} i_{k_{l-2m+2}}} \cdots \delta_{i_{k_{l-1}} i_{k_{l}}}.
    \end{split}
\end{equation*}
Then, we can write
\begin{equation*}
    \begin{split}
        \big\{\hat{\mathbf{r}}^{\otimes \left(l-2m\right)}\otimes\mathbf{I}^{\otimes m}\big\} 
        & = \sum\nolimits_{I_{l-2m}} \sum\nolimits_{J_{l-2m}} \hat{\mathbf{r}}^{\otimes \left(l-2m\right)}_{I_{l-2m}} \otimes \mathbf{I}^{\otimes m}_{J_{l-2m}} \\ 
        & = \sum\nolimits_{I_{l-2m}} \sum\nolimits_{J_{l-2m}} \hat{r}_{i_{k_{1}}} \cdots \hat{r}_{i_{k_{l-2m}}} \delta_{i_{k_{l-2m+1}} i_{k_{l-2m+2}}} \cdots \delta_{i_{k_{l-1}} i_{k_{l}}},
    \end{split}
\end{equation*}
where $I_{l-2m}$ runs over all subsets of $\big\{i_{1}, \dots, i_{l}\big\}$ with $l-2m$ elements and $J_{l-2m}$ comprises all $\left(\prod_{p=m}^{1} \binom{2p}{2}\right) / m!$ combinations of $m$ subsets with $2$ elements of $\big\{i_{1}, \dots, i_{l} \big\} \backslash I_{l-2m}$. Depending on whether $i_{1}$ and/or $i_{2}$ belong to $I_{l-2m}$, the above expression may be further reduced to
\begin{equation*}
    \begin{split}
        & \left(\sum_{i_{1}, i_{2} \in I_{l-2m}}^{\substack{(a)_{m} \\ \phantom{a}}} + \sum_{\substack{i_{1} \in I_{l-2m} \\ i_{2} \in J_{l-2m}}}^{\substack{(b)_{m} \\ \phantom{=}}} + \sum_{\substack{i_{2} \in I_{l-2m} \\ i_{1} \in J_{l-2m}}}^{\substack{(c)_{m} \\ \phantom{=}}} + \sum_{\substack{i_{1}, i_{2} \notin I_{l-2m} \\ \delta_{i_{1} i_{2}}}}^{\substack{(d)_{m} \\ \phantom{=}}} + \sum_{\substack{i_{1}, i_{2} \notin I_{l-2m} \\ \delta_{i_{s} i_{1}} \delta_{i_{2} i_{t}}}}^{\substack{(e)_{m} \\ \phantom{=}}}\right) \hat{r}_{i_{k_{1}}} \cdots \hat{r}_{i_{k_{l-2m}}} \\ 
        & \hspace{91.5mm} \times \delta_{i_{k_{l-2m+1}} i_{k_{l-2m+2}}} \cdots \delta_{i_{k_{l-1}} i_{k_{l}}}.
    \end{split}
\end{equation*}

We get the following traces for each pair of $I_{l-2m}$ and $J_{l-2m}$
\begin{equation*}
    \begin{split}
        \operatorname{Tr}\left(\left(a\right)_{m}\right)
        & = \hat{\mathbf{r}}^{\otimes (l-2m)}_{I_{l-2m} \backslash \{i_{1}, i_{2}\}} \otimes \mathbf{I}^{\otimes m}_{J_{l-2m}}, \\
        \operatorname{Tr}\left(\left(b\right)_{m}\right)
        & = \sum\nolimits_{i_{1}=i_{2}} \hat{r}_{i_{k_{1}}} \cdots \hat{r}_{i_{1}} \cdots \hat{r}_{i_{k_{l-2m}}} \delta_{i_{k_{l-2m+1}} i_{k_{l-2m+2}}} \cdots \delta_{i_{k_{l-1}} i_{k_{l}}} \delta_{i_{k_{s}} i_{2}} \\
        & = \hat{r}_{i_{k_{1}}} \cdots \hat{r}_{i_{k_s}} \cdots \hat{r}_{i_{k_{l-2m}}} \delta_{i_{k_{l-2m+1}} i_{k_{l-2m+2}}} \cdots \delta_{i_{k_{l-1}} i_{k_{l}}} \\
        & = \hat{\mathbf{r}}^{\otimes (l-2m)}_{\{i_{k_{s}}\} \cup I_{l-2m} \backslash \{i_{1}\}} \otimes \mathbf{I}^{\otimes (m-1)}_{J_{l-2m}\backslash \{i_{k_{{s}}}, i_{2}\}}, \\
        \operatorname{Tr}\left(\left(c\right)_{m}\right)
        & = \sum\nolimits_{i_{1}=i_{2}} \hat{r}_{i_{k_{1}}} \cdots \hat{r}_{i_{2}} \cdots \hat{r}_{i_{k_{l-2m}}} \delta_{i_{k_{l-2m+1}} i_{k_{l-2m+2}}} \cdots \delta_{i_{k_{l-1}} i_{k_{l}}} \delta_{i_{k_{s}} i_{1}} \\
        & = \hat{r}_{i_{k_{1}}} \cdots \hat{r}_{i_{k_s}} \cdots \hat{r}_{i_{k_{l-2m}}} \delta_{i_{k_{l-2m+1}} i_{k_{l-2m+2}}} \cdots \delta_{i_{k_{l-1}} i_{k_{l}}} \\
        & = \hat{\mathbf{r}}^{\otimes (l-2m)}_{\{i_{k_{s}}\} \cup I_{l-2m} \backslash \{i_{2}\}} \otimes \mathbf{I}^{\otimes (m-1)}_{J_{l-2m}\backslash \{i_{k_{{s}}}, i_{1}\}}, \\
        \operatorname{Tr}\left(\left(d\right)_{m}\right)
        & = \sum\nolimits_{i_{1}=i_{2}} \hat{r}_{i_{k_{1}}} \cdots \hat{r}_{i_{k_{l-2m}}} \delta_{i_{k_{l-2m+1}} i_{k_{l-2m+2}}} \cdots \delta_{i_{k_{l-1}} i_{k_{l}}} \delta_{i_{1} i_{2}} \\
        & =3 \left(\hat{\mathbf{r}}^{\otimes (l-2m)}_{I_{l-2m}} \otimes \mathbf{I}^{\otimes (m-1)}_{J_{l-2m}\backslash \{i_{1}, i_{2}\}}\right), \\
        \operatorname{Tr}\left(\left(e\right)_{m}\right)
        & = \sum\nolimits_{i_{1}=i_{2}} \hat{r}_{i_{k_{1}}} \dots \hat{r}_{i_{k_{l-2m}}} \delta_{i_{k_{l-2m+1}} i_{k_{l-2m+2}}} \cdots \delta_{i_{s} i_{1}} \delta_{i_{2} i_{t}} \cdots \delta_{i_{k_{l-1}} i_{k_{l}}} \\
        & = \hat{\mathbf{r}}^{\otimes (l-2m)}_{I_{l-2m}} \otimes \mathbf{I}^{\otimes (m-1)}_{J_{l-2m}\backslash \{i_{1}, i_{2}\}}.
    \end{split}
\end{equation*}

\begin{lemma}\label{lem:diag_cancel} Let $l \geq 2$ and $0 \leq m \leq {\lfloor l/2\rfloor} - 1$. For each pair $\left(I_{l-2m}, J_{l-2m}\right)$, there exist $2l-2m-1$ pairs of $\left(I_{l-2(m+1)}, J_{l-2(m+1)}\right)$ whose trace equals to the trace for the pair $\left(I_{l-2m}, J_{l-2m}\right)$, i.e.,
\begin{equation*}
    \left(2l-2m-1\right) \operatorname{Tr}\left(\hat{\mathbf{r}}^{\otimes (l-2m)}_{I_{l-2m}} \otimes \mathbf{I}^{\otimes m}_{J_{l-2m}}\right) = \operatorname{Tr}\left(\sum_{(I_{l-2(m+1)}, J_{l-2(m+1)})} \hat{\mathbf{r}}^{\otimes\left(l-2(m+1)\right)}_{I_{l-2(m+1)}} \otimes \mathbf{I}^{\otimes (m+1)}_{J_{l-2(m+1)}}\right).
\end{equation*}
\end{lemma}

\begin{proof}
Without loss of generality, we may assume $I_{l-2m} = \left(i_{1}, i_{2}, i_{k_{3}}, \dots, i_{k_{l-2m}}\right)$ and $J_{l-2m} = \big\{\big\{i_{k_{l-2m+1}}, i_{k_{l-2m+2}}\big\}, \dots, \big\{i_{k_{l-1}}, i_{k_{l}}\big\}\big\}$. For $\forall\,m$, $\operatorname{Tr}\left(\left(a\right)_{m}\right)$ has the following expression 
\begin{equation*}
    \begin{split}
        \operatorname{Tr}\left(\left(a\right)_{m}\right) 
        & = \sum\nolimits_{i_{1} = i_{2}} \hat{r}_{i_{1}} \hat{r}_{i_{2}} \hat{r}_{i_{k_{3}}} \cdots \hat{r}_{i_{k_{l-2m}}} \delta_{i_{k_{l-2m+1}} i_{k_{l-2m+2}}} \cdots \delta_{i_{k_{l-1}} i_{k_{l}}} \\
        & = \hat{r}_{i_{k_{3}}} \cdots \hat{r}_{i_{k_{l-2m}}} \delta_{i_{k_{l-2m+1}} i_{k_{l-2m+2}}} \cdots \delta_{i_{k_{l-1}} i_{k_{l}}}.
    \end{split}
\end{equation*}

For $\operatorname{Tr}\big((b)_{m+1}\big)$ and $m+1$, we have a set of $\left(I_{l-2(m+1)}, J_{l-2(m+1)}\right)$ with the cardinality of $l-2m-2$ and the corresponding trace equals to $\operatorname{Tr}\left(\left(a\right)_{m}\right)$. For $3 \leq \forall\,s \leq l-2m$, let
\begin{equation*}
    \begin{split}
        I^{(s)}_{l - 2(m+1)} &= \big\{i_{k_{3}}, \dots, \underbrace{i_{1}}_{s\text{-th index}}, \dots, i_{k_{l-2m}}\big\}, \\
        J^{(s)}_{l - 2(m+1)} &= \big\{\big\{i_{k_{l-2m+1}}, i_{k_{l-2m+2}}\big\}, \dots, \big\{i_{k_{l-1}}, i_{k_{l}}\big\}, \big\{i_{k_{s}}, i_{2}\big\}\big\}.
    \end{split}
\end{equation*}
Here, we note that we assume the order of elements in $I^{(s)}_{l - 2(m+1)}$. Then, we obtain
\begin{equation*}
    \begin{split}
        \operatorname{Tr}\big((b)_{m+1}\big) 
        & = \sum\nolimits_{i_{1} = i_{2}} \sum\nolimits_{s=3}^{l-2m} \hat{r}_{i_{k_{3}}} \cdots \underbrace{\hat{r}_{i_{1}}}_{s\text{-th vector}} \cdots \hat{r}_{i_{k_{l-2m}}} \delta_{i_{k_{l-2m+1}} i_{k_{l-2m+2}}} \cdots \delta_{i_{k_{l-1}} i_{k_{l}}} \delta_{i_{k_{s}} i_{2}} \\
        & = \sum\nolimits_{s=3}^{l-2m} \hat{r}_{i_{k_{3}}} \cdots \hat{r}_{i_{k_{s}}} \cdots \hat{r}_{i_{k_{l-2m}}} \delta_{i_{k_{l-2m+1}} i_{k_{l-2m+2}}} \cdots \delta_{i_{k_{l-1}} i_{k_{l}}} \\
        & = \left(l - 2m - 2\right) \operatorname{Tr}\left(\left(a\right)_{m}\right).
    \end{split}
\end{equation*}

The same derivation as above holds for $\operatorname{Tr}\big((c)_{m+1}\big)$. For $\operatorname{Tr}\big((d)_{m+1}\big)$, let
\begin{equation*}
    \begin{split}
        I_{l - 2(m+1)} &= \big\{i_{k_{3}}, \dots, i_{k_{l-2m}}\big\}, \\
        J_{l - 2(m+1)} &= \big\{\big\{i_{k_{l-2m+1}}, i_{k_{l-2m+2}}\big\}, \dots, \big\{i_{k_{l-1}}, i_{k_{l}}\big\}, \big\{i_{1}, i_{2}\big\}\big\}.
    \end{split}
\end{equation*}
and we get
\begin{equation*}
    \begin{split}
        \operatorname{Tr}\big((d)_{m+1}\big) 
        & = \sum\nolimits_{i_{1} = i_{2}} \hat{r}_{i_{k_{3}}} \cdots \hat{r}_{i_{k_{l-2m}}} \delta_{i_{k_{l-2m+1}} i_{k_{l-2m+2}}} \cdots \delta_{i_{k_{l-1}} i_{k_{l}}} \delta_{i_{1} i_{2}} \\ 
        & = 3 \operatorname{Tr}\left(\left(a\right)_{m}\right).
    \end{split}
\end{equation*}

For $\operatorname{Tr}\big((e)_{m+1}\big)$ and $1 \leq \forall\,s \leq m$, we let
\begin{equation*}
    \begin{split}
        I^{(s)}_{l - 2(m+1)} &= \big\{i_{k_{3}}, \dots, i_{k_{l-2m}}\big\}, \\
        J^{(s)}_{l - 2(m+1)} &= \big\{\big\{i_{k_{l-2m+1}}, i_{k_{l-2m+2}}\big\}, \dots ,\big\{i_{k_{l-2m+2s-1}}, i_{1}\big\}, \big\{i_{2}, i_{k_{l-2m+2s}}\big\}, \dots ,\big\{i_{k_{l-1}}, i_{k_{l}}\big\}\big\},
    \end{split}
\end{equation*}
and obtain
\\
\scalebox{0.9}{
\begin{minipage}{\linewidth}
    \begin{equation*}
        \begin{split}
            \operatorname{Tr}\big((e)_{m+1}\big) 
            & = \sum\nolimits_{i_{1} = i_{2}} \sum\nolimits_{s=1}^{m} \hat{r}_{i_{k_{3}}} \cdots \hat{r}_{i_{k_{l-2m}}} \delta_{i_{k_{l-2m+1}} i_{k_{l-2m+2}}} \cdots \delta_{i_{k_{l-2m+2s-1}} i_{1}} \delta_{i_{2} i_{k_{l-2m+2s}}} \cdots \delta_{i_{k_{l-1}} i_{k_{l}}} \\
            & = \sum\nolimits_{s=1}^{m} \hat{r}_{i_{k_{3}}} \cdots \hat{r}_{i_{k_{l-2m}}} \delta_{i_{k_{l-2m+1}} i_{k_{l-2m+2}}} \cdots \delta_{k_{l-2m+2s-1} k_{l-2m+2s}} \cdots \delta_{i_{k_{l-1}} i_{k_{l}}} \\ &= 2m \operatorname{Tr}\left(\left(a\right)_{m}\right).
        \end{split}
    \end{equation*}
\end{minipage}}
\\
The additional factor of two originates from the permutation of indices $i_1$ and $i_2$. 

Using all the above results, we obtain
\begin{equation*}
    \begin{split}
        & \operatorname{Tr}((b)_{m+1}) + \operatorname{Tr}((c)_{m+1}) + \operatorname{Tr}((d)_{m+1}) + \operatorname{Tr}((e)_{m+1}) \\
        & \quad\quad = \left(l - 2m -2\right) \operatorname{Tr}\left(\left(a\right)_{m}\right) + \left(l - 2m -2\right) \operatorname{Tr}\left(\left(a\right)_{m}\right) + 3 \operatorname{Tr}\left(\left(a\right)_{m}\right) + 2m \operatorname{Tr}\left(\left(a\right)_{m}\right) \\
        & \quad\quad = \left(2l-2m-1\right) \operatorname{Tr}\left(\left(a\right)_{m}\right),
    \end{split}
\end{equation*}
which completes the proof of the Lemma.
\end{proof}

For each $m$, define $A(m) = \operatorname{Tr}\left(\left(a\right)_{m}\right)$ and $B(m) = \operatorname{Tr}\left(\left(b\right)_{m}\right) + \operatorname{Tr}\left(\left(c\right)_{m}\right) + \operatorname{Tr}\left(\left(d\right)_{m}\right) + \operatorname{Tr}\left(\left(e\right)_{m}\right)$. Then, the trace of $\mathbf{T}_{l}\left(\hat{\mathbf{r}}\right)$ may be written as a summation of $A(m)$ and $B(m)$ over $m$
\begin{equation*}
    \begin{split}
        \operatorname{Tr}\left(\mathbf{T}_{l}\left(\hat{\mathbf{r}}\right)\right) 
        & = C \sum\nolimits_{m=0}^{\lfloor l/2\rfloor} \left(-1\right)^m \frac{(2l-2m-1)!!}{(2l-1)!!} \sum\nolimits_{i_{1} = i_{2}} \left(\big\{\hat{\mathbf{r}}^{\otimes (l-2m)}\otimes\mathbf{I}^{\otimes m}\big\}\right)_{i_{1}, i_{2}} \\
        & = C ( \hspace{144pt} \left(A(0) + B(0)\right)  \\
        & \hspace{75pt} +\left(-1\right) \frac{(2l-2-1)!!}{(2l-1)!!} \left(A(1) + B(1)\right) \\
        & \hspace{140pt} \cdots \\
        & \hspace{58pt} +\left(-1\right)^{\tilde{m}} \frac{(2l-2\tilde{m}-1)!!}{(2l-1)!!} \left(A(\tilde{m}) + B(\tilde{m})\right)  \\
        & \hspace{23pt} +\left(-1\right)^{\tilde{m}+1} \frac{(2l-2(\tilde{m}+1)-1)!!}{(2l-1)!!}  \left(A(\tilde{m}+1) + B(\tilde{m}+1)\right)  \\
        & \hspace{140pt} \cdots .
    \end{split}
\end{equation*}
By Lemma \ref{lem:diag_cancel}, each pair of the diagonal terms $\left(A\left(\tilde{m}\right), B\left(\tilde{m}+1\right)\right)$ is cancelled, and $B(0)$ and $A\left(\lfloor l/2\rfloor\right)$ are left, which by definition are trivial. This result completes the proof of the traceless property for \eqref{eq:cartesian_irreps}.

\textbf{Even irreducible Cartesian tensor product.} The trace of an even irreducible Cartesian tensor product may be written as
\\
\scalebox{0.965}{
\begin{minipage}{\linewidth}
    \begin{equation*}
        \begin{split}
            & \operatorname{Tr}\left(\left(\mathbf{x}_{l_1} \otimes_{\mathrm{Cart}} \mathbf{y}_{l_2}\right)_{l_3}\right) \\
            & \quad\quad = \sum_{i_{1} = i_{2}} C_{l_1l_2l_3}
            \sum\nolimits_{m=0}^{\mathrm{min}(l_1,l_2) - k}(-1)^m 2^m \frac{(2l_3-2m-1)!!}{(2l_3-1)!!}\left(\big\{\left(\mathbf{x}_{l_1}\cdot(k+m)\cdot\mathbf{y}_{l_2}\right)\otimes\mathbf{I}^{\otimes m}\big\}\right)_{i_{1} i_{2}} \\
            & \quad\quad = C_{l_1l_2l_3}
            \sum\nolimits_{m=0}^{\mathrm{min}(l_1,l_2) - k}(-1)^m 2^m \frac{(2l_3-2m-1)!!}{(2l_3-1)!!} \underbrace{\sum_{i_{1} = i_{2}} \left(\big\{\left(\mathbf{x}_{l_1}\cdot(k+m)\cdot\mathbf{y}_{l_2}\right)\otimes\mathbf{I}^{\otimes m}\big\}\right)_{i_{1} i_{2}}}_{(*)_{\text{even}}}.
        \end{split}
    \end{equation*}
\end{minipage}}

In the following, we let 
\begin{equation*}
    \begin{split}
        I_{l_{1} - (k+m)} & = \big\{i_{k_{1}}, \dots, i_{k_{l_{1}-(k+m)}}\big\} \subset \big\{i_{1}, \dots, i_{l_{3}}\big\}, \\
        J_{l_{2} - (k+m)} & = \big\{i_{k_{l_{1}-(k+m)+1}}, \dots, i_{k_{l_3 - 2m}} \big\} \subset \big\{i_{1}, \dots, i_{l_{3}}\big\} \backslash I_{l_{1} - (k+m)}, \\
        K_{l_{3} - 2m} & = \big\{\big\{i_{k_{l_{3} - 2m+1}}, i_{k_{l_{3} - 2m+2}} \big\},  \dots, \big\{i_{k_{l_{3}-1}}, i_{k_{l_{3}}}\big\}\big\}.
    \end{split}
\end{equation*}
For simplicity, we omit the subscripts $l_{1} - (k+m)$, $l_{2} - (k+m)$, and $l_{3} - 2m$ for $I$, $J$, and $K$, if there is no confusion. We also introduce a new notation $\mathbf{x}_{j_{1} \cdots j_{p}, Q} = \mathbf{x}_{j_{1} \cdots j_{p}, q_{1} \cdots q_{u}},$ where $\mathbf{x}$ is an irreducible Cartesian tensor and $Q = \big\{q_{1}, \dots, q_{u}\big\}$, which is used throughout this section. Then, non-trivial terms in $(*)_{\text{even}}$ may be written as 
\begin{equation*}
    \begin{split}
        \sum_{i_{1}=i_{2}}\left( \sum_{\substack{i_{1} \in I\\i_{2} \in J}}^{\substack{(a)_{m} \\ \phantom{a}}} 
        + \sum_{\substack{i_{2} \in I\\i_{1} \in J}}^{\substack{(b)_{m} \\ \phantom{a}}} 
        +\sum_{\substack{i_{1} \in I \cup J \\ i_{2} \in K}}^{\substack{(c)_{m} \\ \phantom{=}}} 
        +\sum_{\substack{i_{2} \in I \cup J \\ i_{1} \in K}}^{\substack{(d)_{m} \\ \phantom{=}}} 
        + \sum_{i_{1}, i_{2} \in K}^{\substack{(e)_{m} \\ \phantom{=}}} \right) \sum_{j_{1}, \dots, j_{k+m}} \left(\mathbf{x}_{l_1}\right)_{j_{1} \cdots j_{k+m}, I} \left(\mathbf{y}_{l_2}\right)_{j_{1} \cdots j_{k+m}, J} \mathbf{I}^{\otimes m}_{K}.
    \end{split}
\end{equation*}
We note that the trace is trivial when both of $i_{1}$ and $i_{2}$ belong to either of $I$ or $J$, because $\mathbf{x}_{l_1}$ and $\mathbf{y}_{l_2}$ are traceless by assumption.

\begin{lemma} \label{lem:diag_cancel_even} Let $0 \leq m \leq \mathrm{min}(l_1,l_2) - k - 1$. For each triplet $\left(I_{l_{1} - (k+m)}, J_{l_{2} - (k+m)}, K_{l_{3} - 2m}\right)$, there exist $2l_{3}-2m-1$ triplets of $\left(I_{l_{1} - (k+m+1)}, J_{l_{2} - (k+m+1)}, K_{l_{3} - 2(m+1)}\right)$ such that the trace for each triplet equals to the trace for $\left(I_{l_{1} - (k+m)}, J_{l_{2} - (k+m)}, K_{l_{3} - 2m}\right)$.
\end{lemma}

\begin{proof}
Without loss of generality, we may assume
\begin{equation*}
    \begin{split}
        I_{l_{1} - (k+m)} & = \big\{i_{1}, i_{k_{2}}, \dots, i_{k_{l_{1}-(k+m)}}\big\} \subset \big\{i_{1}, \dots, i_{l_{3}}\big\}, \\
        J_{l_{2} - (k+m)} & = \big\{i_{2}, i_{k_{l_{1}-(k+m)+2}}, \dots, i_{k_{l_{3} - 2m}} \big\} \subset \big\{i_{1}, \dots, i_{l_{3}}\big\} \backslash I_{l_{1} - (k+m)}, \\
        K_{l_{3} - 2m} & = \big\{\big\{i_{k_{l_{3} - 2m+1}}, i_{k_{l_{3} - 2m+2}} \big\},  \dots, \big\{i_{k_{l_{3}-1}}, i_{k_{l_{3}}}\big\}\big\}.
    \end{split}
\end{equation*}
Then, $\operatorname{Tr}\left(\left(a\right)_{m}\right)$ has the following expression 
\begin{equation*}
    \begin{split}
        & \operatorname{Tr}\left(\left(a\right)_{m}\right) \\
        & \quad\quad = \sum_{i_{1} = i_{2}} \sum_{j_{1}, \dots, j_{k+m}} \left(\mathbf{x}_{l_1}\right)_{j_{1} \cdots j_{k+m}, i_{1} i_{k_{2}} \cdots i_{k_{l_{1}-(k+m)}}} \left(\mathbf{y}_{l_2}\right)_{j_{1} \cdots j_{k+m}, i_{2} i_{k_{l_{1}-(k+m)+2}} \cdots i_{k_{l_{3} - 2m}}} \mathbf{I}^{\otimes m}_{K} \\
        & \quad\quad = \sum_{\substack{j_{1}, \dots, j_{k+m}, \\ j_{k+m+1}}} \left(\mathbf{x}_{l_1}\right)_{j_{1} \cdots j_{k+m+1}, i_{k_{2}} \cdots i_{k_{l_{1}-(k+m)}}} \left(\mathbf{y}_{l_2}\right)_{j_{1} \cdots j_{k+m+1}, i_{k_{l_{1}-(k+m)+2}} \cdots i_{k_{l_{3} - 2m}}} \mathbf{I}^{\otimes m}_{K}.
    \end{split}
\end{equation*}
The same derivation applies to $\operatorname{Tr}\left(\left(b\right)_{m}\right)$, and it has the same result as above. 

For $\operatorname{Tr}((c)_{m+1})$ and $m+1$, we have $l_{3}-2m-2$ of $\left(I_{l_{1} - (k+m+1)}, J_{l_{2} - (k+m+1)}, K_{l_{3} - 2(k+m+1)}\right)$ triplets whose trace equals to $\operatorname{Tr}\left(\left(a\right)_{m}\right)$. Indeed, like in the proof of Lemma \ref{lem:diag_cancel}, we let 
\begin{equation*}
    \begin{split}
        I_{l_{1} - (k+m+1)} \cup J_{l_{2} - (k+m+1)} &= \big\{i_{k_{2}}, \dots, i_{1} ,\dots, i_{k_{l_{1}-(k+m)}}, i_{k_{l_{1}-(k+m)+2}}, \dots, i_{k_{l_{3} - 2m}}\big\}, \\
        K_{l_{3} - 2(m+1)} &= \big\{\big\{i_{k_{l_3-2m+1}}, i_{k_{l_3-2m+2}}\big\}, \dots, \big\{i_{k_{l_3-1}}, i_{k_{l_3}}\big\}, \big\{ i_{k_{s}}, i_{2}\big\}\big\}.
    \end{split}
\end{equation*}
Then, by letting $\tilde{J} = \big\{j_{1}, \dots, j_{k+m}, j_{k+m+1}\big\},$ we obtain
\begin{equation*}
    \begin{split}
        & \operatorname{Tr}((c)_{m+1}) \\ 
        & \quad\quad = \sum_{i_{1}=i_{2}} \sum_{s} \sum_{\tilde{J}} \left(\mathbf{x}_{l_1}\right)_{\tilde{J} i_{k_{2}} \cdots i_{1} \cdots i_{k_{l_{1}-(k+m)}}} \left(\mathbf{y}_{l_2}\right)_{\tilde{J} i_{k_{l_{1}-(k+m)+2}} \cdots i_{k_{l_{3} - 2m}}} \\ 
        & \hspace{73.5mm} \times \delta_{i_{k_{l_{3}-2m+1}} i_{k_{l_{3}-2m+2}}} \cdots \delta_{i_{k_{l_{3}-1}} i_{k_{l_{3}}}} \delta_{i_{k_{s}} i_{2}} \\
        & \quad\quad = \sum_{s} \sum_{\tilde{J}} \left(\mathbf{x}_{l_1}\right)_{\tilde{J}, i_{k_{2}} \cdots i_{k_{s}} \cdots i_{k_{l_{1}-(k+m)}}} \left(\mathbf{y}_{l_2}\right)_{\tilde{J}, i_{k_{l_{1}-(k+m)+2}} \cdots i_{k_{l_{3} - 2m}}} \\ 
        & \hspace{73.5mm} \times \delta_{i_{k_{l_{3}-2m+1}}, i_{k_{l_{3}-2m+2}}} \cdots \delta_{i_{k_{l_{3}-1}}, i_{k_{l_{3}}}}\\
        & \quad\quad = (l_{3} - 2m -2) \operatorname{Tr}\left(\left(a\right)_{m}\right).
    \end{split}
\end{equation*}
The same derivation applies to $\operatorname{Tr}((d)_{m+1})$, and we get $\operatorname{Tr}((d)_{m+1})= (l_{3} - 2m -2) \operatorname{Tr}\left(\left(a\right)_{m}\right)$.

For $\operatorname{Tr}((e)_{m+1})$, the same derivation as for $\operatorname{Tr}((d)_{m+1})$ and $\operatorname{Tr}((e)_{m+1})$ in the proof of Lemma \ref{lem:diag_cancel} applies. Therefore, we have
\begin{equation*}
    \operatorname{Tr}((e)_{m+1}) = (2m+3)\operatorname{Tr}\left(\left(a\right)_{m}\right).
\end{equation*}

Hence, we obtain
\begin{equation*}
    \begin{split}
        & \operatorname{Tr}((c)_{m+1}) + \operatorname{Tr}((d)_{m+1}) + \operatorname{Tr}((e)_{m+1}) \\
        & \quad = (l_3 - 2m -2) \operatorname{Tr}((a)_m) + (l_3 - 2m -2) \operatorname{Tr}((a)_m) + (2m+3) \operatorname{Tr}((a)_m) \\
        & \quad = (2l_3-2m-1) \operatorname{Tr}((a)_m) = (2l_3-2m-1) \operatorname{Tr}((b)_m),
    \end{split}
\end{equation*}
which completes the proof of the Lemma.
\end{proof}

For each $m$, define $A(m) = \operatorname{Tr}\left(\left(a\right)_{m}\right) + \operatorname{Tr}\left(\left(b\right)_{m}\right)$ and $B(m) = \operatorname{Tr}\left(\left(c\right)_{m}\right) + \operatorname{Tr}\left(\left(d\right)_{m}\right) + \operatorname{Tr}\left(\left(e\right)_{m}\right)$. Then, the trace for $\left(\mathbf{x}_{l_1} \otimes_{\mathrm{Cart}} \mathbf{y}_{l_2}\right)_{l_3}$ may be written as a summation of $A(m)$ and $B(m)$ over $m$
\\
\scalebox{0.925}{
\begin{minipage}{\linewidth}
    \begin{equation*}
        \begin{split}
            & \operatorname{Tr}\left(\left(\mathbf{x}_{l_1} \otimes_{\mathrm{Cart}} \mathbf{y}_{l_2}\right)_{l_3}\right) \\
            & \quad\quad = C_{l_1l_2l_3}
            \sum\nolimits_{m=0}^{\mathrm{min}(l_1,l_2) - k}(-1)^m 2^m \frac{(2l_3-2m-1)!!}{(2l_3-1)!!} \sum\nolimits_{i_{1} = i_{2}} \left(\big\{\left(\mathbf{x}_{l_1}\cdot(k+m)\cdot\mathbf{y}_{l_2}\right)\otimes\mathbf{I}^{\otimes m}\big\}\right)_{i_{1} i_{2}}\\
            & \quad\quad = C_{l_1l_2l_3} ( \hspace{149pt} \left(A(0) + B(0)\right)  \\
            & \quad\quad \hspace{92pt} +(-1) 2 \frac{(2l_3-2-1)!!}{(2l_3-1)!!} \left(A(1) + B(1)\right) \\
            & \quad\quad \hspace{140pt} \cdots \\
            & \quad\quad \hspace{68pt} +(-1)^{\tilde{m}} 2^{\tilde{m}} \frac{(2l_3-2\tilde{m}-1)!!}{(2l_3-1)!!} \left(A(\tilde{m}) + B(\tilde{m})\right)  \\
            & \quad\quad \hspace{23pt} +(-1)^{\tilde{m}+1} 2^{\tilde{m}+1} \frac{(2l_3-2(\tilde{m}+1)-1)!!}{(2l_3-1)!!}  \left(A(\tilde{m}+1) + B(\tilde{m}+1)\right)  \\
            & \quad\quad \hspace{140pt} \cdots .
        \end{split}
    \end{equation*}
\end{minipage}}
\\
By Lemma \ref{lem:diag_cancel_even}, each pair of the diagonal terms $\left(A(\tilde{m}), B(\tilde{m}+1)\right)$ is cancelled taking into account an additional factor of $2^m$ in front of $B(\tilde{m}+1)$; $B(0)$ and $A\left(\min(l_1, l_2) - (l_1 + l_2 - l_3)/ 2\right)$ are left, which by definition are trivial. This result completes the proof of the traceless property for \eqref{eq:product_even}.

\textbf{Odd irreducible Cartesian tensor product.} The trace of an odd irreducible Cartesian tensor product may be written as
\\
\scalebox{0.9}{
\begin{minipage}{\linewidth}
    \begin{equation*}
        \begin{split}
            & \operatorname{Tr}\left(\left(\mathbf{x}_{l_1} \otimes_{\mathrm{Cart}} \mathbf{y}_{l_2}\right)_{l_3}\right) \\
            & \quad\quad = \sum_{i_{1} = i_{2}} D_{l_1l_2l_3}
            \sum\nolimits_{m=0}^{\mathrm{min}(l_1,l_2)-k-1}(-1)^m 2^m \frac{(2l_3-2m-1)!!}{(2l_3-1)!!}\left(\big\{\left(\boldsymbol{\varepsilon}:\mathbf{x}_{l_1}\cdot(k+m)\cdot\mathbf{y}_{l_2}\right)\otimes\mathbf{I}^{\otimes m}\big\}\right)_{i_{1}, i_{2}} \\
            & \quad\quad = D_{l_1l_2l_3}
            \sum\nolimits_{m=0}^{\mathrm{min}(l_1,l_2)-k-1}(-1)^m 2^m \frac{(2l_3-2m-1)!!}{(2l_3-1)!!} \underbrace{\sum_{i_{1} = i_{2}}\left(\big\{\left(\boldsymbol{\varepsilon}:\mathbf{x}_{l_1}\cdot(k+m)\cdot\mathbf{y}_{l_2}\right)\otimes\mathbf{I}^{\otimes m}\big\}\right)_{i_{1}, i_{2}}}_{(*)_{\text{odd}}}
        \end{split}
    \end{equation*}
\end{minipage}}

In the following, we let 
\begin{equation*}
    \begin{split}
        E_{l_{3}} &= \{i_{k_{0}}\} \subset \big\{i_{1}, \dots, i_{l_{3}}\big\}\\
        I_{l_{1} - (k+m)} &= \big\{i_{k_{1}}, \dots, i_{k_{l_{1}-(k+m)}}\big\} \subset \big\{i_{1}, \dots, i_{l_{3}}\big\} \backslash E_{l_{3}}, \\
        J_{l_{2} - (k+m)} &= \big\{i_{k_{l_{1}-(k+m)+1}}, \dots, i_{k_{l_{3} - 2m}} \big\} \subset \big\{i_{1}, \dots, i_{l_{3}}\big\} \backslash \left( I_{l_{1} - (k+m)} \cup E_{l_{3}}\right), \\
        K_{l_{3} - 2m} &= \big\{\big\{i_{k_{l_{3} - 2m+1}}, i_{k_{l_{3} - 2m+2}} \big\},  \dots, \big\{i_{k_{l_{3}-1}}, i_{k_{l_{3}}}\big\}\big\}.
    \end{split}
\end{equation*}
For simplicity, we omit the subscripts $l_3$, $l_{1} - (k+m)$, $l_{2} - (k+m)$, and $l_{3} - 2m$ for $E$, $I$, $J$, and $K$, if there is no confusion. Then, non-trivial terms inside $(*)_{\text{odd}}$ can be written as
\\
\scalebox{0.76}{
\begin{minipage}{\linewidth}
    \begin{equation*}
        \begin{split}
            \sum_{i_{1}=i_{2}}\left(\sum_{\substack{i_{1} \in I\\i_{2} \in J}}^{\substack{(a)_{m} \\ \phantom{a}}} 
            + \sum_{\substack{i_{2} \in I\\i_{1} \in J}}^{\substack{(b)_{m} \\ \phantom{a}}} 
            +\sum_{\substack{i_{1} \in E \cup I \cup J \\ i_{2} \in K}}^{\substack{(c)_{m} \\ \phantom{=}}} 
            +\sum_{\substack{i_{2} \in E \cup I \cup J \\ i_{1} \in K}}^{\substack{(d)_{m} \\ \phantom{=}}} 
            + \sum_{i_{1}, i_{2} \in K}^{\substack{(e)_{m} \\ \phantom{=}}} \right) \sum_{j_{1}, \dots, j_{k+m}} \sum_{a_2,a_3} \boldsymbol{\varepsilon}_{E a_{2} a_{3}}\left(\mathbf{x}_{l_1}\right)_{j_{1} \cdots j_{k+m} a_{2},I\backslash \{a_{2}\}} \left(\mathbf{y}_{l_2}\right)_{j_{1} \cdots j_{k+m} a_{3},J\backslash \{a_{3}\}} \mathbf{I}^{\otimes m}_{K}.
        \end{split}
    \end{equation*}
\end{minipage}}

Note that the trace is trivial when $i_{1}$ or $i_{2}$ belongs to $E$ and the remaining index belongs to $I \cup J$. Indeed, let $\tilde{J} = \big\{j_{1}, \dots, j_{k+m}\big\}$ and suppose $i_{1} \in E$ and $i_{2} \in I$. Then, and we have
\\
\scalebox{0.975}{
\begin{minipage}{\linewidth}
    \begin{equation*}
        \begin{split}
            & \left(\operatorname{Tr}\left(\left(\mathbf{x}_{l_1} \otimes_{\mathrm{Cart}} \mathbf{y}_{l_2}\right)_{l_3}\right)\right)_{\left(I \cup J\right) \backslash \{i_{2}\}} \\
            & \quad\quad = \sum_{i_1 = i_2} \left(\sum_{\tilde{J}} \sum_{a_2,a_3} \boldsymbol{\varepsilon}_{i_{1} a_{2} a_{3}}\left(\mathbf{x}_{l_1}\right)_{\tilde{J}, a_{2}, i_{2}, I\backslash \{i_{2}\}} \left(\mathbf{y}_{l_2}\right)_{\tilde{J}, a_{3}, J} \right)  \\
            & \quad\quad = \sum_{\tilde{J}} \left(\sum_{i_1 = i_2} \sum_{a_2 < a_3} \boldsymbol{\varepsilon}_{i_{1} a_{2} a_{3}} \left(\left(\mathbf{x}_{l_1}\right)_{\tilde{J}, a_{2}, i_{2}, I\backslash \{i_{2}\}} \left(\mathbf{y}_{l_2}\right)_{\tilde{J}, a_{3},J} - \left(\mathbf{x}_{l_1}\right)_{\tilde{J}, a_{3}, i_{2}, I\backslash \{i_{2}\}} \left(\mathbf{y}_{l_2}\right)_{\tilde{J}, a_{2},J} \right)\right)  \\
            & \quad\quad = \sum_{\tilde{J}} \left(\sum_{\substack{(i_1, a_2, a_3) = \\ (1, 2, 3) \\ (2, 1, 3) \\ (3, 1, 2)}} \boldsymbol{\varepsilon}_{i_{1} a_{2} a_{3}} \left(\left(\mathbf{x}_{l_1}\right)_{\tilde{J}, a_{2}, i_{1}, I\backslash \{i_{1}\}} \left(\mathbf{y}_{l_2}\right)_{\tilde{J}, a_{3},J} - \left(\mathbf{x}_{l_1}\right)_{\tilde{J}, a_{3}, i_{1}, I\backslash \{i_{1}\}} \left(\mathbf{y}_{l_2}\right)_{\tilde{J}, a_{2},J} \right)\right)  \\
            & \quad\quad = 0.
        \end{split}
    \end{equation*}
\end{minipage}}
\\
For the non-trivial terms inside $(*)_{\text{odd}}$ we have the lemma below, whose proof is identical to that of Lemma \ref{lem:diag_cancel_even}.

\begin{lemma} \label{lem:diag_cancel_odd} Let $0 \leq m \leq \mathrm{min}(l_1,l_2) - k - 2$. For each tuple $\left(E_{l_{3}}, I_{l_{1} - (k+m)}, J_{l_{2} - (k+m)}, K_{l_{3} - 2m}\right)$, there exist $2l_{3}-2m-1$ tuples of $\left(E_{l_{3}}, I_{l_{1} - (k+m+1)}, J_{l_{2} - (k+m+1)}, K_{l_{3} - 2(m+1)}\right)$ such that the trace for each of the tuples equals to the trace for $\left(E_{l_{3}}, I_{l_{1} - (k+m)}, J_{l_{2} - (k+m)}, K_{l_{3} - 2m}\right)$.
\end{lemma}

\begin{proof}
Without loss of generality, we may assume
\begin{equation*}
    \begin{split}
        E_{l_{3}} &= \{i_{k_{0}}\} \subset \big\{i_{1}, \dots, i_{l_{3}}\big\}\\
        I_{l_{1} - (k+m)} &= \big\{i_{1}, i_{k_{2}}, \dots, i_{k_{l_{1}-(k+m)}}\big\} \subset \big\{i_{1}, \dots, i_{l_{3}}\big\} \backslash E_{l_{3}}, \\
        J_{l_{2} - (k+m)} &= \big\{i_{2}, i_{k_{l_{1}-(k+m)+2}}, \dots, i_{k_{l_{3} - 2m}} \big\} \subset \big\{i_{1}, \dots, i_{l_{3}}\big\} \backslash \left( I_{l_{1} - (k+m)} \cup E_{l_{3}}\right), \\
        K_{l_{3} - 2m} &= \big\{\big\{i_{k_{l_{3} - 2m+1}}, i_{k_{l_{3} - 2m+2}} \big\},  \dots, \big\{i_{k_{l_{3}-1}}, i_{k_{l_{3}}}\big\}\big\}.
    \end{split}
\end{equation*}
Then, $\operatorname{Tr}\left(\left(a\right)_{m}\right)$ has the following expression
\\
\scalebox{0.85}{
\begin{minipage}{\linewidth}
    \begin{equation*}
        \begin{split}
            & \operatorname{Tr}\left(\left(a\right)_{m}\right) \\
            & \quad = \sum_{i_{1} = i_{2}} \sum_{a_{2}, a_{3}} \sum_{j_{1}, \dots, j_{k+m}} \boldsymbol{\varepsilon}_{i_{k_{0}} a_{2} a_{3}} \left(\mathbf{x}_{l_1}\right)_{j_{1} \cdots j_{k+m}, a_{2} i_{1} i_{k_{3}} \cdots i_{k_{l_{1}-(k+m)}}} \left(\mathbf{y}_{l_2}\right)_{j_{1} \cdots j_{k+m}, a_{3} i_{2} i_{k_{l_{1}-(k+m)+3}} \cdots i_{k_{l_{3} - 2m}}} \mathbf{I}^{\otimes m}_{K} \\
            & \quad = \sum_{a_{2}, a_{3}} \sum_{\substack{j_{1}, \dots, j_{k+m}, \\ j_{k+m+1}}} \boldsymbol{\varepsilon}_{i_{k_{0}} a_{2} a_{3}} \left(\mathbf{x}_{l_1}\right)_{j_{1} \cdots j_{k+m+1}, a_{2} i_{k_{3}} \cdots i_{k_{l_{1}-(k+m)}}} \left(\mathbf{y}_{l_2}\right)_{j_{1} \cdots j_{k+m+1}, a_{3} i_{k_{l_{1}-(k+m)+3}} \cdots i_{k_{l_{3} - 2m}}} \mathbf{I}^{\otimes m}_{K}.
        \end{split}
    \end{equation*}
\end{minipage}}
\\
Here, we assume that $a_{2} \in I_{l_{1} - (k+m)}$ and $a_{3} \in J_{l_{2} - (k+m)}$. The same derivation holds for $\operatorname{Tr}\left(\left(b\right)_{m}\right)$. 

For $\operatorname{Tr}((c)_{m+1})$ and $m+1$, we have $l_{3}-2m-2$ of $\left(E_{l_{3}}, I_{l_{1} - (k+m+1)}, J_{l_{2} - (k+m+1)}, K_{l_{3} - 2(m+1)}\right)$ triplets whose trace equals to $\operatorname{Tr}\left(\left(a\right)_{m}\right)$. Indeed, like in the proof of Lemma \ref{lem:diag_cancel}, we let 
\begin{equation*}
    \begin{split}
        E_{l_{3}} \cup I_{l_{1} - (k+m+1)} \cup J_{l_{2} - (k+m+1)} &= \big\{i_{k_{0}}, i_{k_{2}}, \dots, i_{1} ,\dots, i_{k_{l_{1}-(k+m)}}, i_{k_{l_{1}-(k+m)+2}}, \dots, i_{k_{l_{3} - 2m}}\big\}, \\
        K_{l_{3} - 2(m+1)} &= \big\{\big\{i_{k_{l_3-2m+1}}, i_{k_{l_3-2m+2}}\big\}, \dots, \big\{i_{k_{l_3-1}}, i_{k_{l_3}}\big\}, \big\{i_{k_{s}}, i_{2}\big\}\big\}.
    \end{split}
\end{equation*}
Then, by letting $\tilde{J} = \big\{j_{1}, \dots, j_{k+m}, j_{k+m+1}\big\},$ we obtain
\begin{equation*}
    \begin{split}
        & \operatorname{Tr}((c)_{m+1}) \\ 
        & \quad\quad = \sum_{i_{1}=i_{2}} \sum_{s} \sum_{a_{2}, a_{3}} \sum_{\tilde{J}} \boldsymbol{\varepsilon}_{i_{k_{0}} a_{2} a_{3}}\left(\mathbf{x}_{l_1}\right)_{\tilde{J} a_{2} i_{k_{2}} \cdots i_{1} \cdots i_{k_{l_{1}-(k+m)}}} \left(\mathbf{y}_{l_2}\right)_{\tilde{J} a_{3} i_{k_{l_{1}-(k+m)+3}} \cdots i_{k_{l_{3} - 2m}}} \\ 
        & \hspace{73.5mm} \times \delta_{i_{k_{l_{3}-2m+1}} i_{k_{l_{3}-2m+2}}} \cdots \delta_{i_{k_{l_{3}-1}} i_{k_{l_{3}}}} \delta_{i_{k_{s}} i_{2}} \\
        & \quad\quad = \sum_{s} \sum_{a_{2}, a_{3}} \sum_{\tilde{J}} \boldsymbol{\varepsilon}_{i_{k_{0}} a_{2} a_{3}}\left(\mathbf{x}_{l_1}\right)_{\tilde{J} a_{2} i_{k_{2}} \cdots i_{k_{s}} \cdots i_{k_{l_{1}-(k+m)}}} \left(\mathbf{y}_{l_2}\right)_{\tilde{J} a_{3} i_{k_{l_{1}-(k+m)+3}} \cdots i_{k_{l_{3} - 2m}}} \\ 
        & \hspace{73.5mm} \times \delta_{i_{k_{l_{3}-2m+1}} i_{k_{l_{3}-2m+2}}} \cdots \delta_{i_{k_{l_{3}-1}} i_{k_{l_{3}}}} \\
        & \quad\quad = (l_{3} - 2m -2) \operatorname{Tr}\left(\left(a\right)_{m}\right).
    \end{split}
\end{equation*}
The same derivation applies to $\operatorname{Tr}((d)_{m+1})$, and we get $\operatorname{Tr}((d)_{m+1})= (l_{3} - 2m -2) \operatorname{Tr}\left(\left(a\right)_{m}\right)$.

For $\operatorname{Tr}((e)_{m+1})$, the same derivation as for $\operatorname{Tr}((d)_{m+1})$ and $\operatorname{Tr}((e)_{m+1})$ in the proof of Lemma \ref{lem:diag_cancel} applies. Therefore, we obtain
\begin{equation*}
    \operatorname{Tr}((e)_{m+1}) = (2m+3)\operatorname{Tr}\left(\left(a\right)_{m}\right).
\end{equation*}

Finally, we can write
\begin{equation*}
    \begin{split}
        & \operatorname{Tr}((c)_{m+1}) + \operatorname{Tr}((d)_{m+1}) + \operatorname{Tr}((e)_{m+1}) \\
        & \quad\quad = (l_3 - 2m -2) \operatorname{Tr}((a)_m) + (l_3 - 2m -2) \operatorname{Tr}((a)_m) + (2m+3) \operatorname{Tr}((a)_m) \\
        & \quad\quad = (2l_3-2m-1) \operatorname{Tr}((a)_m) = (2l_3-2m-1) \operatorname{Tr}((b)_m),
    \end{split}
\end{equation*}
which completes the proof of the Lemma.
\end{proof}

Noting that $\operatorname{Tr}\left(\left(a\right)_{m}\right)$ is doubled for every $m$ due to the factor $2^m$, the proof is completed similarly to the proof of Lemma~\ref{lem:diag_cancel_even}. This result completes the proof of the traceless property for \eqref{eq:product_odd}.

\section{Experiments and results \label{sec:experiments_results_appendix}}

\subsection{Description of the data sets \label{sec:datasets_appendix}}

\textbf{rMD17 data set.} The revised MD17 (rMD17) data set is a collection of structures, energies, and atomic forces of ten small organic molecules obtained from ab initio molecular dynamics (AIMD)~\cite{Christensen2020b}. These molecules are derived from the original MD17 data set~\cite{Christensen2020b, Chmiela2017, Schuett2017_2, Chmiela2018}, with 100,000 structures sampled for each. Our models are trained using 950 and 50 configurations for each molecule randomly sampled from the original data set using five random seeds, with 50 additional configurations randomly sampled for early stopping. We use the remaining configurations to test the final models. \Tabref{tab:rmd17-results} reports the mean absolute errors (MAEs) in total energies and atomic forces averaged over five independent runs, including the standard deviation between them.

\textbf{MD22 data set.} The MD22 data set includes structures, energies, and atomic forces of seven molecular systems derived from AIMD simulations at elevated temperatures, spanning four major classes of biomolecules and supramolecules~\cite{Chmiela2023}. These molecular systems range from a small peptide containing 42 atoms to a double-walled nanotube comprising 370 atoms. Characterized by complex intermolecular interactions, this data set was designed to challenge short-range models. The training set sizes used in this study are consistent with those in the original publication~\cite{Chmiela2023}, selected to ensure that the sGDML model stays within a target accuracy of 1 kcal/mol ($\approx 43.37$ meV). We randomly selected an additional set of 500 structures for each molecule in the data set for early stopping, while the remaining configurations were reserved for testing the final models. The corresponding training and validation data sets were randomly selected using three random seeds. \Tabref{tab:md22-results} reports MAEs in total energies per atom and atomic forces averaged over three independent runs, including the standard deviation between them.

\textbf{3BPA data set.} The 3BPA data set comprises structures, energies, and atomic forces of a flexible drug-like organic molecule obtained from AIMD at various temperatures~\cite{Kovacs2021}. The training data set consists of 500 configurations sampled at 300~K, while three separate test data sets are obtained from AIMD simulations at 300~K, 600~K, and 1200~K. An additional test data set provides energy values along dihedral rotations of the molecule. This test directly assesses the smoothness and accuracy of the potential energy surface, influencing properties such as binding free energies to protein targets. Our models are trained using 450 and 50 configurations randomly sampled from the training data set using five random seeds, with further 50 configurations reserved for early stopping. \Tabref{tab:3bpa-results} reports the root-mean-square errors (RMSEs) in total energies and atomic forces averaged over five independent runs for a training set size of 450 structures, including the standard deviation between them. \Tabref{tab:3bpa-results-50_configs} presents the corresponding results obtained with a training set size of 50 structures.

\textbf{Acetylacetone data set.} The acetylacetone data set includes structures, energies, and atomic forces of a small reactive molecule obtained from AIMD at various temperatures~\cite{Batatia2022design}. The training data set comprises configurations sampled at 300~K, while the test data sets are sampled at 300~K and 600~K. The generalization ability of final models is evaluated using an elevated temperature of 600~K and along two internal coordinates of the molecule: The hydrogen transfer path and a partially conjugated double bond rotation featuring a high rotation barrier. Our models are trained using 450 and 50 configurations randomly sampled from the training dataset using five random seeds, with further 50 configurations reserved for early stopping. \Tabref{tab:acac-results} reports RMSEs in total energies and atomic forces averaged over five independent runs for a training set size of 450 structures, including the standard deviation between them. \Tabref{tab:acac-results-50_configs} presents the corresponding results obtained with a training set size of 50 structures.

\textbf{Ta--V--Cr--W data set.} The Ta--V--Cr--W data set includes 0~K energies, atomic forces, and stresses for binaries, ternaries, and quaternary and near-melting temperature properties in four-component disordered high-entropy alloys~\cite{Gubaev2023}. In total, this benchmark data set comprises 6711 configurations, with energies, atomic forces, and stresses computed at the density functional theory (DFT) level. More precisely, there are 5680 0 K structures: 4491 binary, 595 ternary, and 594 quaternary structures, along with 1031 structures sampled from MD at 2500 K. Structure sizes range from 2 to 432 atoms in the periodic cell. All models are trained using 5373 configurations, with 4873 used for training and 500 for early stopping. The remaining configurations are reserved for testing the models' performance. The performance is evaluated separately using 0 K binaries, ternaries, quaternaries, and near-melting temperature four-component disordered alloys. The original data set provides ten different training-test splits. In our experiments, each training set is further split into training and validation subsets using a random seed. Furthermore, we use additional binary structures strained along the $[100]$ direction as a part of the test data set. Note that the final models, in this case, were obtained by training using the whole data set of 6711 configurations (training + test), 500 of which were reserved as a validation data set. \Tabref{tab:hea-results} reports RMSEs in total energies per atom and atomic forces averaged over ten independent runs, including the standard deviation between them.

\subsection{Training details \label{sec:training_appendix}}

All ICTP and MACE models employed in this work were trained on a single NVIDIA A100 GPU with 80 GB of RAM. Training times for ICTP and MACE models typically ranged from 30 minutes to a few days, depending on the data set, the data set size, and the employed precision. We used double precision for 3BPA and acetylacetone data sets and single precision for rMD17, in line with the original experiments~\cite{Batatia2022}. Double precision was also used for MD22, while single precision was employed in our Ta--V--Cr--W experiments. Unless stated otherwise, we used two message-passing layers and irreducible Cartesian tensors or spherical tensors of a maximal rank of $l_\mathrm{max} = 3$ to embed the directional information of atomic distance vectors. 

For ICTP models with the full (ICTP$_\mathrm{full}$) and symmetric (ICTP$_\mathrm{sym}$) product basis and MACE, we employ 256 uncoupled feature channels. Exceptions include our experiments with the 3BPA data set, aimed at investigating scaling and computational cost, and the Ta--V--Cr--W experiments, where we used eight and 32 feature channels, respectively. For ICTP models with the symmetric product basis evaluated in the latent feature space (ICTP$_\mathrm{sym+lt}$), we use 64 coupled feature channels for the Cartesian product basis and 256 for two-body features. Radial features are derived from eight Bessel basis functions with polynomial envelope for the cutoff with $p = 5$~\cite{Gasteiger2022c}. These features are fed into a fully connected NN of size $[64, 64, 64]$. We apply SiLU non-linearities to the outputs of the hidden layers~\cite{Elfwing2018, Ramachandran2018}. The readout function of the first message-passing layer is implemented as a linear layer. The readout function of the second layer is a single-layer fully connected NN with 16 hidden neurons. A cutoff radius of 5.0~\AA{} is used across all data sets except MD22, where we used a cutoff radius of 5.5~\AA{} for the double-walled nanotube and 6.0~\AA{} for the other molecules in the data set.

All parameters of ICTP and MACE models were optimized by minimizing the combined squared loss on training data $\mathcal{D}_\mathrm{train} = \left(\mathcal{X}_\mathrm{train}, \mathcal{Y}_\mathrm{train}\right)$, where $\mathcal{X}_\mathrm{train} = \{\mathcal{S}^{(k)}\}_{k=1}^{N_\mathrm{train}}$ and $\mathcal{Y}_\mathrm{train} = \{E_k^\mathrm{ref}, \{\mathbf{F}_{u,k}^\mathrm{ref}\}_{u=1}^{N_\mathrm{at}}, \boldsymbol{\sigma}_k^\mathrm{ref}\}_{k=1}^{N_\mathrm{train}}$
\begin{equation}
    \label{eq:loss}
    \begin{split}
    \mathcal{L}\left( \boldsymbol{\theta}, \mathcal{D}_\mathrm{train}\right) = \sum_{k=1}^{N_\mathrm{train}} \Bigg[ C_\mathrm{e} \Big\lVert E_k^\mathrm{ref} - E(S^{(k)}, \boldsymbol{\theta})\Big\rVert_2^2 & + C_\mathrm{f} \sum_{u=1}^{N_\mathrm{at}^{(k)}} \Big\lVert \mathbf{F}_{u,k}^\mathrm{ref} - \mathbf{F}_u\left(S^{(k)}, \boldsymbol{\theta}\right)\Big\rVert_2^2 \\ & + C_\mathrm{s} \Big\lVert V_k \boldsymbol{\sigma}_{k}^\mathrm{ref} - V_k \boldsymbol{\sigma}\left(S^{(k)}, \boldsymbol{\theta}\right)\Big\rVert_2^2\Bigg].
    \end{split}
\end{equation}
Here, $\boldsymbol{\sigma}_k^\mathrm{ref}$ is the stress tensor defined as $\boldsymbol{\sigma} = \frac{1}{V}\left.\nabla_{\boldsymbol{\epsilon}} E\right|_{\boldsymbol{\epsilon} = \mathbf{0}}$, where $E$ denotes total energy after a strain deformation with symmetric tensor $\boldsymbol{\epsilon} \in \mathbb{R}^{3 \times 3}$ and $V$ is the volume of the periodic box.

When training ICTP models on rMD17, 3BPA, and acetylacetone data sets, we neglected the stress loss and set $C_\mathrm{e}=1 / N_\mathrm{at}^{(k)}$ and $C_\mathrm{f} = 10~\AA^2$ to balance the relative contributions of total energies and atomic forces, respectively. For MACE and 3BPA/acetylacetone, $C_\mathrm{e}=1/(B \times N_\mathrm{at}^{(k)})$ and $C_\mathrm{f} = 1000/(B \times 3 \times N_\mathrm{at}^{(k)})~\AA^2$ were used with $B$ denoting the batch size. For ICTP models trained on the MD22 data set, we set $C_\mathrm{e}=10 / N_\mathrm{at}^{(k)}$ and $C_\mathrm{f} = 1~\AA^2$, using energies and forces in eV and eV/\text{\AA}, respectively. For the Ta--V--Cr--W dataset, the stress loss was incorporated into the combined loss in \eqref{eq:loss}, along with the energy and force losses. For ICTP, we used $C_\mathrm{e}=1 / N_\mathrm{at}^{(k)}$, $C_\mathrm{f} = 0.01~\AA^2$, and $C_\mathrm{s} = 0.001 / N_\mathrm{at}^{(k)}$ to balance the relative contributions of total energies, atomic forces, and stresses, respectively. For MACE, we chose $C_\mathrm{e}=1/(B \times N_\mathrm{at}^{(k)} \times N_\mathrm{at}^{(k)})$, $C_\mathrm{f} = 1/(B \times 3 \times N_\mathrm{at}^{(k)})~\AA^2$, and $C_\mathrm{s} = 0.05/(B \times 9 \times N_\mathrm{at}^{(k)} \times N_\mathrm{at}^{(k)})$. Here, $E(\mathcal{S}^{(k)}, \boldsymbol{\theta})$, $\mathbf{F}_u\left(\mathcal{S}^{(k)}, \boldsymbol{\theta}\right)$, and $\boldsymbol{\sigma}\left(S^{(k)}, \boldsymbol{\theta}\right)$ are total energies, atomic forces, and stresses predicted by ICTP or MACE.

All models for rMD17, 3BPA, and acetylacetone were trained for 2000 epochs using the AMSGrad variant of Adam~\cite{Reddi2018}, with default parameters of $\beta_1 = 0.9$, $\beta_2 = 0.999$, and $\varepsilon = 10^{-8}$. For MD22 and Ta--V--Cr--W, all models were trained for 1000 epochs. For rMD17, 3BPA, and acetylacetone data sets, we used a learning rate of 0.01 and a batch size of 5. For MD22 and Ta--V--Cr--W, we again chose a learning rate of 0.01 but a mini-batch of 2 and 32, respectively. For evaluations on the validation and test data sets, we used a batch size of 10 for rMD17, 3BPA, and acetylacetone, while mini-batches of 2 and 32 were used for MD22 and Ta--V--Cr--W, respectively.

The learning rate was reduced using an on-plateau scheduler based on the validation loss with a patience of 50 and a decay factor of 0.8 for all data sets except for MD22, for which we used a patience of 10. We utilize an exponential moving average with a weight of 0.99 for evaluation on the validation set and for the final model. Additionally, in line with MACE~\cite{Batatia2022}, we apply exponential weight decay of $5 \times 10^{-7}$ on the weights of Eqs.~(\ref{eq:product_basis}), (\ref{eq:messages_uncoupled}), and (\ref{eq:messages_coupled}). Furthermore, we incorporate a per-atom shift of total energies via the average per-atom energy over all the training configurations, including the energies of individual atoms for 3BPA and acetylacetone datasets. If no atomic energies are provided, as for rMD17, MD22, and Ta--V--Cr--W, the per-atom shift is obtained by solving a linear regression problem~\cite{Zaverkin2021b}. Additionally, a per-atom scale is determined as the root-mean-square of the components of the forces over the training configurations.

\subsection{Additional results \label{sec:results_appendix}}

\begin{table}[t!]
    \caption{\textbf{Inference times and memory consumption as a function of the tensor rank $L$ and the correlation order $\nu$ for the 3BPA data set.} All values for ICTP and MACE models are obtained by averaging over five independent runs. The standard deviation is provided if it is available. Best performances are highlighted in bold. Inference time and memory consumption are measured for a batch size of 10. Inference time is reported per structure in ms; memory consumption is provided for the entire batch in GB. \label{tab:3bpa-results-scaling}}
	\begin{center}
    \resizebox{\textwidth}{!}{
    \begin{tabular}{lrrrrrr}
    \toprule 
    			    & \multicolumn{2}{c}{$L=1$}                             & \multicolumn{2}{c}{$L=2$}                                   & \multicolumn{2}{c}{$L=3$}                                 \\
        \cmidrule(lr){2-3} \cmidrule(lr){4-5} \cmidrule(lr){6-7}
                    & ICTP$_\text{full}$        & MACE                      & ICTP$_\text{full}$            & MACE                        & ICTP$_\text{full}$           & MACE                       \\
        \midrule
        \multicolumn{7}{c}{Inference times}                                                                                                                                                           \\
        \midrule
        $\nu=1$     & \textbf{0.76 $\pm$ 0.17}  & 1.02 $\pm$ 0.03	        & \textbf{0.87 $\pm$ 0.18}	    & 1.38 $\pm$ 0.04	          & \textbf{0.98 $\pm$ 0.26}     & 1.88 $\pm$ 0.03            \\
        $\nu=2$     & \textbf{0.59 $\pm$ 0.20}  & 1.12 $\pm$ 0.03	        & \textbf{1.03 $\pm$ 0.21}	    & 1.52 $\pm$ 0.05	          & \textbf{1.34 $\pm$ 0.08}     & 2.0 $\pm$ 0.10             \\
        $\nu=3$     & \textbf{0.79 $\pm$ 0.22}  & 1.23 $\pm$ 0.03	        & \textbf{1.15 $\pm$ 0.08}	    & 1.67 $\pm$ 0.03	          & \textbf{1.85 $\pm$ 0.13}     & 2.23 $\pm$ 0.03            \\
        $\nu=4$     & \textbf{0.94 $\pm$ 0.17}  & 1.41 $\pm$ 0.11	        & \textbf{1.31 $\pm$ 0.21}	    & 1.83 $\pm$ 0.01	          & \textbf{2.07 $\pm$ 0.20}     & 2.53 $\pm$ 0.01            \\
        $\nu=5$     & \textbf{1.02 $\pm$ 0.17}  & 1.52 $\pm$ 0.08	        & \textbf{1.72 $\pm$ 0.07}	    & 2.26 $\pm$ 0.03	          & \textbf{3.61 $\pm$ 0.02}     & OOM                        \\
        $\nu=6$     & \textbf{1.0 $\pm$ 0.07}   & 1.77 $\pm$ 0.05	        & \textbf{1.83 $\pm$ 0.16}	    & 27.85 $\pm$ 0.01	          & \textbf{16.76 $\pm$ 0.35}    & OOM                        \\
        \midrule
        \multicolumn{7}{c}{Memory consumption}                                                                                                                                                        \\
        \midrule
        $\nu=1$     & 0.05 $\pm$ 0.00           & \textbf{0.04 $\pm$ 0.00}	& 0.08 $\pm$ 0.00	           & \textbf{0.06 $\pm$ 0.00}	  & 0.21 $\pm$ 0.00              & \textbf{0.13 $\pm$ 0.00}   \\
        $\nu=2$     & 0.05 $\pm$ 0.00           & \textbf{0.04 $\pm$ 0.00}	& 0.08 $\pm$ 0.00	           & \textbf{0.07 $\pm$ 0.00}	  & 0.28 $\pm$ 0.09              & \textbf{0.13 $\pm$ 0.00}   \\
        $\nu=3$     & 0.05 $\pm$ 0.00           & \textbf{0.04 $\pm$ 0.00}	& 0.10 $\pm$ 0.00	           & \textbf{0.08 $\pm$ 0.00}	  & 0.51 $\pm$ 0.03              & \textbf{0.23 $\pm$ 0.00}   \\
        $\nu=4$     & \textbf{0.05 $\pm$ 0.00}  & \textbf{0.05 $\pm$ 0.00}	& \textbf{0.18 $\pm$ 0.08}	   & 0.30 $\pm$ 0.00	          & \textbf{1.07 $\pm$ 0.10}     & 4.16 $\pm$ 0.00            \\
        $\nu=5$     & \textbf{0.05 $\pm$ 0.00}  & 0.07 $\pm$ 0.00	        & \textbf{0.35 $\pm$ 0.07}	   & 3.18 $\pm$ 0.00	          & \textbf{5.07 $\pm$ 0.02}     & OOM                        \\
        $\nu=6$     & \textbf{0.11 $\pm$ 0.09}  & 0.22 $\pm$ 0.00    	    & \textbf{0.93 $\pm$ 0.00}	   & 50.49 $\pm$ 0.00	          & \textbf{28.48 $\pm$ 0.03}    & OOM                        \\
	    \bottomrule 
	    \end{tabular}}
	\end{center}
\end{table}

\textbf{Scaling and computational cost.} \Tabref{tab:3bpa-results-scaling} provides the numerical results complementing \figref{fig:3bpa-results-scaling} in the main text. All results in \tabref{tab:3bpa-results-scaling} and \figref{fig:3bpa-results-scaling} were obtained using irreducible Cartesian tensors with a maximal rank of $l_\mathrm{max} = L$ to represent local atomic environments. Here, $L$ denotes the tensor rank of employed equivariant messages. We facilitated the exploration of larger $\nu$ values by setting the number of uncoupled feature channels to eight. In the MACE model, intermediate spherical tensors with a rank of $l > l_\mathrm{max}$ are used to construct the product basis. However, the pre-computation of generalized Clebsch--Gordan coefficients for $\nu > 4$, in some cases, would require more than 2~TB of RAM. Therefore, in our experiments, we fixed the maximum rank of intermediate tensors to $l = l_\mathrm{max}$. We also used the full product basis for ICTP to calculate the same number of $\nu$-fold tensor products, i.e., $\mathcal{K} = \mathrm{len}\left(\eta_\nu\right)$, as used in MACE with $l = l_\mathrm{max}$. Finally, we note that ICTP and MACE use different approaches to optimize their runtimes; however, the scaling with respect to the tensor rank and the correlation order is independent of these optimization methods.

\begin{table*}[t!]
	\caption{\textbf{Energy (E) and force (F) mean absolute errors (MAEs) for the MD22 data set.}\textsuperscript{\emph{a}} E- and F-MAE are given in meV/atom and meV/\AA, respectively. Results are shown for models trained using training set sizes defined in the original publication~\cite{Chmiela2023}. All values for ICTP models are obtained by averaging over three independent runs. We also use an additional subset of 500 configurations drawn randomly from the original data set for early stopping. The standard deviation is provided if available. Best performances, considering the standard deviation, are highlighted in bold.
	\label{tab:md22-results}}
	\begin{center}
    \resizebox{\textwidth}{!}{
      \begin{tabular}{lclllllrrr}
        \toprule
			    	                                    &   & ICTP$_\text{sym}$             & ViSNet-LSRM~\cite{Li2024} & Equiformer~\cite{Li2024}  & So3krates~\cite{Frank2024}   & MACE~\cite{Kovacs2023b}    & Allegro~\cite{Li2024}    & TorchMD-Net~\cite{Li2024} & sGDML~\cite{Chmiela2023} \\		    	                            
        \cmidrule(lr){1-2} \cmidrule(lr){3-10}
        \multirow[l]{2}{*}{Ac--Ala3--NHMe}              & E & 0.068 $\pm$ 0.000	            & 0.068		                & 0.085		                & 0.348		                   & \textbf{0.064}		         & 0.105                    & 0.116	                   & 0.403                     \\
                                                        & F & \textbf{3.28 $\pm$ 0.02}	    & 3.91		                & 3.49		                & 10.58		                   & 3.8		                & 4.63		               & 8.15                      & 34.26                     \\
        \cmidrule(lr){1-2} \cmidrule(lr){3-10}
        \multirow[l]{2}{*}{DHA}	                        & E & 0.080 $\pm$ 0.001	            & \textbf{0.068}            & 0.138		                & 0.293		                   & 0.102		                & 0.089		               & 0.093                     & 0.997                     \\
                                                        & F & 2.62 $\pm$ 0.01	            & 2.59		                & \textbf{2.19}		        & 10.49		                   & 2.8		                & 3.17		               & 5.24                      & 32.52                     \\
        \cmidrule(lr){1-2} \cmidrule(lr){3-10}
        \multirow[l]{2}{*}{Stachyose}	                & E & \textbf{0.053 $\pm$ 0.001}    & \textbf{0.053}            & 0.070	                    & 0.22	                       & 0.062	                    & 0.124	                   & 0.069                     & 1.995                     \\
                                                        & F & \textbf{2.51 $\pm$ 0.03}	    & 3.33                      & 2.75	                    & 18.86		                   & 3.8		                & 4.21	                   & 8.33                      & 29.49                     \\
        \cmidrule(lr){1-2} \cmidrule(lr){3-10}
	    \multirow[l]{2}{*}{AT--AT}	                    & E & \textbf{0.057 $\pm$ 0.002}    & \textbf{0.056}		    & 0.095				        & 0.129			               & 0.079			            & 0.103	                   & 0.081                     & 0.52                      \\
                                                        & F & \textbf{3.02 $\pm$ 0.07}      & 3.39		                & 4.16				        & 9.37		                   & 4.3			            & 4.13                     & 8.83                      & 29.92                     \\
        \cmidrule(lr){1-2} \cmidrule(lr){3-10}
        \multirow[l]{2}{*}{AT--AT--CG--CG}	            & E & 0.045 $\pm$ 0.001             & \textbf{0.042}		    & 0.055		                & 0.127		                   & 0.058	                    & 0.145                    & 0.076                     & 0.52                      \\
                                                        & F & \textbf{3.37 $\pm$ 0.05}      & 4.61		                & 5.43	                    & 14.40			               & 5.0		                & 5.55                     & 14.13                     & 30.36                     \\
        \cmidrule(lr){1-2} \cmidrule(lr){3-10}
        \multirow[l]{2}{*}{Buckyball catcher}	        & E & 0.123 $\pm$ 0.001		        & 0.124		                & 0.117 	                & \textbf{0.112}		       & 0.141		                & 0.154	                   & 0.152                     & 0.343                     \\
                                                        & F & \textbf{3.58 $\pm$ 0.07}      & 4.45                      & 4.83			            & 10.28			               & 3.7			            & 3.85                     & 14.39                     & 29.49                     \\
        \cmidrule(lr){1-2} \cmidrule(lr){3-10}
        Double-walled	                                & E & 0.352 $\pm$ 0.003\textsuperscript{\emph{b}}		        & 0.214	                    & 0.140			            & \textbf{0.116}		       & 0.194	                    & 0.259                    & 0.173                     & 0.468                     \\
        nanotube                                        & F & 12.64 $\pm$ 0.23\textsuperscript{\emph{b}}              & 14.71	                    & \textbf{11.91}			& 31.53			               & 12.0			            & 14.87                    & 43.5                      & 22.55                     \\
	    \bottomrule 
	    \end{tabular}
      }
	\end{center}
    \footnotesize{\textsuperscript{\emph{a}} ICTP, MACE, Equiformer, Allegro, and TorchMD-Net are short-range models (i.e., they use local or semi-local atomic representations), ViSNet-LSRM and SO3krates include long-range information, and sGDML is a global model.} \\
    \footnotesize{\textsuperscript{\emph{b}} For consistency, these results are obtained for relative energy and force weights of $C_\mathrm{e}=10 / N_\mathrm{at}^{(k)}$ and $C_\mathrm{f} = 1~\AA^2$, based on our experiments with the Ac--Ala3--NHMe molecule. For comparison, we also trained ICTP for 800 epochs with $C_\mathrm{e}=100 / N_\mathrm{at}^{(k)}$ and $C_\mathrm{f} = 1~\AA^2$, achieving MAEs of 0.209 $\pm$ 0.007~meV/atom for energy and 13.98 $\pm$ 0.30~meV/\AA{} for force.}
\end{table*}

\textbf{Molecular dynamics trajectories.} \Tabref{tab:md22-results} presents the energy and force mean absolute errors (MAEs) for the ICTP models trained on the MD22 data set. This data set was designed to challenge short-range potential models---i.e., those based on local or semi-local atomic representations---and includes large molecular systems with complex intermolecular interactions. We compare errors obtained for ICTP models with those of state-of-the-art approaches that incorporate global, short- and long-range information. From \tabref{tab:md22-results}, we see that ICTP achieves accuracy in predicted energies and atomic forces on par with or better than state-of-the-art methods. The ICTP model uses a cutoff of 6.0~\AA{} (5.5~\AA{} for the double-walled nanotube), resulting in a receptive field of 12.0~\AA{} (11.0~\AA{}), considering the two message-passing layers. This receptive field is, in most cases, smaller than the diameter of molecular systems in the data set. Thus, ICTP is, at most, a semi-local potential model---similar to MACE, which used a 5.0~\AA{} cutoff and two message-passing layers.

We chose a larger cutoff for ICTP than the one used by MACE motivated by the results in the recent work~\cite{Kovacs2023}. In particular, the authors compared MACE to VisNet-LSRM---the best model reported to date for MD22, which employs mixed short- and long-range information in their message passing---and reported that MACE can achieve lower force errors. However, VisNet-LSRM typically had lower energy errors, attributed to the improvement from considering atomic interactions beyond 10.0~\AA{}. Using a larger cutoff radius for ICTP compared to MACE, we expected results closer to those of VisNet-LSRM. 
\Tabref{tab:md22-results} shows that using a receptive field of 12.0~\AA{} is, in most cases, sufficient to achieve accuracy in predicted energies close to the one obtained by VisNet-LSRM.

\begin{table*}[t!]
	\caption{\textbf{Energy (E) and force (F) root-mean-square errors (RMSEs) for the 3BPA data set (results for $N_\mathrm{train}=50$).} E- and F-RMSE are given in meV and meV/\AA, respectively. Results are shown for models trained using 50 molecules randomly drawn from the training data set collected at 300~K, with further 50 used for early stopping. All ICTP and MACE results are obtained by averaging over five independent runs, with the standard deviation provided if available. Best performances, considering the standard deviation, are highlighted in bold.
	\label{tab:3bpa-results-50_configs}}
	\begin{center}
        \begin{tabular}{lcrrrr}
        \toprule 
			    	                          &   & ICTP$_\text{full}$          & ICTP$_\text{sym}$            & ICTP$_{\text{sym} + \text{lt}}$   & MACE                               \\
        \cmidrule(lr){1-2} \cmidrule(lr){3-6}
        \multirow[l]{2}{*}{300~K}             & E & \textbf{14.98 $\pm$ 1.62}	& \textbf{13.43 $\pm$ 1.00}	   & \textbf{16.03 $\pm$ 1.26}		   & \textbf{14.54 $\pm$ 1.02}          \\
                                              & F & \textbf{37.21 $\pm$ 2.14}	& \textbf{37.27 $\pm$ 1.63}	   & \textbf{38.38 $\pm$ 1.55}	       & \textbf{37.68 $\pm$ 1.33}          \\
        \cmidrule(lr){1-2} \cmidrule(lr){3-6}
        \multirow[l]{2}{*}{600~K}	          & E & \textbf{31.68 $\pm$ 3.56}	& \textbf{31.63 $\pm$ 3.94}	   & \textbf{30.74 $\pm$ 1.82}		   & \textbf{30.71 $\pm$ 3.43}          \\
                                              & F & \textbf{69.87 $\pm$ 3.04}	& \textbf{68.87 $\pm$ 2.94}	   & \textbf{69.62 $\pm$ 2.88}		   & \textbf{69.88 $\pm$ 3.88}          \\
        \cmidrule(lr){1-2} \cmidrule(lr){3-6}
        \multirow[l]{2}{*}{1200~K}	          & E & \textbf{92.16 $\pm$ 9.33}	& \textbf{86.0 $\pm$ 12.03}	   & \textbf{78.51 $\pm$ 8.88}		   & \textbf{83.99 $\pm$ 8.20}          \\
                                              & F & \textbf{157.72 $\pm$ 6.53}	& \textbf{153.16 $\pm$ 9.62}   & \textbf{151.37 $\pm$ 9.59}		   & \textbf{154.46 $\pm$ 10.84}        \\
        \cmidrule(lr){1-2} \cmidrule(lr){3-6}
	    \multirow[l]{2}{*}{Dihedral slices}	  & E & \textbf{33.69 $\pm$ 8.03}	& \textbf{31.06 $\pm$ 6.27}	   & \textbf{33.44 $\pm$ 5.56}		   & \textbf{27.79 $\pm$ 7.41}          \\
                                              & F & \textbf{47.12 $\pm$ 5.27}   & \textbf{47.68 $\pm$ 2.21}	   & \textbf{49.08 $\pm$ 4.05}         & \textbf{48.91 $\pm$ 4.71}          \\
	    \bottomrule 
	    \end{tabular}
	\end{center}
\end{table*}

\begin{table*}[t!]
	\caption{\textbf{Energy (E) and force (F) root-mean-square errors (RMSEs) for the 3BPA data set (results for $\nu = 1$).} E- and F-RMSE are given in meV and meV/\AA, respectively. Results are shown for models trained using 450 configurations randomly drawn from the training data set collected at 300~K, with further 50 used for early stopping. For all models, $\nu=1$ is used. Thus, the employed models rely exclusively on two-body interactions. All ICTP and MACE results are obtained by averaging over five independent runs. Best performances, considering the standard deviation, are highlighted in bold. Inference time and memory consumption are measured for a batch size of 100. Inference time is reported per structure in ms, while memory consumption is provided for the entire batch in GB.}
	\label{tab:3bpa-results-no_product}
	\begin{center}
    \resizebox{\textwidth}{!}{
      \begin{tabular}{lcrrrrr}
        \toprule 
			    	                          &   & ICTP$_\text{full}$          & ICTP$_\text{sym}$             & ICTP$_{\text{sym} + \text{lt}}$  & MACE                           \\
        \cmidrule(lr){1-2} \cmidrule(lr){3-6}
        \multirow[l]{2}{*}{300~K}             & E & \textbf{12.90 $\pm$ 1.06}	& \textbf{12.90 $\pm$ 1.06}	    & \textbf{14.97 $\pm$ 1.64}		   & \textbf{13.50 $\pm$ 1.71}      \\
                                              & F & \textbf{29.90 $\pm$ 0.25}	& \textbf{29.90 $\pm$ 0.25}	    & 30.93 $\pm$ 0.47	               & \textbf{30.18 $\pm$ 0.38}      \\
        \cmidrule(lr){1-2} \cmidrule(lr){3-6}
        \multirow[l]{2}{*}{600~K}	          & E & \textbf{29.97 $\pm$ 0.94}	& \textbf{29.97 $\pm$ 0.94}	    & \textbf{31.64 $\pm$ 1.83}		   & \textbf{31.32 $\pm$ 2.16}      \\
                                              & F & \textbf{62.80 $\pm$ 0.45}	& \textbf{62.80 $\pm$ 0.45}	    & 64.54 $\pm$ 0.95		           & \textbf{63.04 $\pm$ 0.73}      \\
        \cmidrule(lr){1-2} \cmidrule(lr){3-6}
        \multirow[l]{2}{*}{1200~K}	          & E & \textbf{81.03 $\pm$ 1.64}	& \textbf{81.03 $\pm$ 1.64}	    & \textbf{79.26 $\pm$ 4.66}		   & \textbf{81.54 $\pm$ 2.02}      \\
                                              & F & \textbf{146.96 $\pm$ 1.30}	& \textbf{146.96 $\pm$ 1.30}    & 151.45 $\pm$ 5.21		           & 149.44 $\pm$ 1.94              \\
        \cmidrule(lr){1-2} \cmidrule(lr){3-6}
	    \multirow[l]{2}{*}{Dihedral slices}	  & E & \textbf{22.84 $\pm$ 2.96}	& \textbf{22.84 $\pm$ 2.96}	    & 29.16 $\pm$ 7.96		           & 28.08 $\pm$ 4.04               \\
                                              & F & \textbf{48.82 $\pm$ 5.25}   & \textbf{48.82 $\pm$ 5.25}	    & \textbf{52.53 $\pm$ 5.27}        & \textbf{49.62 $\pm$ 2.92}      \\
        \cmidrule(lr){1-2} \cmidrule(lr){3-6}
	    Inference time                        &   & 2.63 $\pm$ 0.02			    & 2.62 $\pm$ 0.02			    & \textbf{2.53 $\pm$ 0.01}		   & 2.96 $\pm$ 0.06                \\
	    \cmidrule(lr){1-2} \cmidrule(lr){3-6}
	    Memory consumption                    &   & 32.57 $\pm$ 0.00			& 32.57 $\pm$ 0.00			    & 32.34 $\pm$ 0.09		           & \textbf{23.32 $\pm$ 0.00}      \\
	    \bottomrule 
	    \end{tabular}
      }
	\end{center}
\end{table*}

\begin{figure}[t!]
    \begin{center}
        \includegraphics[width=\textwidth]{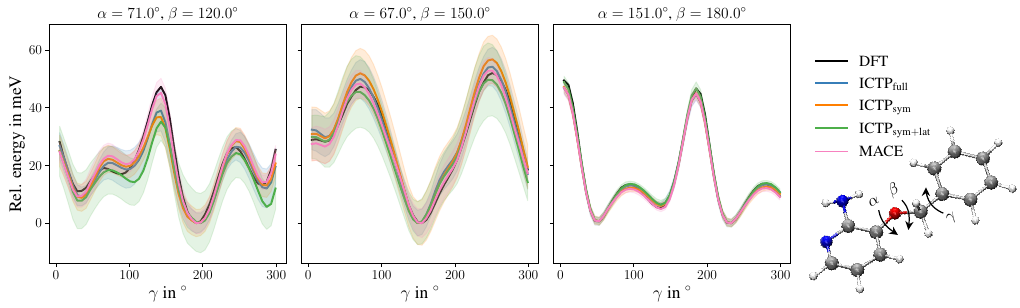}
    \end{center}
    \caption{\textbf{Potential energy profiles for three cuts through the 3BPA molecule's potential energy surface (results for $N_\mathrm{train}=450$).} All models are trained using 450 configurations, and the remaining 50 are used for early stopping. The 3BPA molecule, including the three dihedral angles ($\alpha$, $\beta$, and $\gamma$), provided in degrees~$^\circ$, is shown as an inset. The color code of the inset molecule is C grey, O red, N blue, and H white. The reference potential energy profile (DFT) is shown in black. Each profile is shifted such that each model's lowest energy is zero. Shaded areas denote standard deviations across five independent runs.}
    \label{fig:3bpa-results}
\end{figure}

\textbf{Extrapolation to out-of-domain data.} \Tabref{tab:3bpa-results-50_configs} demonstrates total energy and atomic force RMSEs obtained for ICTP and MACE models trained using 50 configurations randomly drawn from the original data set. ICTP and MACE models perform similarly, considering the standard deviation obtained across five independent runs. However, ICTP models often have lower mean RMSE values compared to MACE. Furthermore, \tabref{tab:3bpa-results-no_product} presents the total energy and atomic force RMSEs for ICTP and MACE models that use $\nu=1$ to compute the product basis. Thus, we provide results for models which rely exclusively on two-body interactions. We note that for $\nu=1$, ICTP$_\text{full}$ and ICTP$_\text{sym}$ are identical; though, we include both results for completeness.

\Figref{fig:3bpa-results} compares potential energy profiles obtained with ICTP and MACE models trained using 450 configurations. Potential energy cuts at $\beta=120^\circ$ and $\beta=180^\circ$ are easier tasks for MLIPs, as there are training points in the data set with similar combinations of dihedral angles~\cite{Batatia2022design}. In contrast, the potential energy cut at $\beta=150^\circ$ is more challenging, with no training points close to it. Notably, \figref{fig:3bpa-results} shows that all models produce smooth potential energy profiles close to the reference ones (DFT) for all values of $\beta$. These results again demonstrate excellent extrapolation capabilities of irreducible Cartesian models that are on par with the spherical MACE model.

\begin{table*}[t!]
	\caption{\textbf{Energy (E) and force (F) root-mean-square errors (RMSEs) for the acetylacetone data set (results for $N_\mathrm{train}=50$).} E- and F-RMSE are given in meV and meV/\AA, respectively. Results are shown for models trained using 50 configurations randomly drawn from the training data set collected at 300~K, with further 50 used for early stopping. All ICTP and MACE results are obtained by averaging over five independent runs. Best performances, considering the standard deviation, are highlighted in bold.
	\label{tab:acac-results-50_configs}}
	\begin{center}
        \begin{tabular}{lcrrrr}
        \toprule 
			    	                &   & ICTP$_\text{full}$        & ICTP$_\text{sym}$          & ICTP$_{\text{sym} + \text{lt}}$  & MACE                       \\
        \cmidrule(lr){1-2} \cmidrule(lr){3-6}
        \multirow[l]{2}{*}{300~K}   & E & \textbf{4.42 $\pm$ 0.39}	& \textbf{4.45 $\pm$ 0.36}	 & \textbf{4.38 $\pm$ 0.24}         & \textbf{4.22 $\pm$ 0.52}   \\
                                    & F & \textbf{28.85 $\pm$ 3.00}	& \textbf{28.28 $\pm$ 1.45}	 & \textbf{29.17 $\pm$ 1.65}        & \textbf{28.38 $\pm$ 2.74}  \\
        \cmidrule(lr){1-2} \cmidrule(lr){3-6}
        \multirow[l]{2}{*}{600~K}	& E & \textbf{17.61 $\pm$ 3.14}	& \textbf{16.13 $\pm$ 1.37}	 & \textbf{17.3 $\pm$ 2.05}         & \textbf{17.53 $\pm$ 3.58}  \\
                                    & F & \textbf{75.09 $\pm$ 8.70}	& \textbf{69.99 $\pm$ 5.12}	 & \textbf{74.96 $\pm$ 4.68}        & \textbf{74.93 $\pm$ 9.91}  \\
	    \bottomrule 
	    \end{tabular}
	\end{center}
\end{table*}

\begin{figure}[t!]
    \begin{center}
        \includegraphics[width=\textwidth]{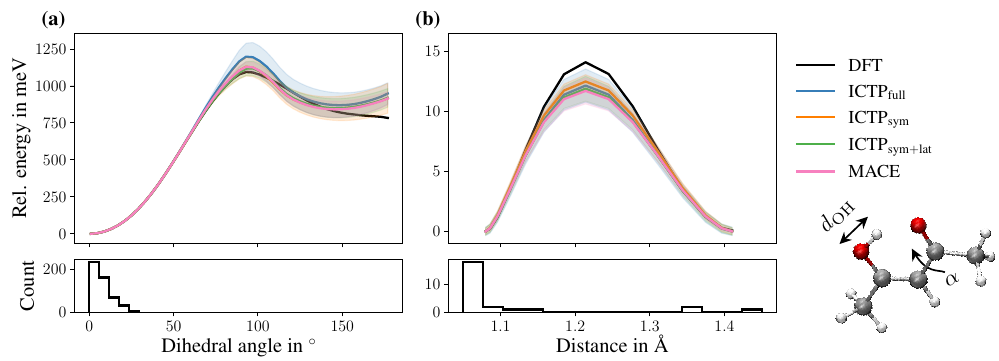}
    \end{center}
    \caption{\textbf{Potential energy profiles of (a) the dihedral angle describing the rotation around the C-C bond and (b) hydrogen transfer between two oxygen atoms (results for $N_\mathrm{train}=450$).} All models are trained using 450 molecules, and the remaining 50 are used for early stopping. The acetylacetone molecule, including the dihedral angle in degrees $^\circ$ describing the rotation around the C-C bond ($\alpha$), is shown as an inset in (a). The color code of the inset molecule is C grey, O red, and H white. The reference potential energy profile (DFT) is shown in black. Each profile is shifted such that each model's lowest energy is zero. The histograms demonstrate the distribution of dihedral angles and O-H distances in the training data. Shaded areas denote standard deviations across five independent runs.}
    \label{fig:acac-results}
\end{figure}

\begin{figure}[t!]
    \begin{center}
        \includegraphics[width=\textwidth]{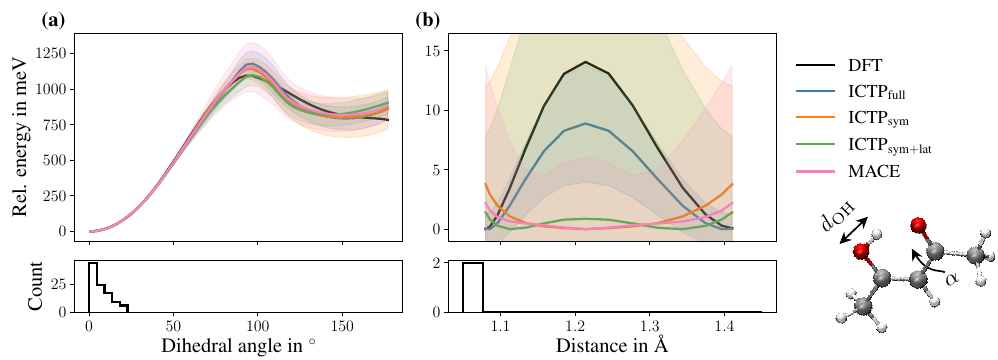}
    \end{center}
    \caption{\textbf{Potential energy profiles of (a) the dihedral angle describing the rotation around the C-C bond and (b) hydrogen transfer between two oxygen atoms (results for $N_\mathrm{train}=50$).} All models are trained using 50 molecules, and additional 50 are used for early stopping. The acetylacetone molecule, including the dihedral angle in degrees $^\circ$ describing the rotation around the C-C bond ($\alpha$), is shown as an inset in (a). The color code of the inset molecule is C grey, O red, and H white. The reference potential energy profile (DFT) is shown in black. Each profile is shifted such that each model's lowest energy is zero. The histograms demonstrate the distribution of dihedral angles and O-H distances in the training data set for one random seed used to split the original data. Shaded areas denote standard deviations across five independent runs.}
    \label{fig:acac-results-50_configs}
\end{figure}

\textbf{Flexibility and reactivity.} \Tabref{tab:acac-results-50_configs} demonstrates total energy and atomic force RMSEs obtained for ICTP and MACE models trained using 50 configurations randomly drawn from the original data set. Similar to the 3BPA data set, ICTP and MACE models demonstrate comparable accuracy in predicted energies and forces, considering the standard deviation obtained across five independent runs.

\Figref{fig:acac-results} further investigates the generalization capabilities of ICTP models trained using 450 configurations, demonstrating potential energy profiles for the rotation around the C-C bond, i.e., the C-C-C-O dihedral angle ($\alpha$), and for the hydrogen transfer (i.e., the O--H distance $d_\mathrm{OH}$ in \figref{fig:acac-results}). The training data set encompasses dihedral angles less than $30^\circ$. Furthermore, the energy barrier of 1~eV is outside the energy range of the training data set obtained at 300~K. As for the hydrogen transfer, the training data does not contain transition geometries, but the reaction still occurs in a region that is not too far from the training data. Overall, ICTP models perform on par with MACE for predicting potential energy profiles for the rotation around the corresponding C-C bond and for the hydrogen transfer.

\Figref{fig:acac-results-50_configs} shows potential energy profiles for MLIPs trained with 50 configurations, similar to \figref{fig:acac-results}. Notably, \figref{fig:acac-results} (b) demonstrates that ICTP$_\text{full}$ is the only MLIP consistently producing the potential energy profile for the hydrogen transfer close to the reference (DFT). This task is particularly challenging, as most data splits do not include configurations sufficiently close to the transition structure as in \figref{fig:acac-results}.

\begin{table}
\caption{\textbf{Energy (E) and force (F) root-mean-square errors (RMSEs) for the Ta--V--Cr--W data set.} E- and F-RMSEs are given in meV/atom and eV/\AA. Results are obtained by averaging over ten splits of the original data set, except for the deformed structures. For the latter, the results are obtained using the whole data set (training $+$ test). For the ICTP, MACE, and GM-NN models, we randomly selected a validation data set of 500 structures from the corresponding training data sets. Best performances, considering the standard deviation, are highlighted in bold. Inference time and memory consumption are measured for a batch size of 50. Inference time\textsuperscript{\emph{a}} is reported per atom in $\mu$s; memory consumption is provided for the entire batch in GB. \label{tab:hea-results}}
\begin{center}
\resizebox{\textwidth}{!}{
\begin{tabular}{lcrrrrrrrrrr}  
\toprule
Subsystem                                   &    & ICTP$_\text{sym}$ ($L=2$)        & ICTP$_\text{sym}$ ($L=1$)         & ICTP$_\text{sym}$ ($L=0$)         & MACE ($L=2$)                  & MACE ($L=1$)                  & MACE ($L=0$)              & ICTP$_\text{sym}$ ($\nu=2$)       & MTP~\cite{Gubaev2023} & GM-NN~\cite{Gubaev2023}   & EAM~\cite{Gubaev2023} \\
\cmidrule(lr){1-2} \cmidrule(lr){3-9} \cmidrule(lr){10-12}
\multirow[l]{2}{*}{TaV}                     & E  & \textbf{1.02 $\pm$ 0.27}         & \textbf{1.21 $\pm$ 0.54}          & 1.65 $\pm$ 1.06                   & 1.72 $\pm$ 0.67               & 1.76 $\pm$ 0.53               & 2.24 $\pm$ 1.34           & \textbf{1.24 $\pm$ 0.50}          & 1.94                  & 1.54                      & 32.00                 \\
                                            & F  & \textbf{0.020 $\pm$ 0.002}       & 0.022 $\pm$ 0.002                 & 0.024 $\pm$ 0.002                 & \textbf{0.019 $\pm$ 0.002}    & \textbf{0.020 $\pm$ 0.003}    & 0.022 $\pm$ 0.002         & 0.023 $\pm$ 0.002                 & 0.050                 & 0.029                     & 0.404                 \\
\cmidrule(lr){1-2} \cmidrule(lr){3-9} \cmidrule(lr){10-12}
\multirow[l]{2}{*}{TaCr}                    & E  & \textbf{1.81 $\pm$ 0.29}         & \textbf{1.94 $\pm$ 0.23}          & 2.13 $\pm$ 0.19                   & 3.26 $\pm$ 0.42               & 3.31 $\pm$ 0.44               & 4.18 $\pm$ 0.56           & 2.40 $\pm$ 0.33                    & 3.26                  & 2.98                      & 43.60                 \\
                                            & F  & \textbf{0.025 $\pm$ 0.007}       & \textbf{0.024 $\pm$ 0.006}        & 0.027 $\pm$ 0.005                 & 0.029 $\pm$ 0.01              & \textbf{0.026 $\pm$ 0.007}    & 0.028 $\pm$ 0.007         & \textbf{0.026 $\pm$ 0.006}        & 0.057                 & 0.038                     & 0.343                 \\
\cmidrule(lr){1-2} \cmidrule(lr){3-9} \cmidrule(lr){10-12}
\multirow[l]{2}{*}{TaW}                     & E  & \textbf{1.75 $\pm$ 0.11}         & 1.87 $\pm$ 0.14                   & 2.45 $\pm$ 0.31                   & 2.73 $\pm$ 0.53               & 3.21 $\pm$ 0.55               & 3.57 $\pm$ 0.48           & 2.19 $\pm$ 0.54                   & 2.72                  & 2.99                      & 44.80                 \\
                                            & F  & \textbf{0.017 $\pm$ 0.002}       & \textbf{0.018 $\pm$ 0.002}        & 0.020 $\pm$ 0.002                 & \textbf{0.017 $\pm$ 0.002}    & \textbf{0.018 $\pm$ 0.002}    & 0.019 $\pm$ 0.002         & \textbf{0.018 $\pm$ 0.002}        & 0.038                 & 0.025                     & 0.248                 \\
\cmidrule(lr){1-2} \cmidrule(lr){3-9} \cmidrule(lr){10-12}
\multirow[l]{2}{*}{VCr}                     & E  & \textbf{1.74 $\pm$ 1.20}          & 2.52 $\pm$ 2.43                   & \textbf{2.13 $\pm$ 1.24}          & \textbf{2.19 $\pm$ 0.78}      & 2.82 $\pm$ 1.28               & 3.11 $\pm$ 1.42           & \textbf{1.89 $\pm$ 1.27}          & \textbf{2.29}         & 2.82                      & 44.80                 \\
                                            & F  & \textbf{0.016 $\pm$ 0.002}       & 0.018 $\pm$ 0.001                 & 0.019 $\pm$ 0.001                 & \textbf{0.016 $\pm$ 0.001}    & \textbf{0.017 $\pm$ 0.001}    & 0.018 $\pm$ 0.002         & 0.019 $\pm$ 0.001                 & 0.036                 & 0.025                     & 0.270                 \\
\cmidrule(lr){1-2} \cmidrule(lr){3-9} \cmidrule(lr){10-12}
\multirow[l]{2}{*}{VW}                      & E  & \textbf{1.32 $\pm$ 0.2}          & \textbf{1.46 $\pm$ 0.16}          & 1.69 $\pm$ 0.21                   & 1.90 $\pm$ 0.19                & 1.94 $\pm$ 0.23               & 2.42 $\pm$ 0.24           & 1.61 $\pm$ 0.16                   & 2.50                  & 2.00                      & 21.30                 \\
                                            & F  & \textbf{0.014 $\pm$ 0.002}       & \textbf{0.015 $\pm$ 0.002}        & 0.018 $\pm$ 0.003                 & \textbf{0.014 $\pm$ 0.002}    & \textbf{0.015 $\pm$ 0.002}    & 0.017 $\pm$ 0.002         & 0.016 $\pm$ 0.002                 & 0.037                 & 0.023                     & 0.292                 \\
\cmidrule(lr){1-2} \cmidrule(lr){3-9} \cmidrule(lr){10-12}
\multirow[l]{2}{*}{CrW}                     & E  & \textbf{2.18 $\pm$ 0.93}         & \textbf{2.45 $\pm$ 1.53}          & 2.76 $\pm$ 1.15                   & \textbf{2.31 $\pm$ 1.18}      & 2.84 $\pm$ 0.98               & 4.14 $\pm$ 1.38           & 3.12 $\pm$ 1.90                   & 4.35                  & 2.87                      & 23.40                 \\
                                            & F  & \textbf{0.018 $\pm$ 0.004}       & \textbf{0.020 $\pm$ 0.005}        & 0.024 $\pm$ 0.008                 & \textbf{0.020 $\pm$ 0.009}     & \textbf{0.019 $\pm$ 0.006}   & 0.023 $\pm$ 0.007         & 0.022 $\pm$ 0.006                 & 0.041                 & 0.029                     & 0.248                 \\
\cmidrule(lr){1-2} \cmidrule(lr){3-9} \cmidrule(lr){10-12}
\multirow[l]{2}{*}{TaVCr}                   & E  & \textbf{0.79 $\pm$ 0.08}         & 0.92 $\pm$ 0.17                   & 1.00 $\pm$ 0.24                   & 2.26 $\pm$ 0.54               & 2.71 $\pm$ 0.66               & 3.92 $\pm$ 0.77           & 0.97 $\pm$ 0.13                   & 2.43                  & 1.97                      & 34.10                 \\
                                            & F  & 0.027 $\pm$ 0.001                & 0.029 $\pm$ 0.002                 & 0.033 $\pm$ 0.002                 & \textbf{0.023 $\pm$ 0.002}    & \textbf{0.024 $\pm$ 0.001}    & 0.028 $\pm$ 0.001         & 0.031 $\pm$ 0.002                 & 0.054                 & 0.045                     & 0.313                 \\
\cmidrule(lr){1-2} \cmidrule(lr){3-9} \cmidrule(lr){10-12}
\multirow[l]{2}{*}{TaVW}                    & E  & \textbf{1.00 $\pm$ 0.20}          & \textbf{0.98 $\pm$ 0.18}          & 1.26 $\pm$ 0.23                   & 1.80 $\pm$ 0.35                & 1.97 $\pm$ 0.44               & 2.29 $\pm$ 0.86           & \textbf{0.95 $\pm$ 0.25}          & 1.67                  & 1.70                      & 39.60                 \\
                                            & F  & \textbf{0.021 $\pm$ 0.001}       & 0.022 $\pm$ 0.001                 & 0.025 $\pm$ 0.001                 & \textbf{0.021 $\pm$ 0.002}    & 0.023 $\pm$ 0.001             & 0.026 $\pm$ 0.001         & 0.023 $\pm$ 0.001                 & 0.043                 & 0.034                     & 0.321                 \\
\cmidrule(lr){1-2} \cmidrule(lr){3-9} \cmidrule(lr){10-12}
\multirow[l]{2}{*}{TaCrW}                   & E  & \textbf{1.16 $\pm$ 0.15}         & \textbf{1.28 $\pm$ 0.13}          & 1.58 $\pm$ 0.29                   & 1.67 $\pm$ 0.38               & 1.48 $\pm$ 0.50               & 2.08 $\pm$ 0.57           & \textbf{1.24 $\pm$ 0.11}          & 2.08                  & 2.19                      & 23.60                 \\
                                            & F  & \textbf{0.022 $\pm$ 0.001}       & 0.024 $\pm$ 0.001                 & 0.027 $\pm$ 0.001                 & 0.028 $\pm$ 0.002             & 0.030 $\pm$ 0.002             & 0.033 $\pm$ 0.002         & 0.026 $\pm$ 0.001                 & 0.051                 & 0.039                     & 0.327                 \\
\cmidrule(lr){1-2} \cmidrule(lr){3-9} \cmidrule(lr){10-12}
\multirow[l]{2}{*}{VCrW}                    & E  & \textbf{1.00 $\pm$ 0.16}         & \textbf{1.07 $\pm$ 0.14}          & 1.37 $\pm$ 0.13                   & 1.97 $\pm$ 0.5                & 2.21 $\pm$ 0.42               & 2.86 $\pm$ 0.64           & \textbf{1.10 $\pm$ 0.14}          & 1.37                  & 1.94                      & 19.40                 \\
                                            & F  & \textbf{0.018 $\pm$ 0.001}       & 0.019 $\pm$ 0.001                 & 0.022 $\pm$ 0.001                 & \textbf{0.017 $\pm$ 0.001}    & 0.019 $\pm$ 0.001             & 0.021 $\pm$ 0.001         & 0.020 $\pm$ 0.001                 & 0.040                 & 0.031                     & 0.314                 \\
\cmidrule(lr){1-2} \cmidrule(lr){3-9} \cmidrule(lr){10-12}
\multirow[l]{2}{*}{\ce{TaVCrW} (0~K)}       & E  & \textbf{1.22 $\pm$ 0.07}         & \textbf{1.30 $\pm$ 0.10}                    & 1.48 $\pm$ 0.16                   & 2.26 $\pm$ 0.55               & 2.48 $\pm$ 0.46               & 3.60 $\pm$ 0.54           & \textbf{1.33 $\pm$ 0.17}          & 2.09                  & 2.16                      & 50.80                 \\ 
                                            & F  & \textbf{0.021 $\pm$ 0.002}       & \textbf{0.022 $\pm$ 0.002}        & 0.025 $\pm$ 0.002                 & \textbf{0.022 $\pm$ 0.001}    & 0.023 $\pm$ 0.002             & 0.027 $\pm$ 0.001         & 0.024 $\pm$ 0.002                 & 0.049                 & 0.037                     & 0.488                 \\
\cmidrule(lr){1-2} \cmidrule(lr){3-9} \cmidrule(lr){10-12}
\multirow[l]{2}{*}{\ce{TaVCrW} (2500~K)}    & E  & \textbf{1.63 $\pm$ 0.07}         & 1.74 $\pm$ 0.11                   & 2.09 $\pm$ 0.09                   & 2.22 $\pm$ 0.48               & 2.34 $\pm$ 0.59               & 3.68 $\pm$ 0.70           & 2.06 $\pm$ 0.09                   & 2.40                  & 2.67                      & 59.40                 \\ 
                                            & F  & \textbf{0.116 $\pm$ 0.002}       & 0.121 $\pm$ 0.002                 & 0.141 $\pm$ 0.003                 & \textbf{0.119 $\pm$ 0.007}    & 0.126 $\pm$ 0.006             & 0.150 $\pm$ 0.003         & 0.140 $\pm$ 0.002                 & 0.156                 & 0.179                     & 0.521                 \\
\cmidrule(lr){1-2} \cmidrule(lr){3-9} \cmidrule(lr){10-12}
\multirow[l]{2}{*}{Overall}                 & E  & \textbf{1.38 $\pm$ 0.09}         & 1.56 $\pm$ 0.21                   & 1.80 $\pm$ 0.18                   & 2.19 $\pm$ 0.31               & 2.42 $\pm$ 0.31               & 3.17 $\pm$ 0.28           & 1.67 $\pm$ 0.21                   & 2.43                  & 2.32                      & 37.14                 \\
                                            & F  & \textbf{0.028 $\pm$ 0.001}       & \textbf{0.029 $\pm$ 0.001}        & 0.034 $\pm$ 0.001                 & \textbf{0.029 $\pm$ 0.001}    & 0.030 $\pm$ 0.001             & 0.034 $\pm$ 0.001         & 0.032 $\pm$ 0.001                 & 0.054                 & 0.043                     & 0.443                 \\
\cmidrule(lr){1-2} \cmidrule(lr){3-9} \cmidrule(lr){10-12}
\multicolumn{12}{l}{Deformed Structures} \\
\cmidrule(lr){1-2} \cmidrule(lr){3-9} \cmidrule(lr){10-12}
TaV                                         & E & 5.47 $\pm$ 0.66                   & 4.97 $\pm$ 0.89                   & 5.56 $\pm$ 1.84                   & \textbf{3.57 $\pm$ 1.61}      & \textbf{3.41 $\pm$ 1.57}      & \textbf{3.63 $\pm$ 1.84}  & 5.80 $\pm$ 1.56                   & 4.43                  & \textbf{3.63}             & 56.8                  \\
\cmidrule(lr){1-2} \cmidrule(lr){3-9} \cmidrule(lr){10-12}
CrW                                         & E & 1.94 $\pm$ 0.90                   & 2.19 $\pm$ 0.80                   & 2.04 $\pm$ 0.84                   & 2.74 $\pm$ 1.35               & 2.14 $\pm$ 1.37               & 2.19 $\pm$ 1.24           & \textbf{1.53 $\pm$ 1.65}          & 1.25                  & \textbf{1.04}             & 27.1                  \\
\cmidrule(lr){1-2} \cmidrule(lr){3-9} \cmidrule(lr){10-12}
TaCr                                        & E & \textbf{0.81 $\pm$ 0.73}          & \textbf{0.76 $\pm$ 0.66}          & \textbf{0.83 $\pm$ 0.66}          & 2.47 $\pm$ 1.42               & 3.57 $\pm$ 2.11               & 7.10 $\pm$ 1.96           & 1.57 $\pm$ 1.86                   & 1.89                  & \textbf{0.49}             & 13.3                  \\
\cmidrule(lr){1-2} \cmidrule(lr){3-9} \cmidrule(lr){10-12}
VW                                          & E & \textbf{0.49 $\pm$ 0.41}          & \textbf{0.51 $\pm$ 0.40}          & \textbf{0.65 $\pm$ 0.42}          & 2.90 $\pm$ 0.93               & 2.20 $\pm$ 0.83               & 3.61 $\pm$ 2.36           & 0.78 $\pm$ 0.38                   & \textbf{0.41}         & 0.52                      & 66.1                  \\
\cmidrule(lr){1-2} \cmidrule(lr){3-9} \cmidrule(lr){10-12}
TaW                                         & E & \textbf{0.89 $\pm$ 0.56}          & \textbf{1.16 $\pm$ 0.93}          & \textbf{1.09 $\pm$ 0.69}          & 5.32 $\pm$ 1.79               & 8.52 $\pm$ 2.05               & 5.68 $\pm$ 3.68           & 2.82 $\pm$ 1.79                   & 3.74                  & 3.11                      & 161.9                 \\
\cmidrule(lr){1-2} \cmidrule(lr){3-9} \cmidrule(lr){10-12}
VCr                                         & E & 6.30 $\pm$ 2.65                   & 8.49 $\pm$ 3.39                   & 8.40 $\pm$ 1.78                   & 7.20 $\pm$ 2.95               & 4.21 $\pm$ 2.17               & 3.50 $\pm$ 2.52           & 4.60 $\pm$ 3.38                   & 4.29                  & \textbf{2.70}             & 283.2                 \\
\cmidrule(lr){1-2} \cmidrule(lr){3-9} \cmidrule(lr){10-12}
Overall                                     & E & 2.65 $\pm$ 0.46                   & 3.01 $\pm$ 0.48                   & 3.10 $\pm$ 0.43                   & 4.03 $\pm$ 0.67               & 4.01 $\pm$ 0.70               & 4.29 $\pm$ 0.73           & 2.85 $\pm$ 0.56                   & 2.67                  & \textbf{1.91}             & 101.4                 \\
\cmidrule(lr){1-2} \cmidrule(lr){3-9} \cmidrule(lr){10-12}
Inference time                              &    & 51.78 $\pm$ 1.18                 & 25.09 $\pm$ 0.02                  & 14.59 $\pm$ 0.01                  & 29.48 $\pm$ 0.23              & 15.37 $\pm$ 0.04              & 4.43 $\pm$ 0.00           & 14.97 $\pm$ 0.09                  & 17.57                 & 7.25                      & 0.50                  \\
\cmidrule(lr){1-2} \cmidrule(lr){3-9} \cmidrule(lr){10-12}
Memory consumption                          &    & 36.78 $\pm$ 0.00                 & 16.93 $\pm$ 0.00                  & 8.48 $\pm$ 0.00                   & 28.82 $\pm$ 0.00              & 13.87 $\pm$ 0.00              & 5.91 $\pm$ 0.00           & 13.15 $\pm$ 0.00                  & --                    & --                        & --                    \\
\bottomrule
\end{tabular}}
\end{center}
\footnotesize{\textsuperscript{\emph{a}} Inference times for MTP and EAM were measured on two Intel Xeon E6252 Gold (Cascade Lake) CPUs. For GM-NN, an NVIDIA GeForce RTX 3090 Ti 12GB GPU was used, while all other models were evaluated using an NVIDIA A100 GPU with 80 GB of RAM.}
\end{table}

\textbf{Multicomponent alloys.} The Ta--V--Cr--W data set is designed to evaluate the performance of MLIPs across atomic systems with varying numbers of atom types/components, comprising both relaxed (0 K) and high-temperature structures~\cite{Gubaev2023}. In particular, this data set includes 0 K energies, forces, and stresses for 2-, 3-, and 4-component systems and 2500 K properties in 4-component disordered alloys. It contains 6711 configurations with sizes ranging from two to 432 atoms in the periodic cell. \Tabref{tab:hea-results} demonstrates the energy and force RMSEs for the ICTP and MACE models, evaluated separately on 0 K binaries, ternaries, quaternaries, and near-melting temperature four-component disordered alloys. ICTP consistently outperforms state-of-the-art models for the Ta--V--Cr--W data set, i.e., MTP and GM-NN. In contrast, for MACE, we were not able to identify a set of relative weights for energy, forces, and virial losses that consistently yield results better than those of MTP and GM-NN in both energies and forces. \Tabref{tab:hea-results} shows that MACE often matches the accuracy of ICTP on forces but is typically outperformed by a factor of $\leq$ 2.0 on energies.

We further compare the ICTP, MACE, MTP, and GM-NN models using the separate test data set containing binary structures strained along the $[100]$ direction. We find that neither ICTP nor MACE consistently outperforms MTP and GM-NN in this case. However, because this test data set contains only a single configuration of 432 atoms per binary, it may not serve as a valuable benchmark in this study and is included merely for completeness. Additionally, we trained a Cartesian model with $\nu=2$ and $L = l_\text{max} = 2$, which resembles a TensorNet-like architecture~\cite{Simeon2023}, though it incorporates equivariant convolution filters. The corresponding results are provided in \tabref{tab:hea-results}. We found that model configurations with $\nu=3$ and $L=2$ (and often with $L=1$) outperform the $\nu=2$ and $L = l_\text{max} = 2$ configuration by factors of $\leq$ 1.4 and $\leq$ 1.2 in energy and force RMSEs, respectively.

Finally, similar to our results for the 3BPA data set, we observe that MACE can be more computationally efficient than ICTP. We attribute longer inference times of ICTP to the pre-factor $\mathcal{K}$ arising from the Cartesian product basis in \eqref{eq:product_basis}. Thus, we attribute ICTP's lower computational efficiency to the use of the MACE architecture, which is optimized for spherical tensors. We expect that further modifications to the architecture can facilitate more efficient use of the advantages of operations based on irreducible Cartesian tensors.

\section{Broader social impact \label{sec:impact}}

This section discusses the broader social impact of the presented work. Our work has important implications for the chemical sciences and engineering, as many problems in these fields require atomistic simulations; we also discuss it in \secref{sec:introduction}. Although this work focuses on standard benchmark data sets, our experiments demonstrate the scalability of our method to larger atomic systems. Beyond constructing machine-learned interatomic potentials, equivariant models based on irreducible Cartesian tensors can be applied for molecular property prediction, protein structure prediction, protein generation, ribonucleic acid structure ranking, and many more.

Our work has no obvious negative social impact. As long as it is applied to the chemical sciences and engineering in a way that benefits society, it will have positive effects.

\end{appendices}

\end{document}